%% file: main.tex
\documentclass[twoside,leqno,twocolumn]{article}

\usepackage[letterpaper]{geometry}
\PassOptionsToPackage{ruled}{algorithm}
\usepackage{siamproceedings}

\usepackage[T1]{fontenc}
\usepackage{amsfonts}
\usepackage{graphicx}
\usepackage{epstopdf}
\usepackage{enumitem}
\usepackage{booktabs}
\usepackage{needspace}
\usepackage{wrapfig}
\usepackage{makecell}
\usepackage{xcolor}
\usepackage{thm-restate}
\usepackage[noend]{algorithmic}
\usepackage{bm}
\usepackage{amsopn}
\usepackage{setspace}
\usepackage{subcaption}
\usepackage{tikz}
\usepackage{tikz-3dplot}
\usepackage{listings}
\usepackage{xspace}
\usepackage{adjustbox}
\usepackage{dblfloatfix}

\definecolor{Emerald}{HTML}{00A99D}
\definecolor{hhublue}{RGB}{1, 106, 179}
\colorlet{lighthhublue}{hhublue!75}

\hypersetup{hidelinks}
\crefname{figure}{Fig.}{Figs.}
\crefname{algorithm}{Alg.}{Alg.}
\crefname{section}{Sec.}{Sec.}
\crefname{appendix}{App.}{App.}

\ifpdf
  \DeclareGraphicsExtensions{.eps,.pdf,.png,.jpg}
\else
  \DeclareGraphicsExtensions{.eps}
\fi

\newsiamthm{problem}{Problem}

\providecommand{\titlerunning}[1]{}

\input{macros}

\newcommand{\geetal}{Ge et al.}
\newcommand{\thompsonmerrifield}{Thompson and Merrifield}

\begin{document}

\onecolumn
\title{\Large Connected Subspace Clustering: Hardness, a Scalable Heuristic, and an Application to Sea Level Geodesy}
\titlerunning{Connected Subspace Clustering}
\author{}
\date{}

\author{Johanna Hillebrand\thanks{Heinrich Heine University Düsseldorf, Universitätsstraße 1, 40225 Düsseldorf, Germany (\email{johanna.hillebrand@hhu.de}).}
\and Jan Höckendorff\thanks{University of Cologne, Albertus-Magnus-Platz, 50923 Köln, Germany (\email{hoeckendorff@cs.uni-koeln.de}).}
\and Jürgen Kusche\thanks{University of Bonn, Regina-Pacis-Weg 3, 53115 Bonn, Germany.}
\and Kelin Luo\thanks{University at Buffalo, 12 Capen Hall, Buffalo, New York 14260-1660, USA (\email{kelinluo@buffalo.edu}).}
\and Heiko Röglin\footnotemark[3]
\and Melanie Schmidt\footnotemark[1]
\and Christian Sohler\footnotemark[2]
\and Bernd Uebbing\footnotemark[3]}

\maketitle
\begin{center}
\begin{minipage}{0.7\textwidth}
  \input{abstract}

\end{minipage}
\end{center}

\twocolumn

\input{introduction}
\input{Connected_subspace_clustering}

\input{practical-algorithm-for-connected-subspace-clustering}
\input{conclusion}
\newpage

\bibliographystyle{siamplain}
\bibliography{biblio}
\newpage
%%%%%%%%%%%%%%%%%%%%%%%%%%%%%%%%%%%%%%%%%%%%%%%%%%%%%%%%%%%%%%
\appendix
\section{Theory}
\input{appendix/appendix-related-work}
\input{appendix/appendix-remaining-proof}

\section{Application}
\input{appendix/appendix-data-set-choice}

\input{appendix/appendix_eofs_andPCs}
\input{appendix/appendix_geodetic-analysis}

\section{Implementation}
\input{appendix/appendix_accompanying_pseudo_code}
\input{appendix/appendix-merging}

\section{Experiments}
\input{appendix/appendix_experiments}

\input{appendix/results_table}

\end{document}

%% file: macros.tex
\def \exp{{\mathrm exp}}

\DeclareMathOperator*{\argmin}{arg\,min}

%% file: abstract.tex
\begin{abstract}
Constrained optimization extends classical optimization by integrating side information, making it widely applicable across scientific and engineering domains. 
Consider a setting where we are measuring a set of variables at different physical locations. When grouping these measurements, it is natural to try to satisfy two different objectives. 
On the one hand, we want our groups (clusters) to contain similar measurements. 
On the other hand, we want to have physically coherent clusters. 
Thus, we have a constrained clustering problem where the constraint models the coherence. Motivated by an application in geodesy, where contiguous regions of the sea surface must be identified for subsequent principal component analysis of sea level anomaly data, we introduce the \emph{Connected Subspace Clustering} problem. Given high-dimensional points and a connectivity graph on them, the objective is to partition the points into $k$ connected clusters while minimizing their total squared distance to the clusters' best-fit $m'$-dimensional affine subspaces.
We prove that, even for $m' = 0$ and when the connectivity graph is a grid graph with holes, this problem is NP-hard to approximate within a factor of $\Omega(n^{1/2-\varepsilon})$ for every $\varepsilon>0$, where $n$ is the number of measurements in the input.
We then introduce an efficient Lloyd-style heuristic that alternates subspace fitting with an iterative component-merging procedure to enforce connectivity.
Our method returns exactly $k$ connected regions by construction, whereas unconstrained subspace clustering methods leave up to $1{,}966$ disconnected fragments and substantially higher reconstruction cost. In an experimental study of 160 configurations on global sea level time series, our merging-based connectivity repair is the strongest of four strategies in $73.75\%$ of configurations, and consistently outperforms competitors such as (connected) Ward's method across all tested cluster counts.
The resulting regions isolate signals that align with established climate indices such as those describing the El Niño--Southern Oscillation and Indian Ocean Dipole.
Although developed for sea level geodesy, the approach applies in principle to other spatially embedded multivariate time series, such as climate fields, remote sensing imagery, neuro-imaging, and sensor networks. 
\end{abstract}

%% file: introduction.tex
\section{Introduction}

Clustering is a widely used machine learning tool, applied in fields as diverse as medical imaging and astronomy. In practice, additional domain-specific structure often imposes constraints on how clusters can be formed. 
Consequently, various forms of constrained clustering, such as capacitated or fairness-aware clustering, have been discussed extensively in the recent literature. 

\begin{figure*}[t!]
  \centering
  \begin{tikzpicture}
     \node [label=below:{(a) Greedy $k$-means++}] at (0,0) {\includegraphics[width=0.49\textwidth]{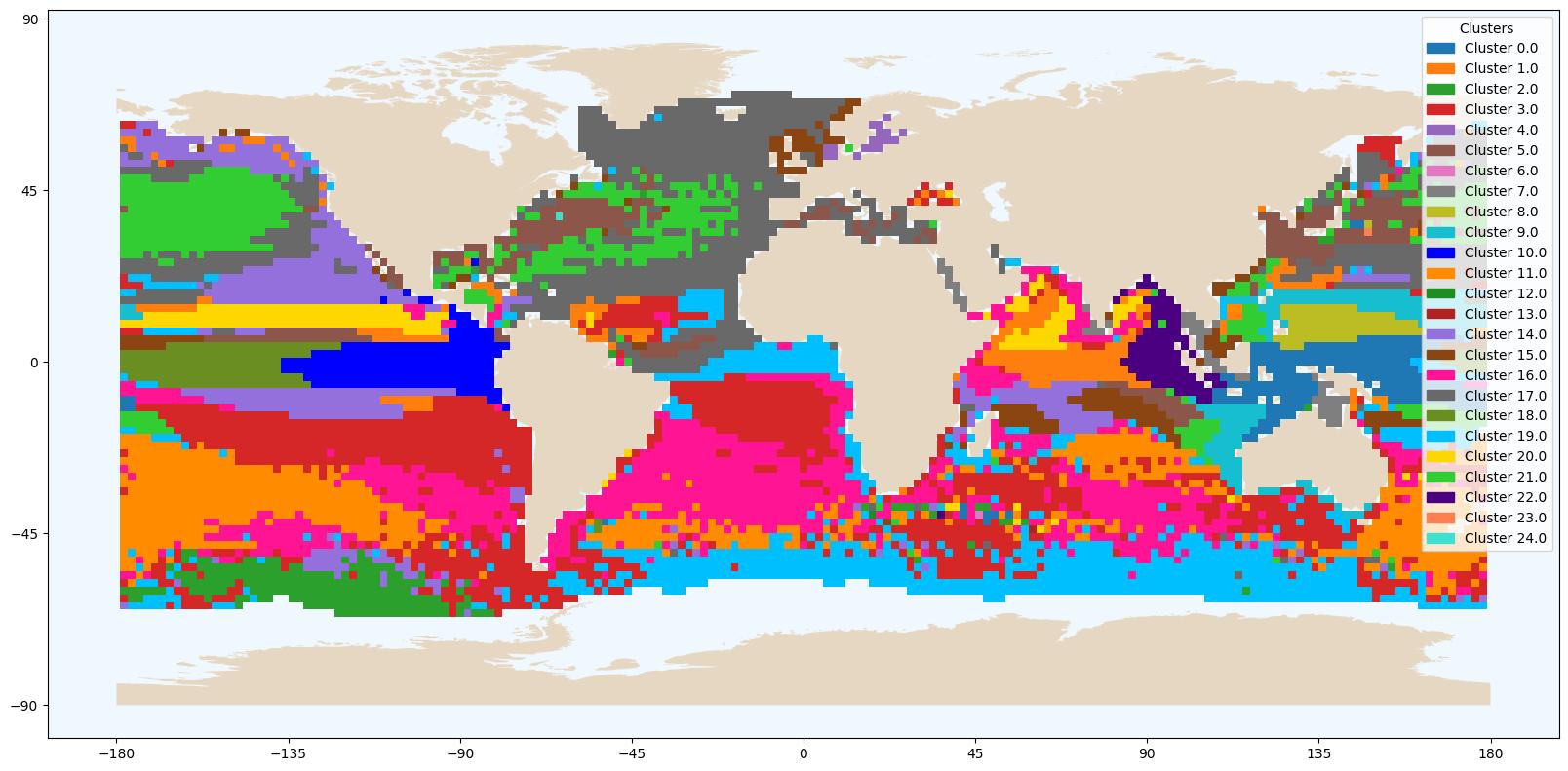}};
     \node [label=below:{(b) Ward's method}]at (8.55,0) {\includegraphics[width=0.49\textwidth]{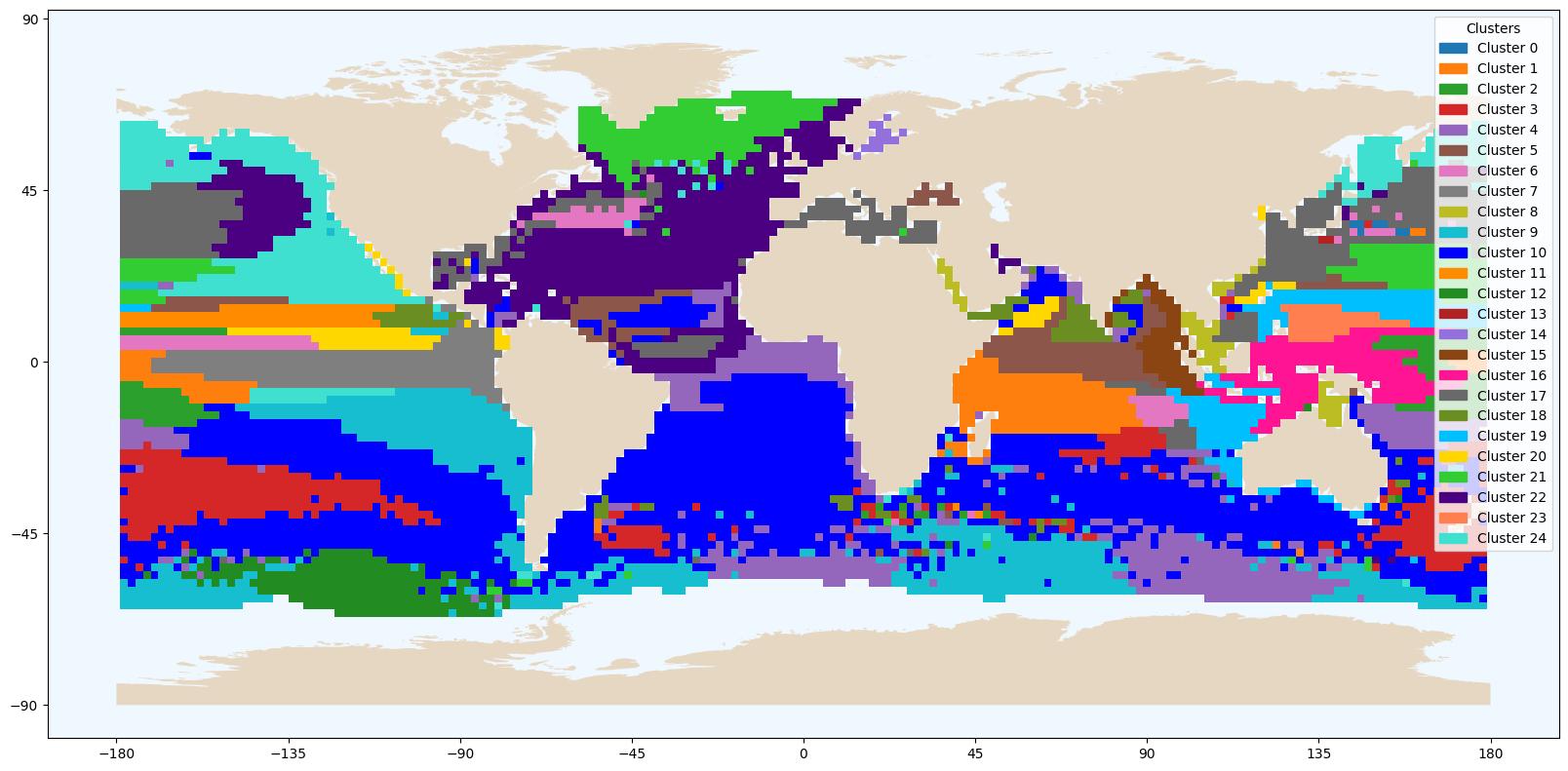}};
     \node [label=below:{(c) SSC-OMP}] at (0,-5)  {\includegraphics[width=0.49\textwidth]{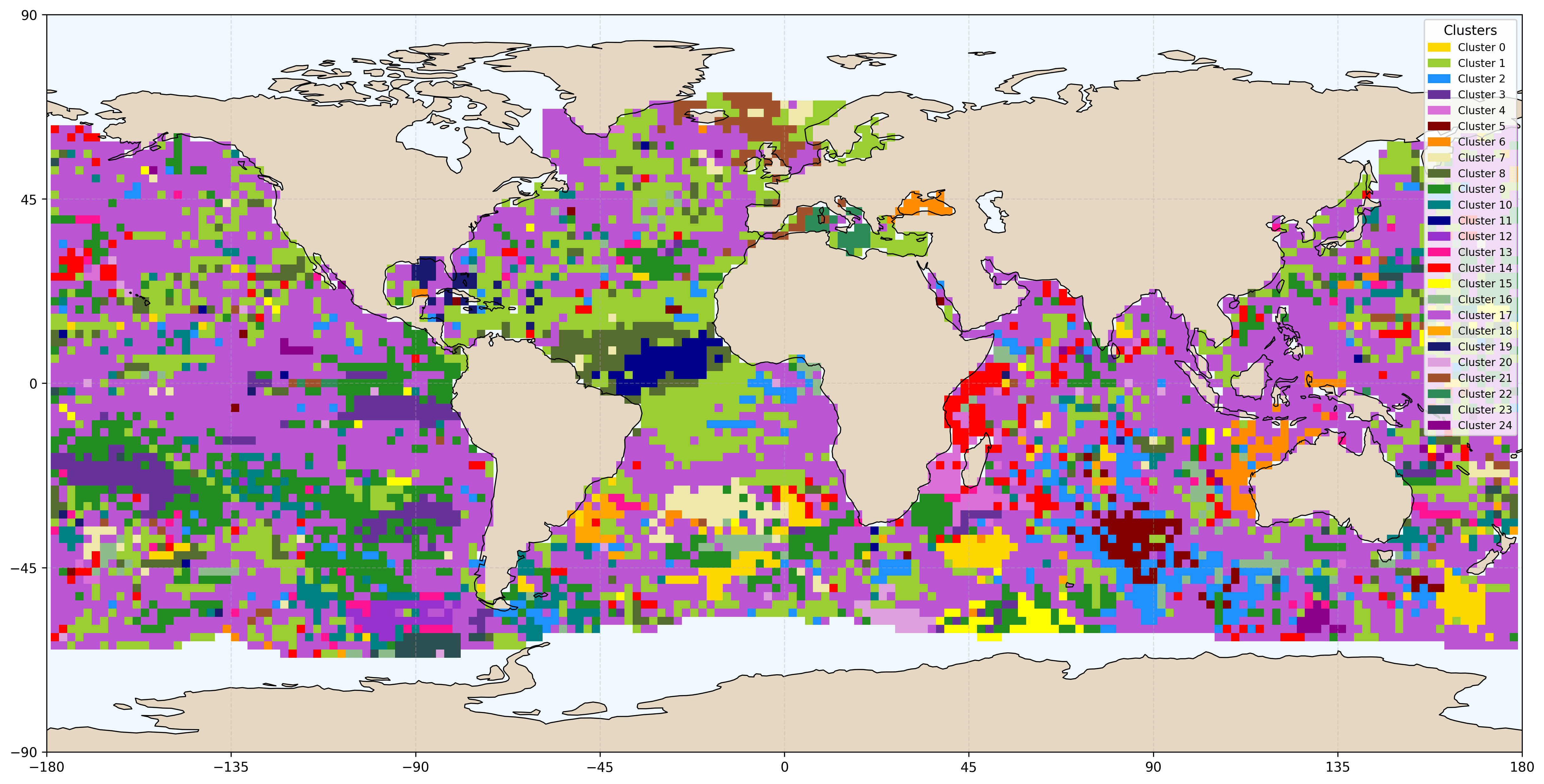}};
     \node [label=below:{(d) EGCSC}]at (8.55,-5){\includegraphics[width=0.49\textwidth]{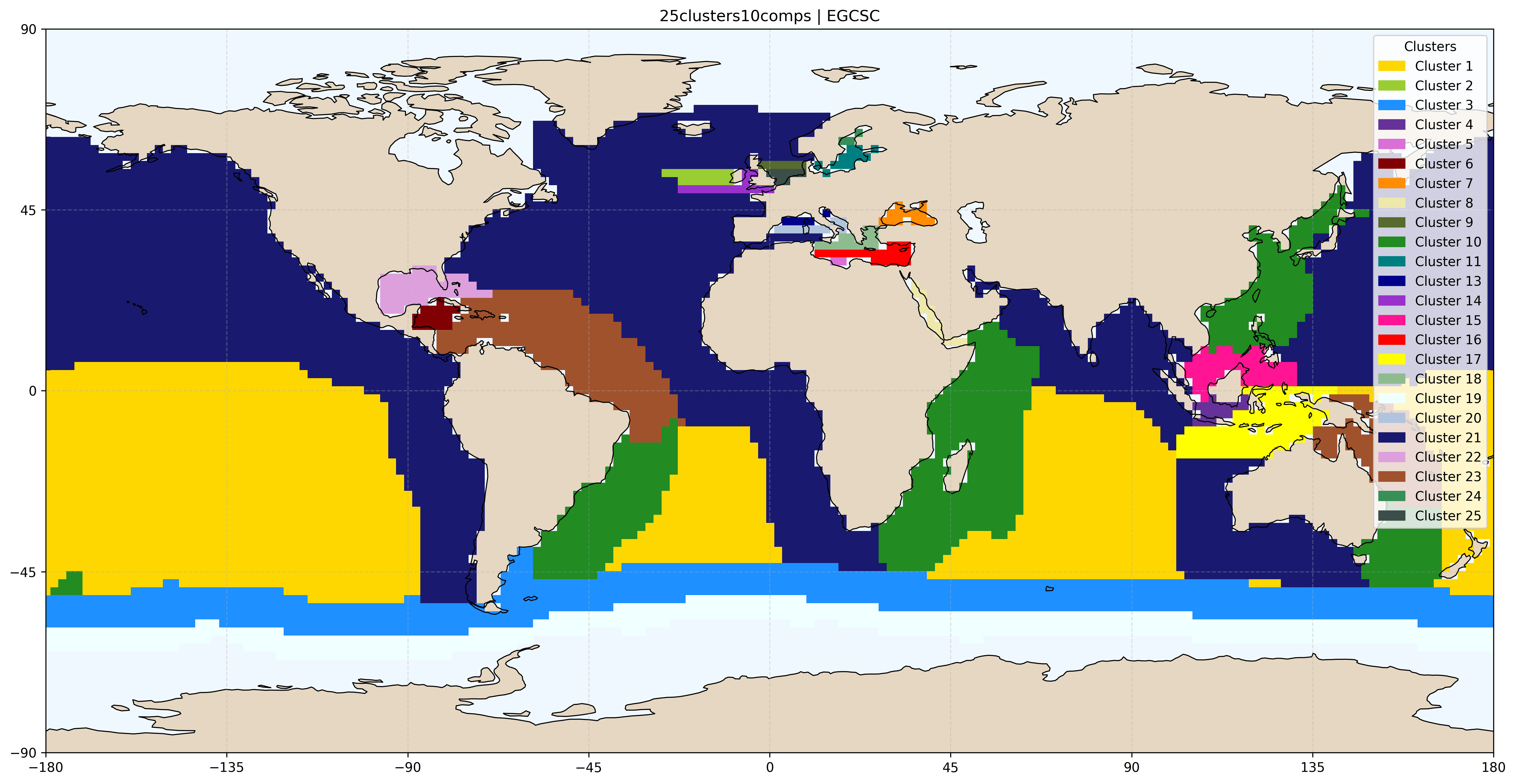}};
     \node [label=below:{(e) Connected Ward's method}]at (0,-10) {\includegraphics[width=0.49\textwidth]{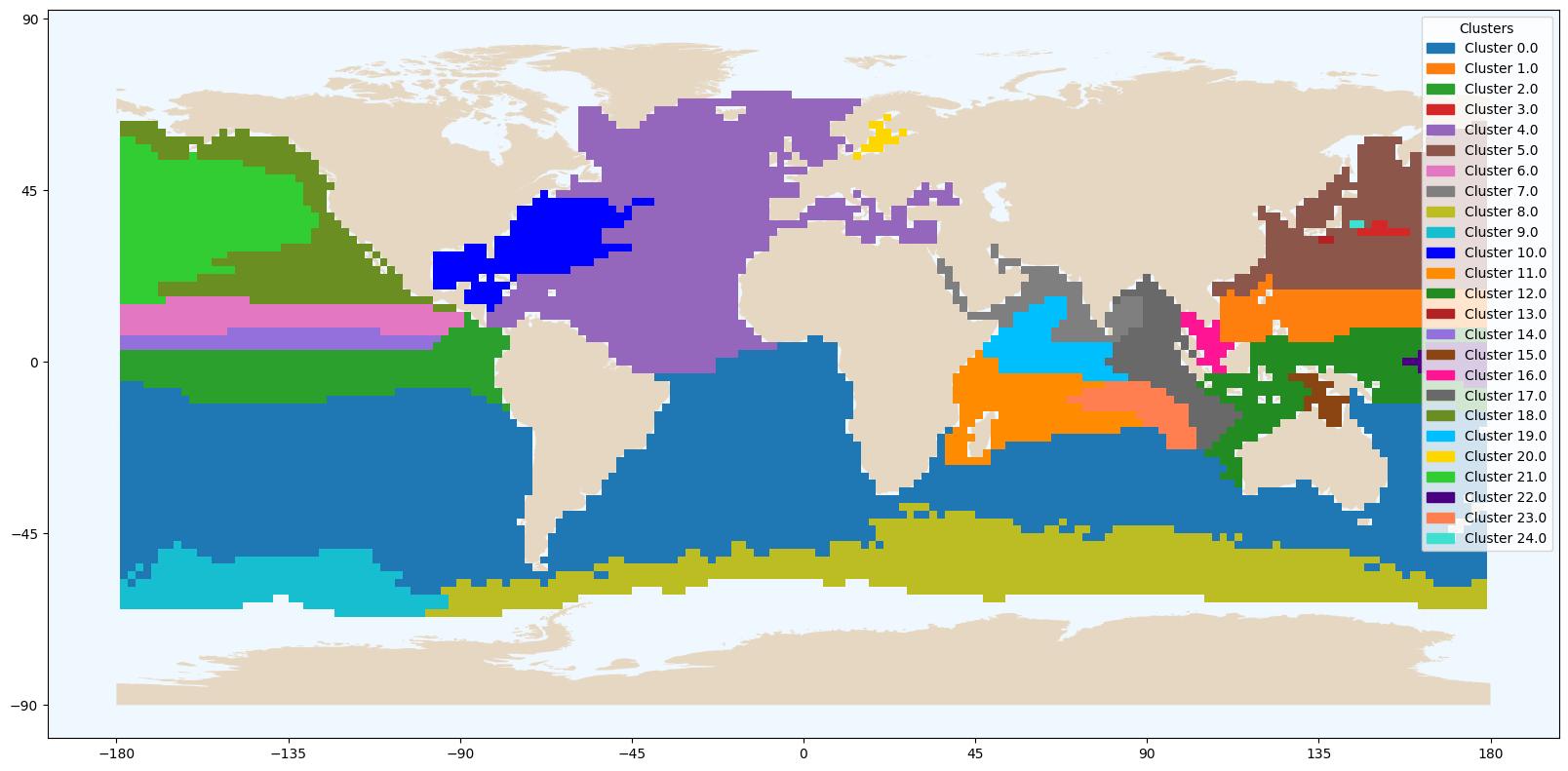}};     
     \node [label=below:{(f) Connected Subspace Clustering}]at (8.55,-10) {\includegraphics[width=0.49\textwidth]{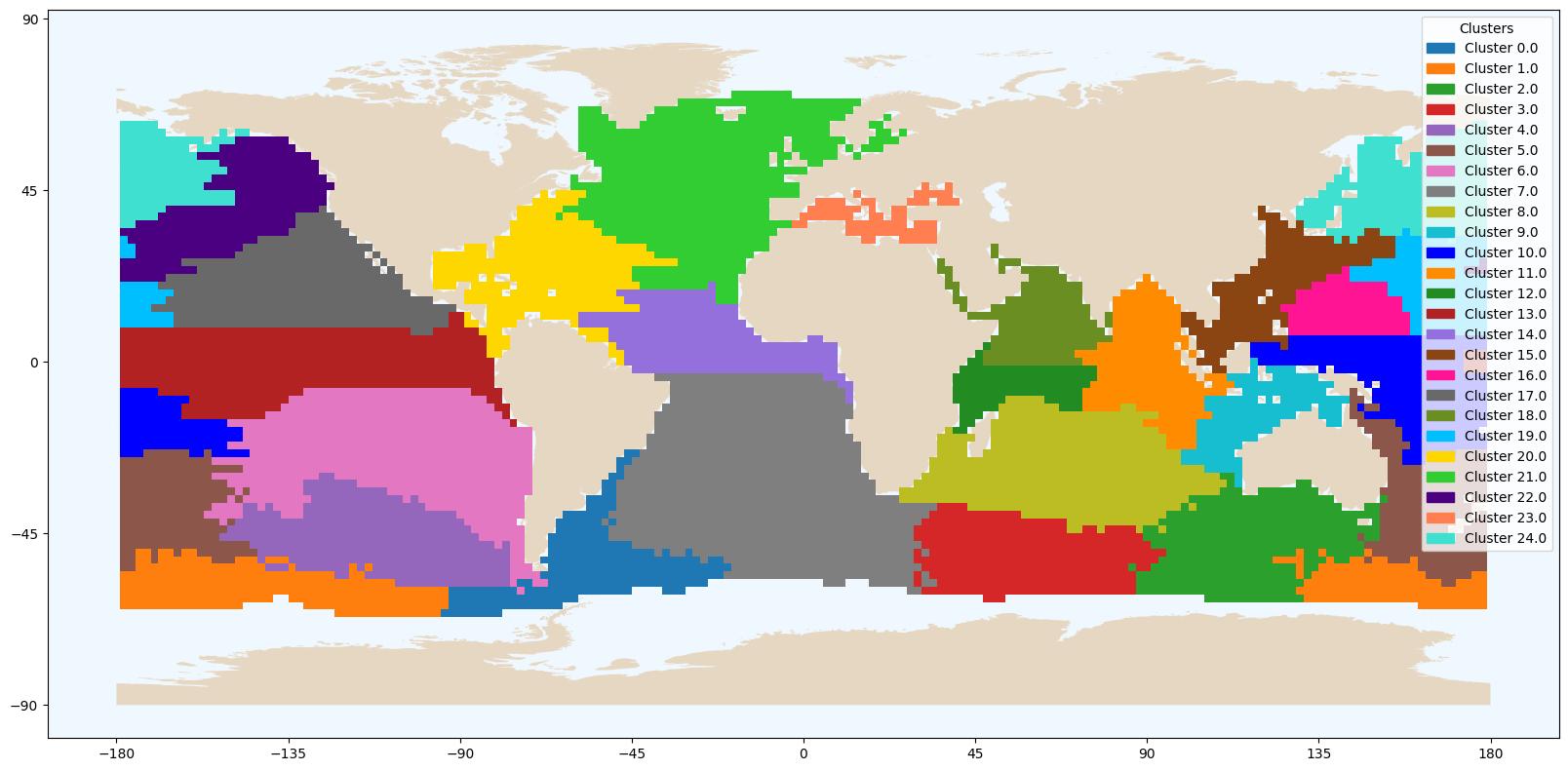}} ;\end{tikzpicture}
  \caption{Clustering of the sea level anomaly dataset for $k=25$. Unconstrained clustering and subspace clustering methods produce spatially fragmented clusters (a-d), whereas connectivity-constrained methods (e-f) produce spatially coherent clusters. Connected Subspace Clustering (f) produces spatially coherent and interpretable clusters, while also achieving the lowest subspace-clustering cost among the displayed methods.} 
  \label{fig:exampleclusterings}
\end{figure*}

Connectivity-constrained clustering naturally arises when handling spatial datasets, which are common in many fields including remote sensing, neuro-imaging and sensor networks, and it can be described as the requirement of identifying contiguous regions with coherent dynamics. More formally, it is defined as a family of constrained clustering problems where a \emph{connectivity graph $G=(V,E)$} is given in addition to the input (the standard input contains the points, the target number of clusters, and the metric) \cite{GEGHBB08}. The connectivity graph is unweighted and undirected. The constraint is that for every cluster, the induced subgraph in $G$ has to be a connected graph. The graph has no guaranteed correspondence with the metric, so points $u$, $v$ that share an edge $\{u,v\}$ can still be very far away from each other with respect to the metric.
This sort of constraint has a long history. In a survey from 1985, Murtagh~\cite{murtagh85} reviews applications from social sciences, earth sciences, pattern recognition and image processing and describes \emph{contiguity-constrained} agglomerative clustering methods.
The idea was rediscovered in 2008 by Ge et al.~\cite{GEGHBB08} under the name \emph{connected clustering}, in the context of community detection which spurred subsequent work on connectivity-constrained clustering \cite{DELRRSW24,ELRRS25,gupta_clustering_2011}. 

Another important line of work is subspace clustering which addresses the fact that high-dimensional data often contains meaningful patterns that exist only within specific subsets of dimensions. In particular, with modern sensing technologies, the number of measured dimensions grows quickly. Traditional clustering methods may fail to detect these patterns due to the \textit{curse of dimensionality} \cite{kriegel2009clustering,sim2013survey}. Restricting attention to fewer, more relevant directions improves accuracy and interpretability in complex datasets. Thus, subspace clustering has become especially valuable in fields like bioinformatics, image analysis, motion segmentation and market research. Formally, it is defined as follows: given points $P\subset \mathbb{R}^m$, a number $k$ and a target dimension $m'$, find $k$ affine subspaces of dimension $m'$ that minimize the sum of the squared distances of points to their closest subspace.
Representative approaches include CLIQUE, using grid-based dense units \cite{agrawal1998automatic}, PROCLUS, using projected $k$-medoids with per-cluster relevant dimensions \cite{aggarwal1999fast}, SUBCLU a DBSCAN-style algorithm, which uses density-based clusters in subspaces defined via density-reachability \cite{kailing2004density} and numerous other methods \cite{elhamifar2013sparse}.

When datasets are both structured and high-dimensional, the ideas of subspace clustering and connectivity naturally meet and thus, we introduce \emph{Connected Subspace Clustering}, a problem that combines the graph-topological constraint of connectivity with the statistical objective of minimizing the subspace reconstruction error in one optimization problem. 
\subsection{Our Contributions}

\emph{(i)} We introduce \emph{Connected $m'$-Subspace $k$-Clustering}, which combines affine-subspace reconstruction error with connectivity in a given graph. \\
\emph{(ii)} We prove that, even for $m'=0$ and grid graphs with holes, the
problem is NP-hard to approximate within a factor of
$\Omega(n^{1/2-\varepsilon})$ for every $\varepsilon>0$. \\
\emph{(iii)} We design a generic (heuristic) merging subroutine that converts fragmented cluster assignments into connected regions and can be
incorporated into Lloyd-style $k$-subspaces algorithms. On
bounded-degree connectivity graphs, the routine runs in
$O(n(k+\log n))$ time. \\
\emph{(iv)} We provide extensive experimental evaluation for an important application in sea level geodesy, comparing our algorithm to methods from the areas of classical clustering, subspace clustering and hyperspectral imaging.

\subsection{Application}\label{sec:application}
The motivation for our clustering problem stems from sea level geodesy, in particular from the analysis of climate-driven sea level variability.
Contemporary sea level rise is predominantly driven by climate change via glacier/ice sheet melt and volumetric expansion (ocean warming) \cite{ipcc2021}. Since 1993, satellite altimetry has provided (near) global records of \emph{sea level anomaly} (SLA) with centimeter-level accuracy on an almost gap-free ocean grid, excluding only land areas and parts of the high latitudes. This data enables the study of sea level variability on global as well as regional scales. 

A standard tool for analyzing such data is \emph{principal component analysis} (PCA), which decomposes the $3$-dimensional SLA field into spatial patterns known as \emph{empirical orthogonal functions} (EOFs) and associated temporal components, both ordered by explained variance. 
On the global domain this ordering by variance favors large-scale, dominant signals such as the long-term sea level trend and the annual cycle (e.g., seasonal land-water exchange). Regionally important but lower-variance phenomena, such as the \emph{El Niño–Southern Oscillation} (ENSO), therefore tend to appear only in higher-order modes (Fig.~\ref{fig:eofs}). Thus, analyses performed on the global field can obscure spatially localized signals.
For this reason, regional analyses are standard in sea level research \cite{li2024analysis}. 

Such analyses require partitioning the ocean surface into subregions that exhibit internally similar SLA variability and then can be studied separately, for example by applying PCA on each region. Partitions are often defined manually or based on prior geographic knowledge. In this work, we replace such a partitioning with data-driven clustering of altimetry time series.
The resulting clusters group ocean areas with similar SLA behavior and provide a basis for downstream tasks such as reconstructing pre-satellite-era sea level from tide gauges \cite{church2004} and separating steric and mass contributions \cite{rietbroek2016revisiting}. 
Because the downstream analysis applies PCA separately within each region, a natural partitioning objective is the total rank-$m'$ reconstruction error over all regions. This is exactly the sum of squared distances to the regions' best-fit $m'$-dimensional affine subspaces. Requiring each cluster to be geographically contiguous therefore leads directly to our Connected Subspace Clustering problem.

Existing methods do not simultaneously optimize subspace reconstruction error and enforce connectivity in a prescribed spatial graph. 
\Cref{fig:exampleclusterings} shows the output of six methods, including our own. 
Greedy $k$-means++ and Ward's method (scikit-learn \cite{scikitlearn}) yield spatially fragmented regions, but spatial contiguity is crucial for oceanographic analysis (\Cref{fig:exampleclusterings}(a,b)). 
Typical subspace clustering algorithms also do not enforce connectivity and thus can produce highly fragmented clusters (\Cref{fig:exampleclusterings}(c)).
Related work in hyperspectral imaging (HSI) also handles high-dimensional data under spatial coherence constraints e.g.~\cite{cai2020graph,wang2022graph,guan2024contrastive}.
While it shares these challenges with our setting, there are key differences in goals and problem formulation. Most importantly, it encourages spatial coherence but does not strictly enforce it while our problem formulation requires strictly connected clusters. Applying HSI on ocean data does not produce connected clusters (\Cref{fig:exampleclusterings}(d)).

A well-established connected clustering method is connected Ward's method, which is a rather old agglomerative clustering method (\cite{murtagh85,li11}). It provides coherent clusters, but  yields uneven patches that do not align well with oceanographic phenomena (\Cref{fig:exampleclusterings}(e)). 
The regions resulting from our method, in comparison, are both contiguous and interpretable (\Cref{fig:exampleclusterings}(f)).

\subsection{Related Work}

We are not aware of any algorithm specifically designed for our problem formulation of connected subspace clustering. For subspaces of dimension zero, i.e., for $k$-means, and generally for point-based clustering problems, some results on connected clustering are known. 
\emph{Connected $k$-center:} The radius-based $k$-center problem is somewhat easier to tackle than sum-based objectives like $k$-means. Surprisingly, it is not $O(1)$-approximable: The connected $k$-center problem is hard to approximate in $o(\log^\ast k)$~\cite{janheiko-connected-25}. 
On trees, the connected $k$-center problem is solvable optimally \cite{DELRRSW24,GEGHBB08}; for general graphs, there exists a $O(\log^2 k)$-approximation~\cite{DELRRSW24}. 
\emph{Connected $k$-median/$k$-means:} For general graphs, connected $k$-median is NP-hard to approximate within $\Omega(n^{1-\epsilon})$, but is solvable on trees \cite{ELRRS25}. The hardness extends to Euclidean $k$-means and thus to subspace clustering. 
 On star-like graphs, connected $k$-means is NP-hard to approximate within $\Omega(\log n)$ and admits an $O(\log n)$-approximation \cite{gupta_clustering_2011}. To our knowledge, the case where the connectivity graph is a grid has not been studied before this work. 
\emph{ILP formulations:} The point-based connected clustering problem has also been modeled as an integer linear program in the area of cartography in the context of political redistricting (e.g. \cite{ValidiBL22}). While ILP solvers nowadays are quite fast, this approach does not scale to our data.

As mentioned above, we are not aware of any algorithms proposed for subspace clustering where the connectivity is explicitly enforced but now review related algorithms and heuristics. The approaches need to be able to handle larger data sizes (the ocean dataset includes 532{,}783 time series with 365 monthly time steps).
\emph{Iterative improvement heuristics:}
\geetal~\cite{GEGHBB08} propose a heuristic called NetScan for connected $k$-center and adapt it to connected $k$-means. The broad idea of this algorithm is: initialize centers, assign while enforcing connectivity via simultaneous \emph{breadth-first search} (BFS) with a look-ahead for ``bridge''-nodes, then recompute better centers for each subgraph and iterate this process until a convergence criterion is reached. This approach does not scale well, and there is no implementation available. 
We implement a related Lloyd-style scheme with a more efficient assignment step.
\emph{Connected Ward's method:}
Agglomerative clustering is a general approach that starts with singleton clusters and then subsequently merges a pair of clusters into one cluster until the desired number of clusters is reached. This  can naturally be adapted to respect connectivity by only allowing merges of clusters that are neighboring.
We use the implementation of this method from scikit-learn \cite{scikitlearn} that is based on Murtagh~\cite{murtagh85} and Li~\cite{li11}. 
\emph{Subspace clustering methods:} 
Subspace clustering is a large and active research area. We choose two well-established methods to investigate the connectivity of the produced solutions: 
\emph{Sparse subspace clustering with orthogonal matching pursuit} (SSC-OMP)~\cite{YouRV16}, which addresses subspace clustering without any connectivity constraints, and \emph{elastic net subspace clustering} (EnSC)~\cite{you2016oracle}, which encourages well-connected affinity graphs within recovered subspaces, but it does not enforce connectivity in an externally specified spatial graph.
\emph{Hyperspectral imaging (HSI):}
Hyperspectral images typically are 3D tensors (two spatial dimensions and one spectral dimension). Many spatial-spectral clustering methods developed for HSI data
encourage spatial coherence via a regularization term, which encourages neighboring pixels to have the same label. Such regularization promotes local coherence but generally does not require every cluster to induce a single connected component. This is the method of choice for, e.g., identifying landscape types. 
 We nevertheless compare our method with two popular HSI approaches~\cite{cai2020graph},  \emph{efficient graph convolutional subspace clustering} (EGCSC) and \emph{efficient kernel graph convolutional subspace clustering} (EKGCSC).

%% file: Connected_subspace_clustering.tex
\section{Problem Definition, Hardness and Algorithmic Overview}\label{sec:problem}

We define the connected subspace clustering problem as follows: 
Let $P \subset \mathbb{R}^{m}$ be a set of $m$-dimensional points with $|P| = n$ (corresponding to $n$ time series of length $m$). We also think of this point set as an $n\times m$-matrix where the $i$-th point is in the $i$-th row. 
Let $d : \mathbb{R}^m \times \mathbb{R}^m \to \mathbb{R}$ denote the Euclidean distance which we use as a distance measure for time series.  
Let $G=(P,E)$ denote an unweighted and undirected connectivity graph. 
We want to partition $G$ into connected subgraphs.

\begin{definition}{(Feasible clustering)}
   For a number of target clusters $k$,  a partition of $P$ into clusters $C_1,\dots,C_k$ is feasible if the subgraph of $G$ induced by $C_i$ is connected for all $i \in \{1,\dots,k\}$.
\end{definition}

The purpose of the clusters from the point of the application is that computing a PCA with $m'$ components on a cluster shall induce a smaller error than computing it on the global dataset, in order to enable better analysis of sea level behavior for each local area. 

To this end, we want to minimize the sum of squared distances of points to the best-fit subspace of their cluster. The best-fit subspace of a set of points \( C_i \subset \mathbb{R}^m \) is the affine subspace of a given dimension \( m' \) that minimizes the sum of squared Euclidean distances from the points in \( C_i \) to the subspace. It is obtained by computing the mean $\mu_i = \frac{1}{|C_i|} \sum_{x \in C_i} x$,
 the centered data matrix
$X = [x - \mu_i : x \in C_i] \in \mathbb{R}^{m \times |C_i|}$, and then  the singular value decomposition of $X$. Then the best-fit subspace is $\mu_i + \operatorname{span}(u_1, \dots, u_{m'})$,
where $u_1, \dots, u_{m'}$ are the top $m'$ left singular vectors of $X$, which are the $m'$ top eigenvectors of $X X^T$.

\begin{definition}{(Connected $m'$-subspace $k$-clustering)} \label{def:subspace_clustering}
Given $P$, $G$, $k$, and $m'$, find a feasible clustering $C_1,\dots,C_k$ of $P$ with corresponding best-fit subspaces $S_i$ of dimension $m'$ such that 
$\sum^k_{i=1}\sum_{x\in C_i} d(x,S_i)^2$  
is minimized, where $d(x,S_i) = d(x-\mu_i, \operatorname{span}(u_1,\dots,u_{m'}))$.   
\end{definition}

\subsection{Hardness}

We prove that the connected $m'$-subspace $k$-clustering problem is hard to approximate to a constant even for $m'=0$. This is even true if the connectivity graph is a grid graph with holes which is the situation present in the application that we study in \Cref{sec:exp}.
 We show this via a non-trivial reduction from the vertex-disjoint paths problem on grid graphs. The proof of the following theorem is given in the supplementary material, see Appendix~\ref{appendix-remaining-proof}. 

\begin{theorem}\label{thm-hardness}
    For any $k \in \mathbb{N}$, connected $m'$-subspace $k$-clustering with $m'=0$ is NP-hard to approximate within factor $\Omega(|V'|^{1/2-\varepsilon})$ for any $\varepsilon > 0$, on grid graphs $G'=(V',E')$ with holes in Euclidean $t$-dimensional space, where $t \geq k$.
\end{theorem}

\subsection{Establishing Connectivity:} We design a generic merging subroutine for clusterings with fragmented clusters. 
It operates on a neighborhood graph $N$ over the input points; $N$ may be any finite, simple, undirected graph, as only adjacency information is used, with no further structural assumptions. In our experiments, $N$ is a grid graph with holes as dictated by our application.
The routine first decomposes every cluster into its connected
components. It then repeatedly selects a smallest component and
merges it into an adjacent cluster, until $k$ connected components remain. The target cluster is chosen according to which neighboring cluster subspace best explains the
points of the selected component. The routine can be plugged into various clustering algorithms to establish connectivity, see \Cref{sec:algos}. 

\subsection{Algorithmic Framework:} 
Our heuristic combines an initial clustering with a Lloyd-style
iterative improvement procedure. In each iteration, we compute a
best-fit $m'$-dimensional affine subspace for every cluster and reassign
each point according to its reconstruction error. Since these
reassignments may violate connectivity, we consider several strategies
for maintaining or restoring connected clusters, including iterative
component merging, postprocessing, spatial filtering, and locally
restricted reassignment. The framework is intentionally modular:
different initialization methods and connectivity strategies can be
combined within the same iterative procedure.
Section~\ref{alg:sc} presents these methods in detail.
Section~\ref{sec:exp} compares the resulting configurations
on sea level anomaly data.

%% file: practical-algorithm-for-connected-subspace-clustering.tex
\section{An Algorithm for Connected Subspace Clustering}\label{sec:algos}

We begin with the general subspace clustering algorithm (\Cref{sec:basic_algo}), followed by two key components: the initialization (\Cref{sec:initialclustering}) and a method to (re)establish connectivity (\Cref{sec:connectivity}). 
\subsection{Basic Algorithm}\label{sec:basic_algo}
The general structure of our algorithm follows a Lloyd-style iterative improvement strategy: Alternately, assign points to clusters and update the clusters' optimal representation. 
Lloyd-type methods strictly decrease the objective each step but can take exponentially many iterations in the worst case (e.g., even for $k$-means \cite{ArthurV06}). Stronger convergence guarantees require additional assumptions in the input \cite{wang2022convergence}. We make no such assumptions and cap the number of iterations to ensure polynomial running time. 
This works well for subspace clustering because the $m'$-dimensional subspace $S$ minimizing $\sum_{x\in C_i} d(x,S)$ is just the best-fit subspace of dimension~$m'$ and it can be computed via principal component analysis. This procedure is given in \Cref{alg:sc}.
Several steps need further specification. First and foremost, we need a good initial clustering to start with. It is known that for $k$-means, the initialization is of utmost importance. We discuss several options in Section~\ref{sec:initialclustering}. Second, we need to specify a convergence criterion. We use the following simple criterion: Either terminate after a fixed number of iterations or when assignments no longer change.
Third, it is unclear how to maintain or (re)establish connectivity for the clusters. This can be done after each iteration or once after the procedure terminated. There is a big caveat to this question, since enforcing connectivity can increase the objective function (Section~\ref{sec:connectivity}).
\begin{algorithm}[t]
\caption{Connected Subspace Clustering\label{alg:sc}}
\begin{algorithmic}[1]
\STATE \textbf{Require:}\enspace $P = \{x_1, \dots, x_n\}$, number of clusters $k$, subspace dimension $m'$
\STATE \textbf{Ensure:}\enspace Clusters $C_1,\ldots,C_k$ with best-fit subspaces $S_1,\ldots,S_k$ of dimension $m'$

\STATE Compute an initial clustering $C_1, \dots, C_k$

\WHILE{stopping criterion not met}
    \FOR{$i=1$ to $k$}
        \STATE Compute the best-fit subspace
        \STATE \hspace{1em}$S_i \gets \mu_i + \operatorname{span}(v_1, \dots, v_{m'})$
        \STATE $C_i' \gets \emptyset$
    \ENDFOR
    \FOR{$x \in P$}
        \STATE Let $j \in \arg \min_{1 \leq i \leq k} d(x,S_i)$\label{step:assignment1}
        \STATE $C'_{j} \gets C'_j \cup \{x\}$\label{step:assignment2}
    \ENDFOR
   
    \label{step:establish_conn}
    \FOR{$i = 1$ to $k$}
        \STATE Set $C_i \gets C_i'$
         \STATE \textbf{Optional:} Establish connectivity for $C_i$
    \ENDFOR
\ENDWHILE

\STATE \textbf{Optional:} Establish connectivity for all $C_i$ \label{step:final_conn}
\end{algorithmic}
\end{algorithm}

\subsection{Initial Clustering}\label{sec:initialclustering}
We employ five initialization methods based on four clustering
approaches. Two methods use the same connectivity-constrained
agglomerative procedure with different dissimilarity measures. \\
\emph{1) Agglomerative clustering without a connectivity constraint (Agglo-ST)} 
\thompsonmerrifield~\cite{TM14} apply average-linkage clustering with a spatiotemporal distance function $D(X_i,X_j)=1-\exp(-d(x_i,x_j)/c_0)r(t_i,t_j)$, where $X_i=(x_i,t_i)$ combines location and time series. The exponential term penalizes spatial separation, while Pearson's $r$ rewards temporal similarity. Connectivity is only implicit and thus, connected clusters appear only when the data are smooth in space and time. This is explained in greater detail in the appendix (see~\Cref{appendix-tm}). \\
\emph{2) Agglomerative connected clustering}
The main difference to standard agglomerative clustering is the restriction of merges to pairs of neighboring clusters.
We defer the pseudocode to Appendix~\ref{appendix:pseudocode}.
We implemented Algorithm~\ref{alg-ac-connected}  and use it with (i) average linkage $D(A, V) = \frac{1}{|A||V|} \sum_{x \in A} \sum_{y \in V} D(x, y)$ (\emph{Conn-Agglo-Euc}) and (ii) with the spatiotemporal distance function described in 1) (\emph{Conn-Agglo-ST}). \\
\emph{3) Ward's method with connectivity (Conn-Ward)} 
This follows the same idea as agglomerative connected clustering, but with Ward's objective function. This algorithm was publicly available and we use the implementation of scikit-learn~\cite{scikitlearn}. \\
\emph{4) Greedy $k$-means++ with connectivity (Conn-KMeans++)}\label{sec:greedy-kmeans}
As a final initialization method, we use scikit-learn's implementation of greedy $k$-means++ \cite{scikitlearn}, a variant of $k$-means++ by \cite{AV07}, reported to deliver very good results, e.g.~\cite{CKV13}. For completeness, pseudocode is in \Cref{appendix:pseudocode}. Since $k$-means does not enforce spatial connectivity, we post-process the output with \Cref{algo:merging_for_connectivity} to obtain a connected clustering.

\subsection{Methods for (re)Establishing Connectivity}\label{sec:connectivity}
\begin{figure*}[]
  \centering
  \begin{tikzpicture}
     \node [label=below:{\begin{minipage}{5cm}(a) Before merging \end{minipage}}] at (0,0) {\includegraphics[width=0.42\textwidth]{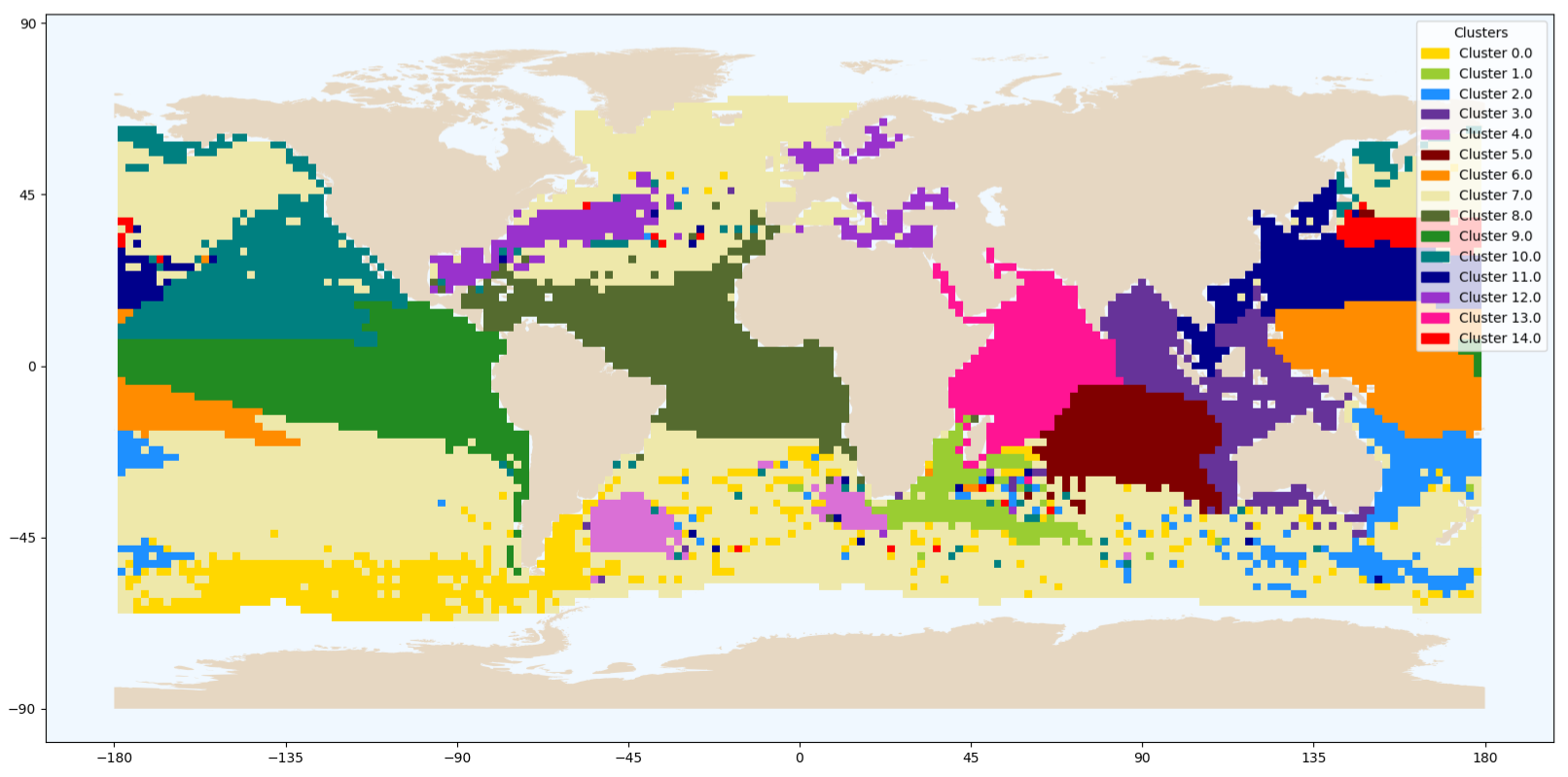}};
     \node [label=below:{(b) After merging}]at (8.6,0) {\includegraphics[width=0.42\textwidth]{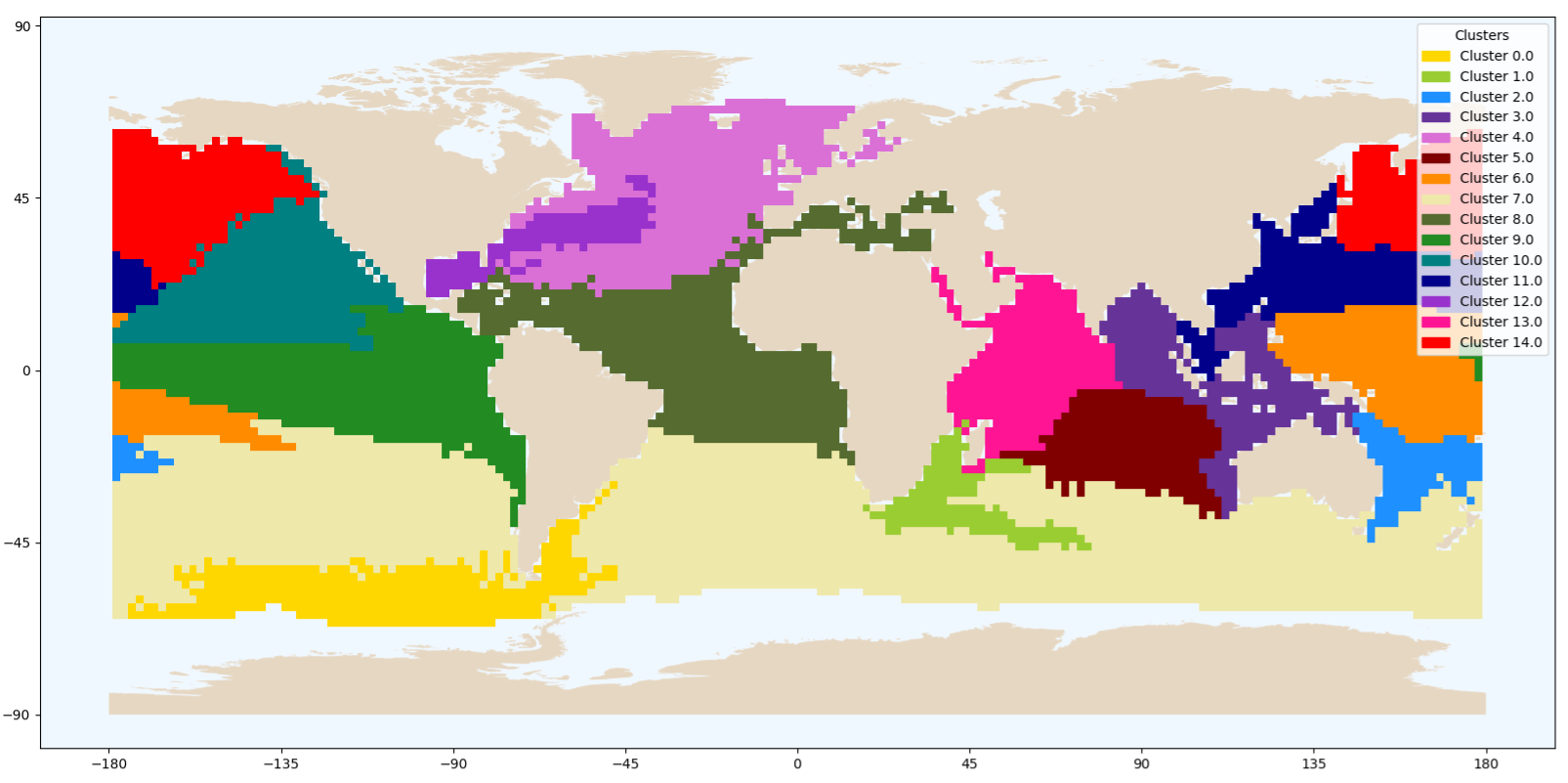}};\end{tikzpicture}
  \caption{Application of \Cref{algo:merging_for_connectivity} to an assignment of grid points for 15 clusters. The merging step enforces connectivity by reassigning disconnected subregions, prioritizing spatial coherence over original cluster labels.}
  \label{fig:merging}
\end{figure*}
Assigning points to best-fit subspaces in Step \ref{step:assignment2} of Algorithm~\ref{alg:sc} can violate the connectivity constraint. 
We evaluate the following options: 
\textbf{(i)} Reestablish connectivity after each iteration via Algorithm~\ref{algo:merging_for_connectivity} (\emph{IterMerge}), \textbf{(ii)} establish connectivity once at the end via the same algorithm (\emph{PostMerge}), \textbf{(iii)} apply a filtering step that improves, but does not strictly enforce connectivity (\emph{SmoothMerge}), and finally \textbf{(iv)} restrict Step~\ref{step:establish_conn} to maintain connectivity throughout (\emph{IntegratedConn}). For an overview of the different options, see \Cref{tab:connectivity-strategies}. We discuss each option in detail below.

\begin{table}[h]
  \centering
  \small
  \setlength{\tabcolsep}{3pt}
  \renewcommand{\arraystretch}{1.1}

  \begin{tabular}{
    p{0.25\columnwidth}
    p{0.40\columnwidth}
    p{0.23\columnwidth}}
    \hline
    \textbf{Method} &
    \textbf{During iteration} &
    \textbf{Final repair} \\
    \hline
    IterMerge       & Component merging  & No  \\
    PostMerge       & None               & Yes \\
    SmoothMerge     & Label smoothing    & Yes \\
    IntegratedConn  & Local reassignment & Yes \\
    \hline
  \end{tabular}

  \caption{Connectivity strategies. ``Final repair'' indicates whether
Algorithm~\ref{algo:merging_for_connectivity} is applied after the
iterative procedure terminates.}
  \label{tab:connectivity-strategies}
\end{table}

\subsubsection{Merging Connected Components}
To restore connectivity (Alg.~\ref{algo:merging_for_connectivity}), we form the induced subgraph on the neighborhood graph $N$ by retaining only edges whose endpoints share a cluster label. The connected components of this subgraph are contiguous subregions of clusters (see \Cref{fig:merging}(a)). 
We then repeatedly select the smallest component and merge it with an
adjacent cluster whose fitted subspace is preferred by most of its
points, until $k$ components remain
(see \Cref{fig:merging}).
If a reassignment reduces the number of components below $k$, the connected components can be split along edges of a spanning tree until
exactly $k$ components remain. All ties in $\arg\min$ and $\arg\max$ expressions are broken
arbitrarily but consistently. The procedure is motivated by the expectation that an unconstrained assignment typically consists of several substantial contiguous
regions together with smaller disconnected fragments. More details are in \Cref{appendix-merging}, and the proof of~\Cref{thm:merging_runtime} is in~\Cref{appendix:runtime-proof}.
\begin{algorithm}[]
\caption{Merging Connected Components}
\label{algo:merging_for_connectivity}
\begin{algorithmic}[1]
\STATE \textbf{Require:}\enspace
Possibly disconnected clusters $C_1,\dots,C_k$,
subspaces $S_1,\dots,S_k$, the connectivity graph $N$

\STATE \textbf{Ensure:}\enspace
Connected clustering $\tilde{C}_1,\dots,\tilde{C}_k$

\STATE Construct the graph $G_C = (V,E')$ where
\label{step-merging-graph}
$E' = \{ \{u,v\} \mid \{u,v\} \in E \text{ and }$
$\exists i: u,v \in C_i \}$

\STATE Compute set $Z$ of connected components of $G_C$
\label{step-merging-cc}

\STATE Compute the component graph $G_{CC} = (Z,E'')$ where
\label{step-merging-gcc}
$E'' = \{ \{A,B\} \mid \exists u \in A,\exists v \in B,$
$A,B \in Z,\{u,v\} \in N \}$

\WHILE{$|Z| > k$}
\label{step-merge-loop-start}

    \STATE Compute
    $A_{\min} \in \arg\min_{A \in Z}|A|$
    \label{step-merge-min}

    \FOR{$x \in A_{\min}$}

        \STATE Compute the winner
        \label{step-merge-winner}
        $w(x) = \arg\min \{ d(x,S_i) \mid i \in [k]$
        $\text{ s.\,t. } \exists B \subset C_i, B \in Z,$
        $\{A_{\min},B\} \in E'' \}$

    \ENDFOR

    \STATE Compute
    $i^w = \arg\max_i
    |\{x \in A_{\min} \mid w(x)=i\}|$
    \label{step-merge-clusterwinner}

    \STATE Reassign all points in $A_{\min}$ to cluster $C_{i^w}$

    \STATE Update graphs $G_C$ and $G_{CC}$
    \label{step-merge-loop-end}

\ENDWHILE

\STATE For each $A \in Z$ create a cluster $\tilde{C}$ and output these
\label{step-merge-output}

\end{algorithmic}
\end{algorithm}
An optimized implementation computes the initial connected components by BFS/DFS, precomputes point-to-subspace rankings and component score vectors, and maintains the smallest component in a min-heap.
Component merges use union--find with union-by-rank and path compression, while winner lookups and score-vector updates cost $O(k)$ per iteration; these ingredients yield the bound below.
\begin{restatable}{theorem}{thmruntime}
\label{thm:merging_runtime}
Algorithm~\ref{algo:merging_for_connectivity} can be implemented in
$O(|E|+n(k+\log n))$
time. Consequently, its running time is
$O(n(k+\log n))$ on bounded-degree graphs, including the grid graphs used in our application, and $O(nk)$ whenever $k=\Omega(\log n)$.
\end{restatable}

\subsubsection{Improving Connectivity via Filtering}\label{sec:filter_for_conn} 
Enforcing connectivity at every step has the side effect of a fluctuating cost (the reassignment decreases cost, while enforcing connectivity raises it).
For \emph{SmoothMerge}, we encourage spatial coherence after each reassignment step using Gaussian weighted label voting. For each point
$x$, nearby grid points contribute to the support of their current
cluster label, with weights decreasing with their Haversine distance
from $x$. We assign $x$ the label with maximum weighted support.
This operation generally reduces fragmentation but does not guarantee
connectivity, so Algorithm~\ref{algo:merging_for_connectivity} is
applied after the iterative procedure terminates. Full details are
given in \Cref{appendix-filtering-connectivity}.

\subsubsection{Integrating Connectivity Constraints}
We also evaluate incorporating local connectivity information directly
into the reassignment step, rather than relying exclusively on
connectivity repair afterward.
For this we replace Step~\ref{step:assignment1} and~\ref{step:assignment2} of \Cref{alg:sc} with \Cref{alg:integrated_conn}.
\begin{algorithm}
\caption{Integrated connectivity}\label{alg:integrated_conn}
    \begin{algorithmic}
    \FOR{each point $x \in P$}
     \FOR{each $i = 1$ to $k$ such that there exists $y \in C_i$ with $y \in \delta_N(x) \cup \{x\}$}
        \STATE Compute $d(x, S_i) = d(x - \mu_i, \operatorname{span}(v_1, \ldots, v_{m'}))$
    \ENDFOR
    \STATE Assign $x$ to the cluster $C_i'$ that minimizes $d(x, S_i)$
\ENDFOR
\end{algorithmic}
\end{algorithm}

The sole change is a locality restriction: $x$ may switch from $C_i$ to $C_j$ only if a neighbor of $x$ in $N$ lies in $C_j$ and $d(x, S_j) < d(x,S_i)$. Among these
candidate labels, it chooses the cluster whose fitted subspace is
closest. This locality restriction encourages spatially coherent
assignments but does not guarantee that every cluster remains
connected. We therefore apply the final merging step in
Step~\ref{step:final_conn}.

\section{Experimental Evaluation}\label{sec:exp}
We evaluate our Connected Subspace Clustering framework on global
sea level anomaly data, with the goal of identifying contiguous
regions exhibiting internally similar temporal variability.
\subsection{Dataset, Preprocessing and the Connectivity Graph}
We use the Copernicus Marine Service (CMEMS) global gridded
sea level anomaly product~\cite{altimetry_data}. We choose the CMEMS product because its enhanced spatial sampling and time-step accuracy are well suited to identifying regional
spatiotemporal patterns. 
This dataset contains monthly \emph{sea level anomaly} (SLA) values on a $0.25^{\circ} \times 0.25^{\circ}$ regular, rectilinear grid ($720$ latitudes and $1440$ longitudes). We restrict the analysis to grid cells with complete SLA observations throughout the study period. Owing to the dataset's validity mask, this excludes land, many coastal cells, and parts of the high latitudes, leaving $532{,}783$ valid ocean grid points. For tractability, all reported clustering experiments use this grid coarsened to a $2^\circ \times 2^\circ$ resolution, yielding $8{,}160$ grid points.

Following the preprocessing strategy of
\cite{TM14,EMastersthesis}, we apply a temporal boxcar filter
(half-width seven months; approximately $90\%$ gain at a
24-month period) and a spatial Gaussian filter based on Haversine
distance, thereby accounting for the Earth's curvature. The spatial filter is applied independently to each monthly field
using normalized Gaussian weights with half-width
$h=500$, truncated at a radius of $3\sigma$, where
$\sigma=h/1.178$. Masked grid cells are excluded from the weighted
average. All methods are evaluated on both unfiltered and filtered
data.
\label{sec:preprocessing} 
More on the dataset can be found in \Cref{sec:data-set-choice} and further preprocessing details are described in \Cref{appendix-preprocessing}.

\subsubsection{Connectivity Graph} \label{sec:grid_graph}
We map the combinatorial structure of the data to a connectivity graph $N=(V,E)$ with one vertex per grid point and edges for each of the four neighbors in the grid. Because masked cells are omitted, $N$ is a grid graph with holes. 

To account for the periodicity of longitude, retained cells in the
westernmost and easternmost grid columns are treated as adjacent when
they occur at the same latitude. Thus, $N$ can be viewed as a subgraph
of a two-dimensional grid on a cylinder. Optionally, diagonal edges can be added; all algorithms described in this paper support either choice. In the evaluation the four-neighbor graph is used.
The set of neighbors of a vertex $x\in V$ is $\delta_N(x) = \{y \mid \{x,y\} \in E\}$.

\subsubsection{Hardware and Source Code}
We implemented all algorithms in Python. The
anonymized source code is available via~\url{https://anonymous.4open.science/r/subspace_clustering-C00F}. The repository README provides complete environment setup, configuration, and execution instructions for reproducing the reported experiments. The experiments were conducted on a single machine with an AMD Ryzen 7~PRO~7840U processor (8 cores, 16 threads) and 30~GiB of RAM.

\subsubsection{Method Names}
We use the abbreviations \emph{Agglo-ST}, \emph{Conn-Agglo-Euc},
\emph{Conn-Agglo-ST}, \emph{Conn-KMeans++}, and \emph{Conn-Ward}
for the five initialization methods from \Cref{sec:initialclustering}.
The four connectivity strategies are \emph{IterMerge},
\emph{PostMerge}, \emph{SmoothMerge}, and \emph{IntegratedConn};
see \Cref{tab:connectivity-strategies}.

\subsection{Initializer Quality under $k$-Means and Subspace Objectives}
\begin{figure}[t]
\centering \vspace{-0.8cm}
    \begin{tikzpicture}
        \node[label=below:{(a) $k$-means cost function}] at (0,0) {\includegraphics[width=0.4\textwidth]{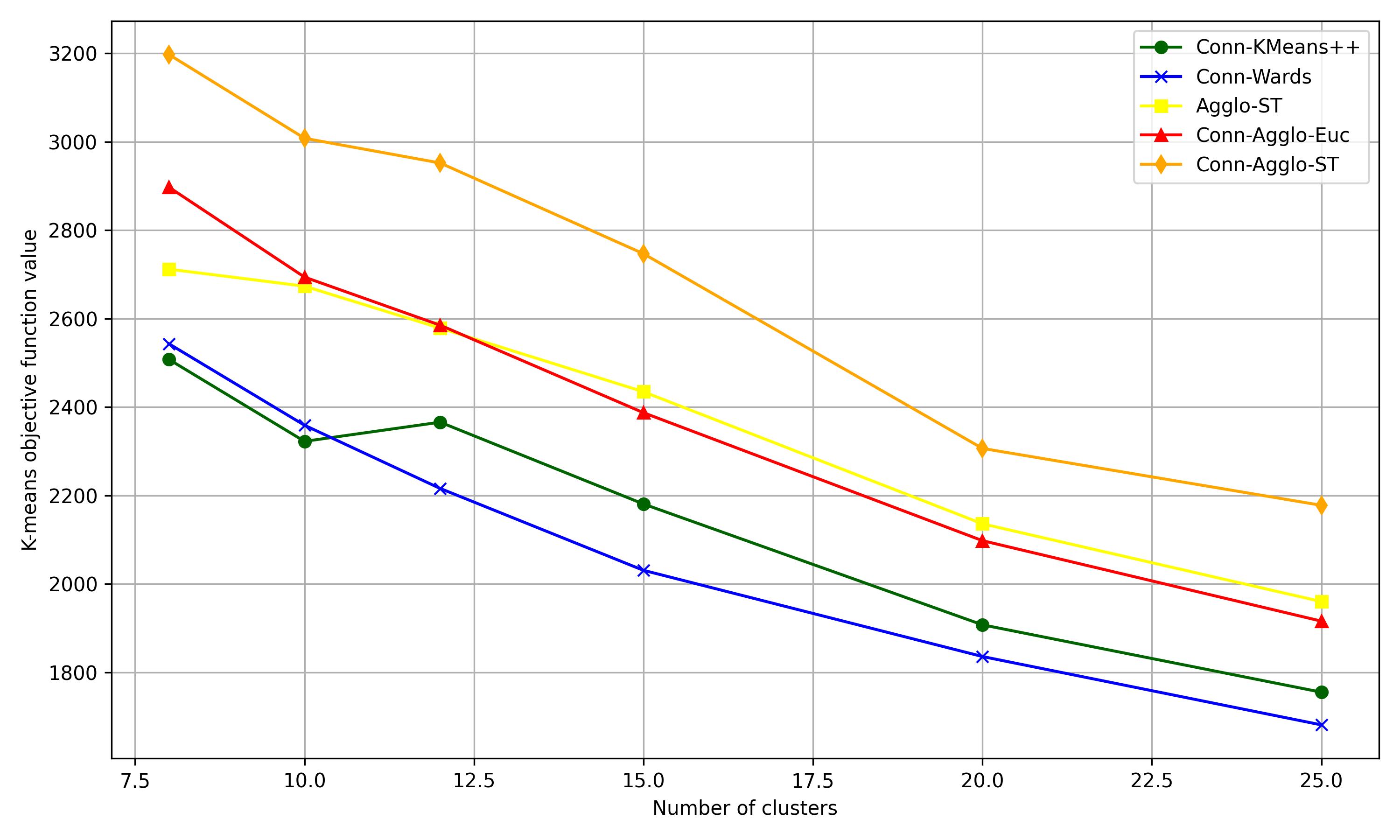}};
        \node[label=below:{(b) Subspace clustering cost for $m'=15$}] at (0,-4.8) {\includegraphics[width=0.4\textwidth]{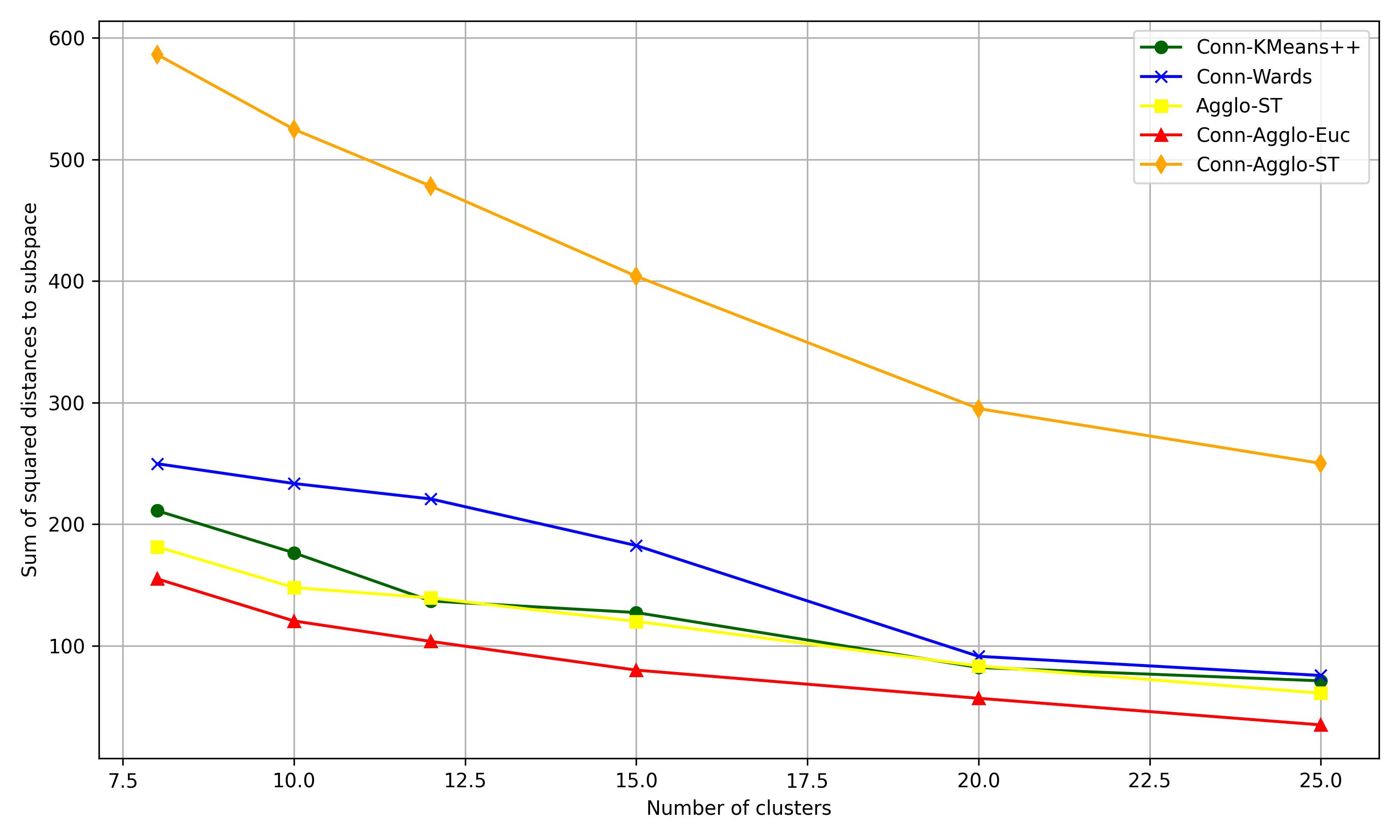}};
    \end{tikzpicture}
    \caption{Comparison of our five initial clustering methods evaluated with two different cost functions.\label{fig:cost-evaluation-k-means-sc}}
    \vspace{-6pt}
\end{figure}
The $k$-means problem is a special case of subspace clustering with $m'=0$, and we use $k$-means as an initial clustering. When evaluating the quality of the initial clusterings, we made an interesting observation: For Figure~\ref{fig:cost-evaluation-k-means-sc}, we ran all five initialization methods with an increasing number of clusters.
For each initialization method, we evaluate the resulting partition
before applying the iterative subspace-clustering procedure. We score
each fixed partition once using the $k$-means objective and once using
the rank-$15$ subspace-reconstruction objective. 
Under the $k$-means objective, Conn-Ward and Conn-KMeans++ outperform the other approaches, see Figure~\ref{fig:cost-evaluation-k-means-sc}(a).
However, under the subspace clustering objective, the ranking flips (Figure~\ref{fig:cost-evaluation-k-means-sc}(b)). Notably, agglomerative connected clustering outperforms all other methods, and the tailored method by Thompson and Merrifield (\emph{Agglo-ST}) is now competitive with \emph{Conn-KMeans++}. Conn-Ward has comparatively high subspace-reconstruction cost, particularly for smaller values of $k$.
This trend is consistently observed across multiple configurations (see~\Cref{sec:results_appendix}). 
These results show that initializer quality is strongly
objective-dependent: methods that perform well under the $k$-means
objective do not necessarily provide good starting partitions for
the rank-$m'$ reconstruction objective used in our downstream regional
PCA analysis. We examine the geographic and oceanographic interpretation of the
resulting regions in \Cref{sec:oceanographic-interpretation}.

\subsection{Robustness / Filtering}\label{sec:robustness}
\begin{figure}[h]
  \centering
  \begin{tikzpicture}
     \node [label=below:{(a) Clustering on filtered input data.}] at (0,0) {\includegraphics[width=0.4\textwidth]{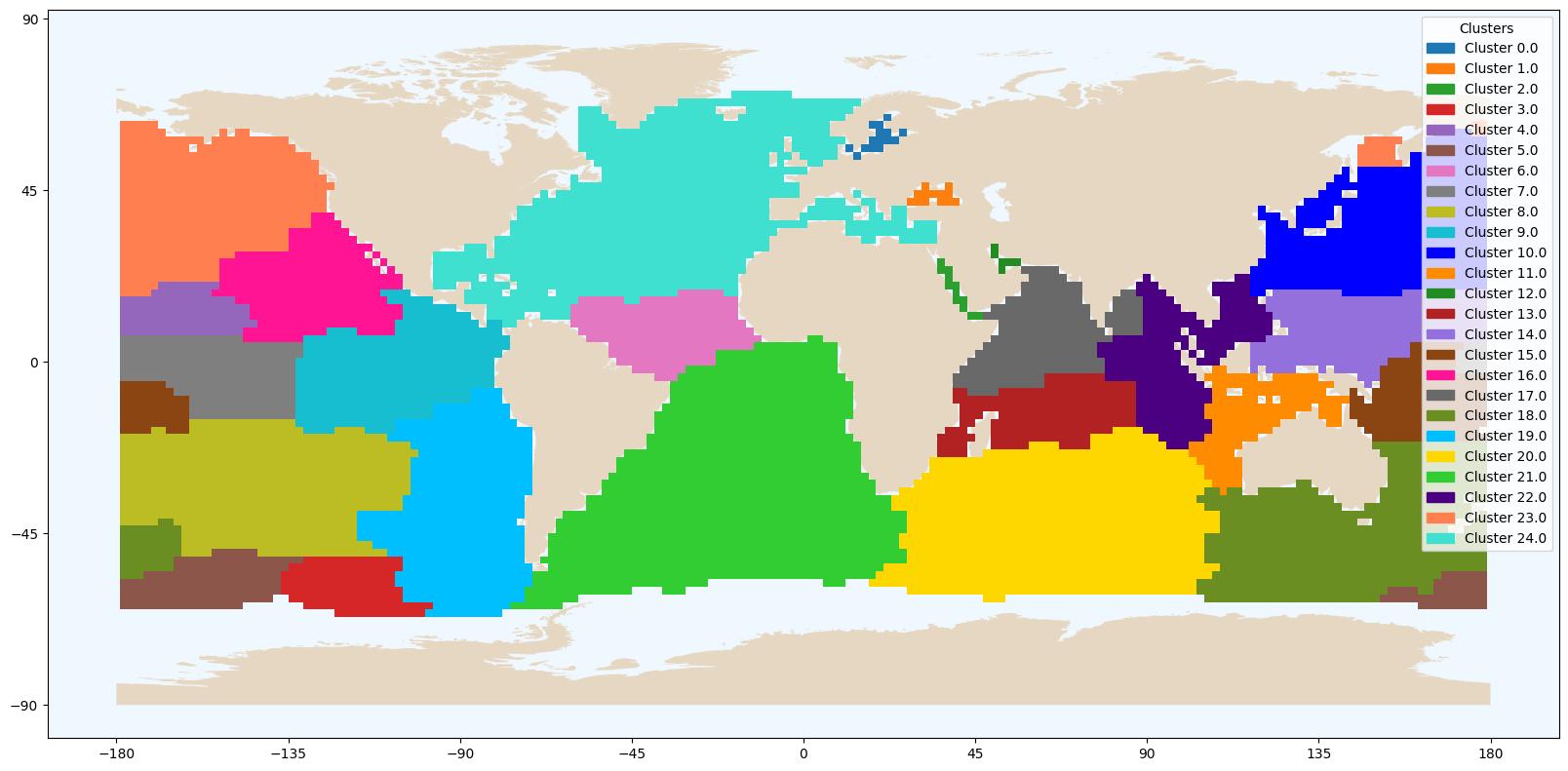}};
     \node [label=below:{(b) Clustering on unfiltered input data.}]at (0,-5) {\includegraphics[width=0.4\textwidth]{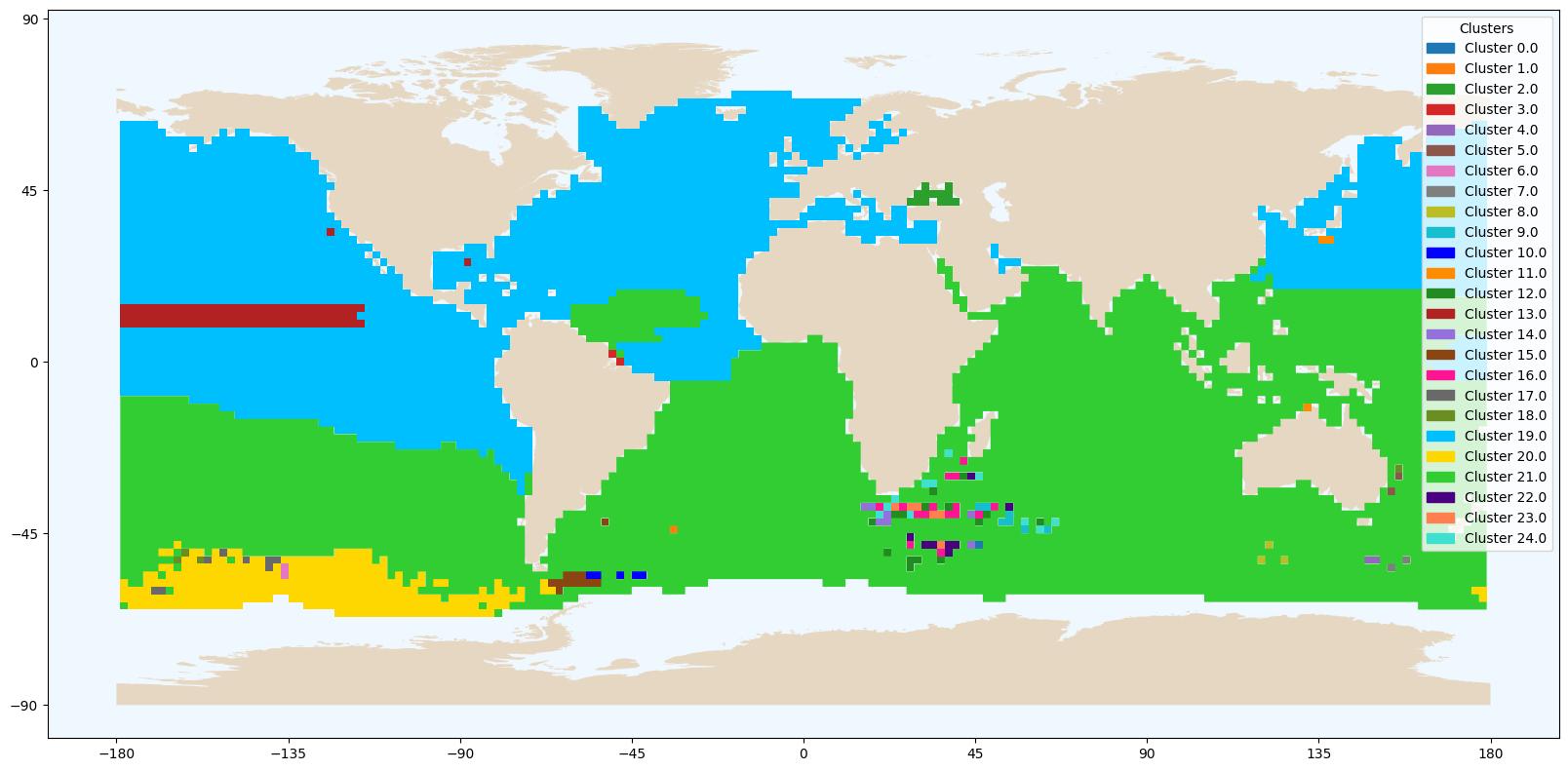}};\end{tikzpicture}
  \caption{Agglo-ST initialization for $k=25$ on
(a) filtered and (b) unfiltered SLA data. Without filtering,
Agglo-ST produces one dominant cluster and several small regions,
some of which are too small to fit the requested subspace dimension.}
  \label{fig:thompson_filtering}
\end{figure}
We next examine how spatial and temporal filtering affect the performance of the
initialization methods.
Agglo-ST is highly sensitive to preprocessing. On unfiltered data, it
frequently degenerates into one dominant cluster and several very small
clusters (\Cref{fig:thompson_filtering}). Some of these clusters contain too few points for our PCA
implementation to fit an $m'$-dimensional subspace, so the resulting
initialization cannot be used by Conn-Subspace for those parameter
settings.

Even for larger $k$, \emph{Agglo-ST} consistently underperforms (\Cref{tab:15_unfiltered,tab:20_unfiltered,tab:25_unfiltered}); e.g., at $k=20$, it is at least 8.98\% higher than the next worst clustering method, \emph{Conn-Ward}.
The other four initialization methods produce usable starting
partitions on both filtered and unfiltered data for all tested
configurations.

\subsection{Comparison of Connectivity Strategies}\label{sec:connectivity-evaluation}
\begin{figure*}[t]
  \centering
  \begin{subfigure}[t]{0.4\textwidth}
    \centering
    \includegraphics[width=\linewidth]{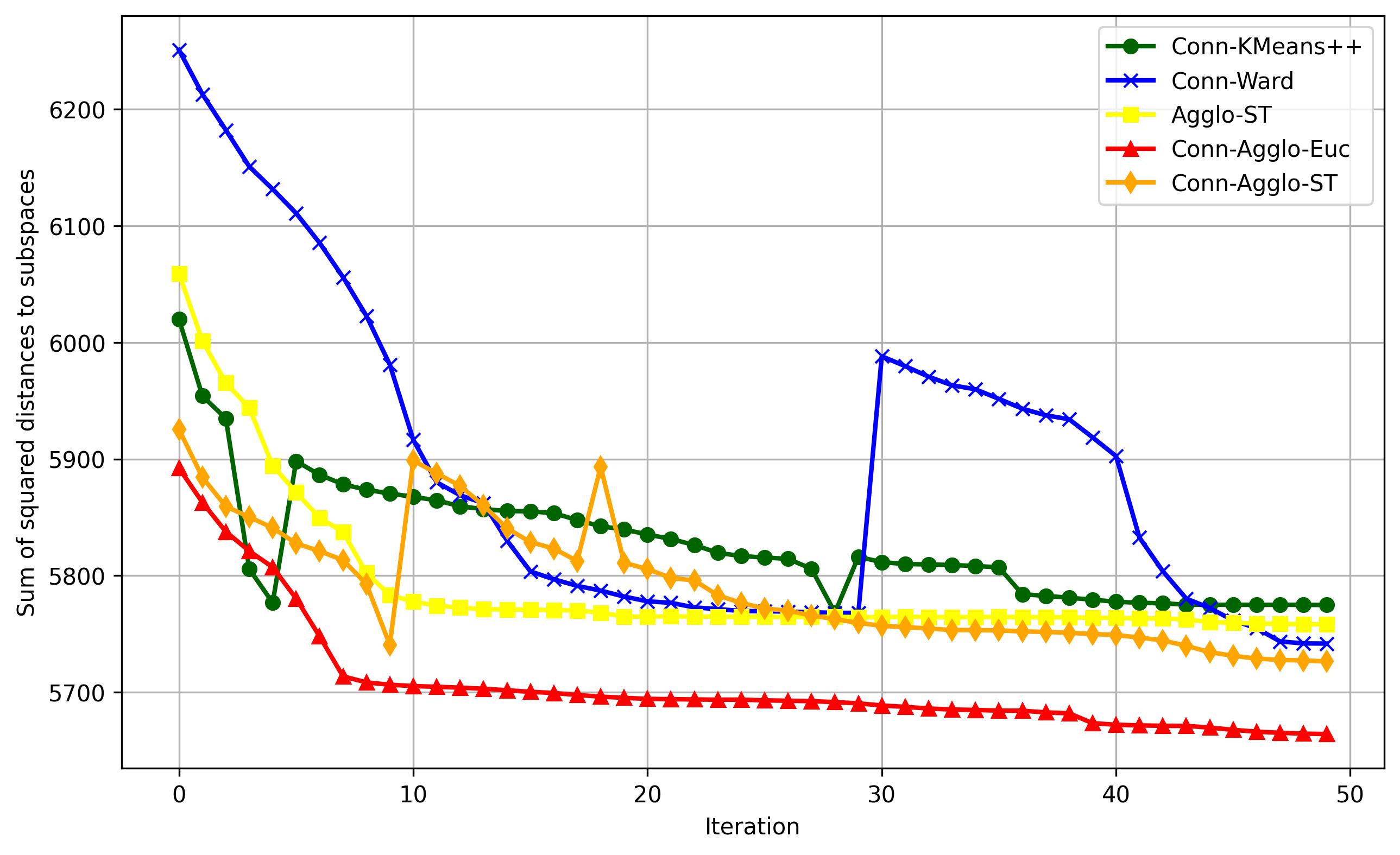}
    \caption{IterMerge}
  \end{subfigure}\hfil
  \begin{subfigure}[t]{0.4\textwidth}
    \centering
    \includegraphics[width=\linewidth]{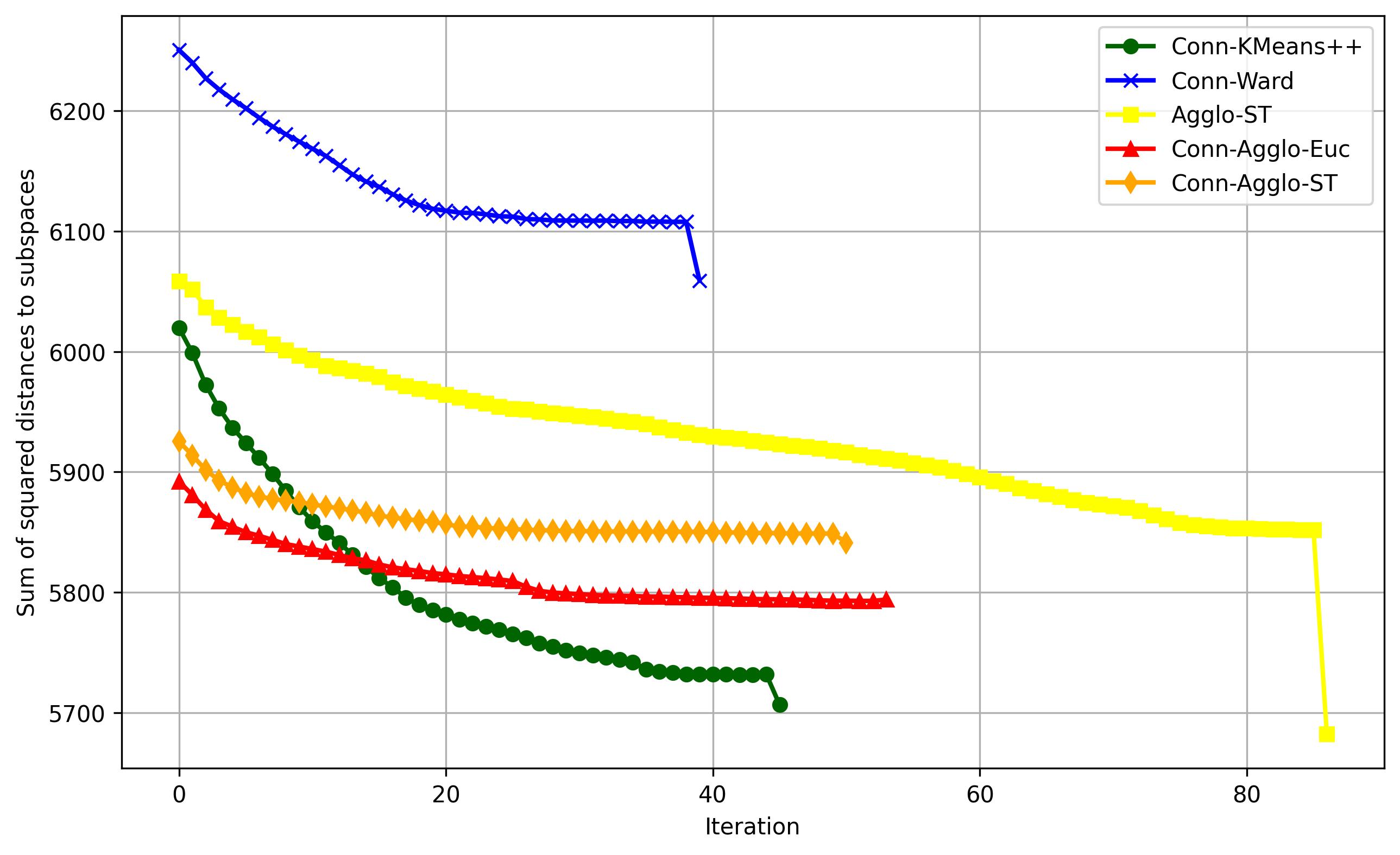}
    \caption{IntegratedConn}
  \end{subfigure}
  \begin{subfigure}[t]{0.4\textwidth}
    \centering
    \includegraphics[width=\linewidth]{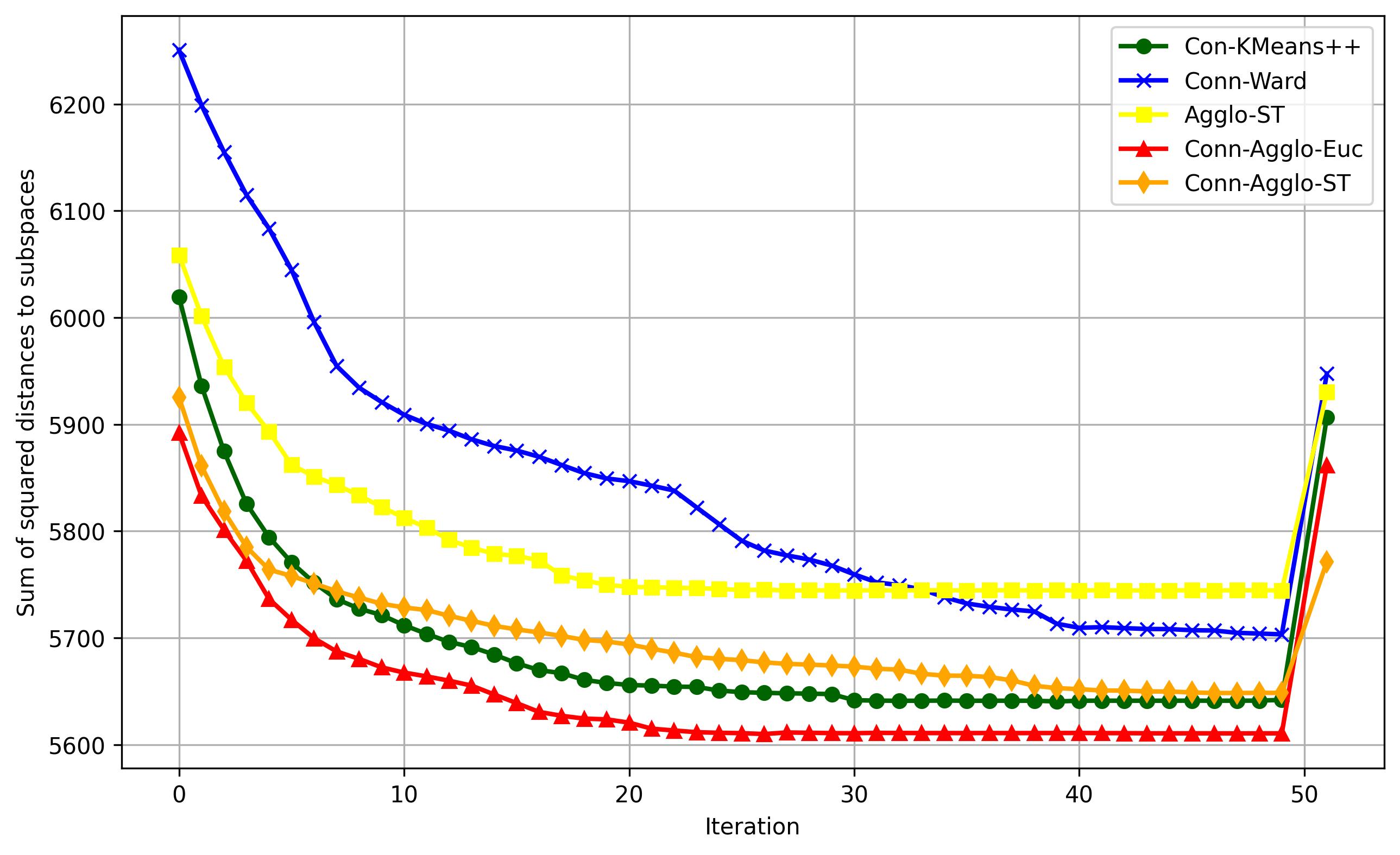}
    \caption{SmoothMerge}
  \end{subfigure}\hfil
  \begin{subfigure}[t]{0.4\textwidth}
    \centering
    \includegraphics[width=\linewidth]{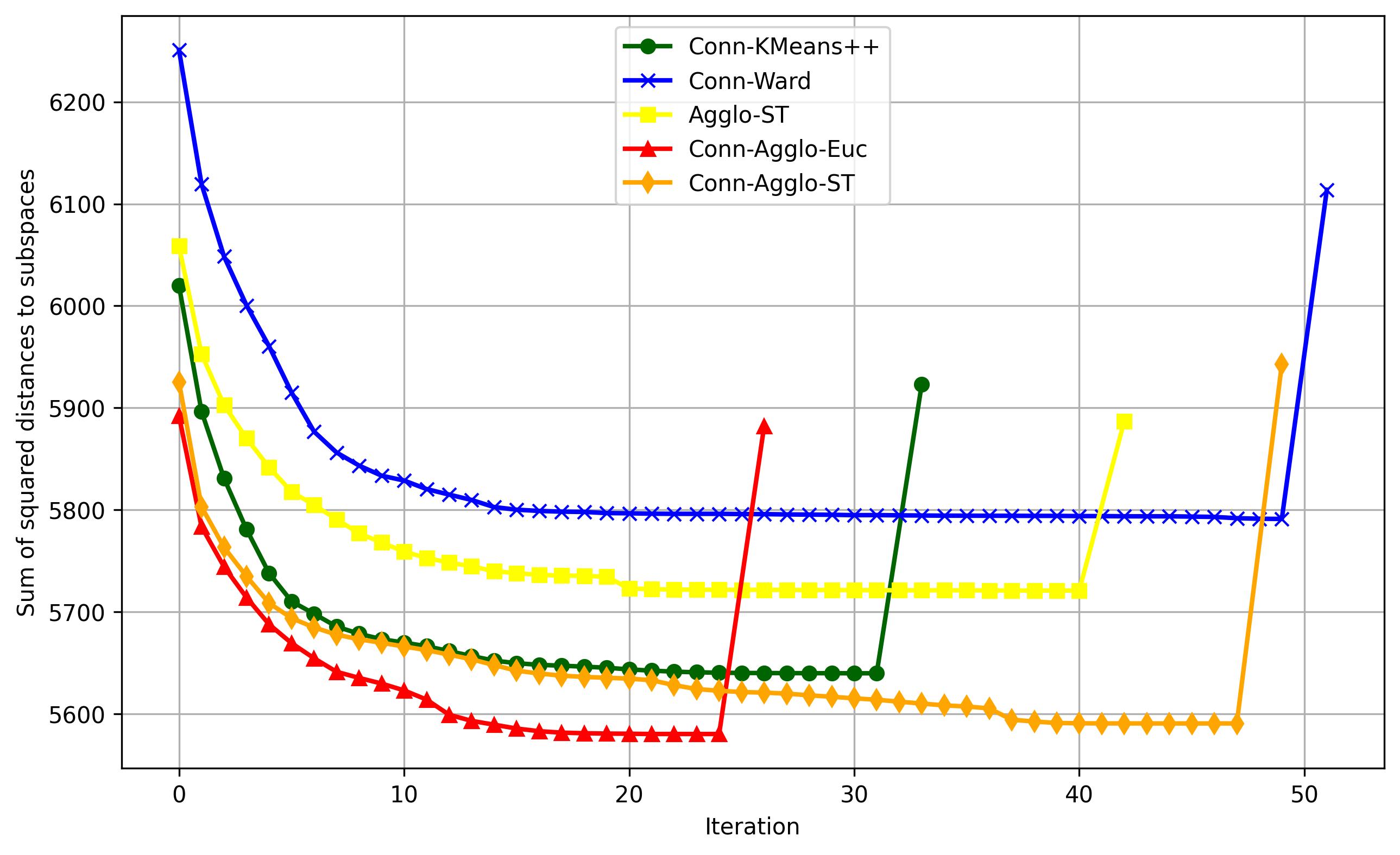}
    \caption{PostMerge}
  \end{subfigure}
  \caption{Cost function trajectories for the four connectivity heuristics over algorithm iterations. IterMerge and IntegratedConn show relatively continuous improvements, while SmoothMerge and PostMerge show late-stage cost increases due to delayed enforcement of connectivity.}
  \label{fig:connectivity_approaches}
\end{figure*}
%%%%%%%%%%%%%%%%%%%% filtered 8 cluster, smaller
\begin{table}[]
\centering
\caption{Initial and final subspace-reconstruction objectives for IterMerge on filtered data with $k=8$ and $m'=15$.
For this representative setting, IterMerge improves the objective value for all five initialization methods.}\label{tab:8_mini}
\small
\setlength{\tabcolsep}{4pt}
\begin{tabular}{@{}lrr@{}}
\toprule
& \multicolumn{1}{c}{Initial objective}
& \multicolumn{1}{c}{Final objective} \\
\midrule
Agglo-ST       & 6058.8 & 5681.5 \\
Conn-Agglo-Euc & 5892.1 & 5656.5 \\
Conn-Agglo-ST  & 5925.4 & 5716.8 \\
Conn-KMeans++  & 6019.4 & 5742.0 \\
Conn-Ward      & 6250.7 & 5728.4 \\
\bottomrule
\end{tabular}
\end{table}
We investigate how different strategies for reestablishing connectivity (see \Cref{sec:connectivity}) during subspace clustering impact the objective.
Therefore, we compare the four connectivity strategies across five
initializations, four subspace dimensions
$m'\in\{5,10,15,30\}$, four cluster counts
$k\in\{8,15,20,25\}$, and both filtered and unfiltered data.
This yields 160 base configurations. Each base configuration is evaluated with all four connectivity strategies, resulting in 640 runs.
For each base configuration, we compare the final
subspace-reconstruction objective attained by the four strategies.

Overall, \emph{IterMerge} attains the lowest final objective value in 118 of the 160 base configurations ($73.75\%$). \emph{SmoothMerge} and \emph{IntegratedConn} follow, \emph{PostMerge} is worst. 
In two of the 160 base configurations ($1.25\%$), Agglo-ST on
unfiltered data produces clusters that are too small to fit the requested subspace dimension; no refined solution is therefore available for these configurations (\Cref{sec:robustness}).
In addition to outperforming the other connectivity strategies most
frequently, \emph{IterMerge} lowers the objective value relative to the
corresponding initialization in every valid configuration.
\Cref{tab:8_mini} shows one representative setting; complete results
are given in \Cref{sec:results_appendix}.
As illustrated in \Cref{fig:connectivity_approaches}, because connectivity is enforced only at the end, \emph{SmoothMerge} and especially \emph{PostMerge} show a final cost spike. The iterative smoothing step in \emph{SmoothMerge} mitigates this effect, but does not remove it. 
\emph{IterMerge} and \emph{IntegratedConn} behave more like Lloyd's method, where costs improve through the iterations.
Even though the strict connectivity enforced by \Cref{algo:merging_for_connectivity} in \emph{IterMerge} can lead to a temporary increase in the cost function after the merging, it is the approach that most successfully balances the connectivity constraint with the subspace clustering objective.

\subsection{Comparison with Connected and Subspace-Clustering Baselines}
\begin{figure}[h]
\centering 
  \includegraphics[width=0.4\textwidth]{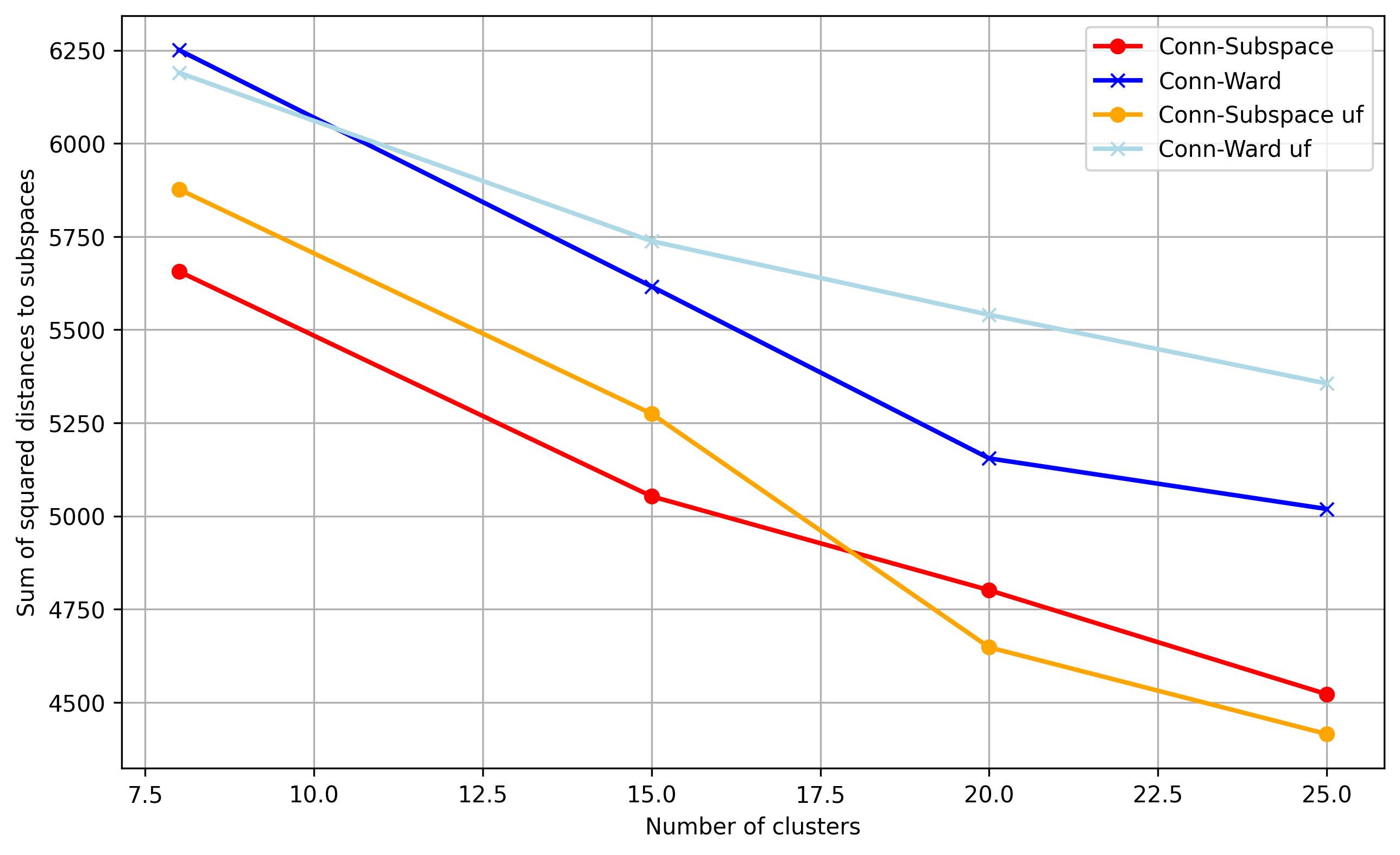}
    \caption{Subspace-reconstruction objective as a function of the targeted number of clusters $k$ for Conn-Ward and Conn-Subspace initialized with Conn-Agglo-Euc and using IterMerge. Results are shown for filtered and unfiltered data.}\label{fig:comp_CS_CW}
\end{figure}
Across the valid configurations, Conn-Subspace consistently lowers
the objective value relative to its initialization, even from suboptimal starts; the largest reduction is $20.3\%$ (Agglo-ST initialization, unfiltered data, $k=15$, $m'=30$).

Among the configurations compared in \Cref{fig:comp_CS_CW},
Conn-Agglo-Euc followed by Conn-Subspace with IterMerge attains the
lowest objective value for every tested $k$, on both filtered and
unfiltered data.
\emph{Conn-Ward} is the connected clustering method most readily available prior to this work. While \emph{Conn-Ward} relies on a hierarchical agglomeration based on local variance, it lacks the ability to revise cluster assignments after initial formation. In contrast, Conn-Subspace can iteratively correct misassignments.

\subsubsection{Comparison with Subspace-Clustering and
Spatial-Spectral Methods}
We additionally compare against two unconstrained subspace-clustering methods, SSC-OMP and EnSC, and two spatial-spectral methods from hyperspectral image segmentation, EGCSC and EKGCSC. None of these methods guarantees that each output cluster induces a single connected component in the prescribed graph. Implementation and parameter details are given in \Cref{sec:additional-experiments}.
SSC-OMP and EnSC both yield higher subspace-reconstruction
objectives and more fragmented partitions than our method. For
example, on unfiltered data with $k=25$, SSC-OMP attains an objective value of $14{,}473.1$, compared with $2{,}947.9$ for Conn-Agglo-Euc followed by Conn-Subspace with IterMerge at $m'=30$.
Despite only 25 targeted clusters, SSC-OMP produces up to $1{,}966$ connected components; EnSC is less fragmented (at most $222$) but its objective value stays consistently above ours (see \Cref{tab:number_of_ssc-omp-ensc,sec:results_appendix}).

EGCSC and EKGCSC yield less fragmentation than SSC-OMP and
EnSC and attain lower objective values, yet in every tested configuration they still produce more connected components than requested clusters and remain worse than our method (\Cref{tab:number_of_conn_comps_hsi,sec:results_appendix}).
In contrast, our method returns exactly $k$ connected clusters.

\subsection{Implications for Sea Level Research}\label{sec:oceanographic-interpretation}
Combining subspace clustering with \emph{Empirical Orthogonal Function} (EOF) analysis isolates regionally coherent sea level variability that global EOFs can blur.
We illustrate this using the best-performing configuration, Conn-Agglo-Euc followed by Conn-Subspace with IterMerge, on filtered data with $k=15$ (see \Cref{fig:clustering_geodesy}). 
\begin{figure}[h]
    \centering
    \includegraphics[width=0.4\textwidth]{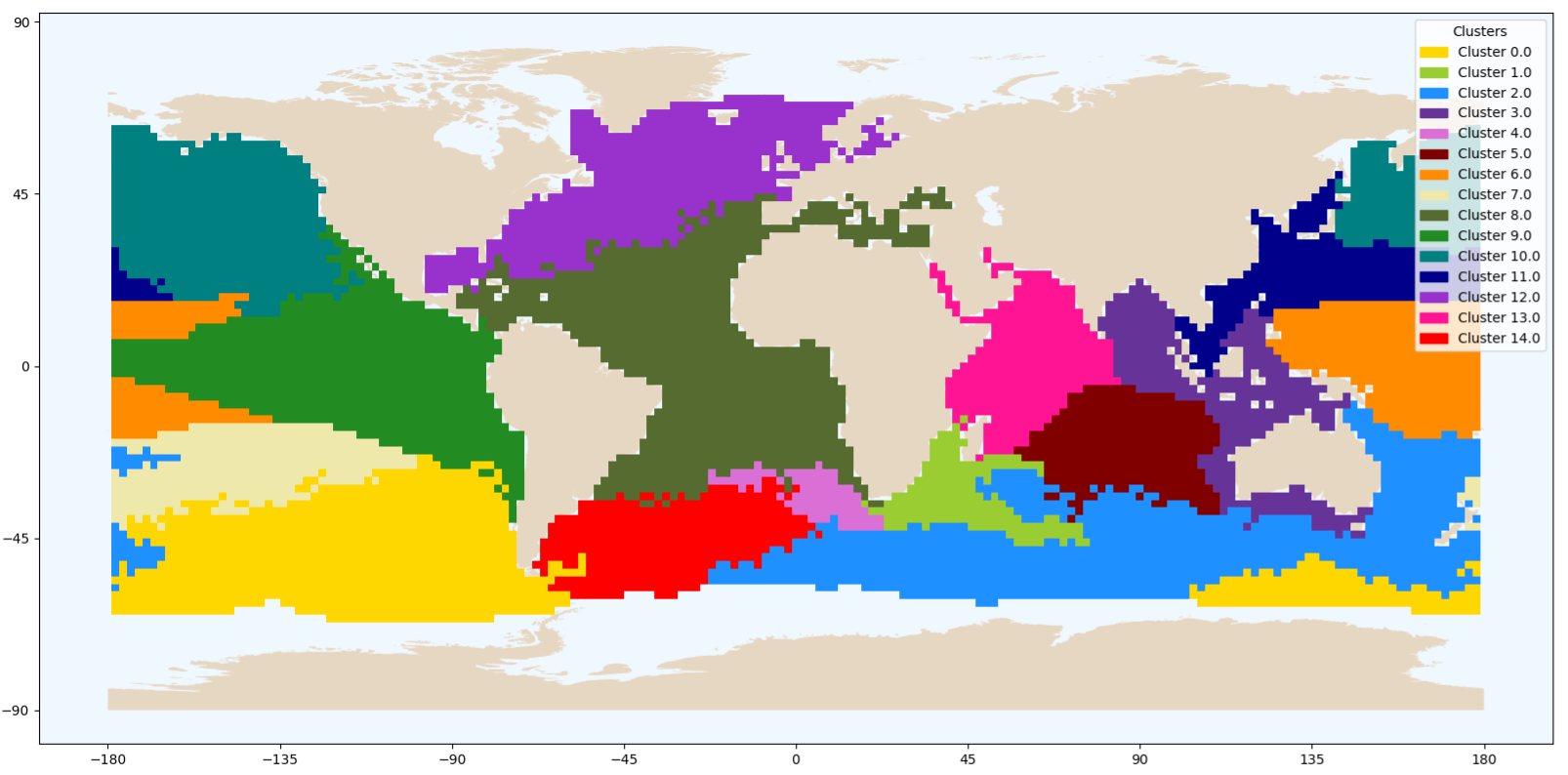}
    \caption{Clustering achieved by Conn-Subspace with IterMerge, for $k=15$.}
    \label{fig:clustering_geodesy}
    \vspace{-6pt}
\end{figure}
For example, cluster 1 directly relates to the Agulhas western boundary current, where water is transported southward along the South African coast and then diverted eastward back into the Indian Ocean.
Similarly, cluster 12 is related to the Gulf Stream region in the Atlantic Ocean. While the
trend and annual cycle dominate the global EOF decomposition
(Appendix, \Cref{fig:eofs}(a),(b)), regional decomposition reveals locations
in which other modes dominate.
\begin{figure}[h]
    \centering
  \begin{subfigure}{0.4\textwidth}
    \includegraphics[width=\textwidth]{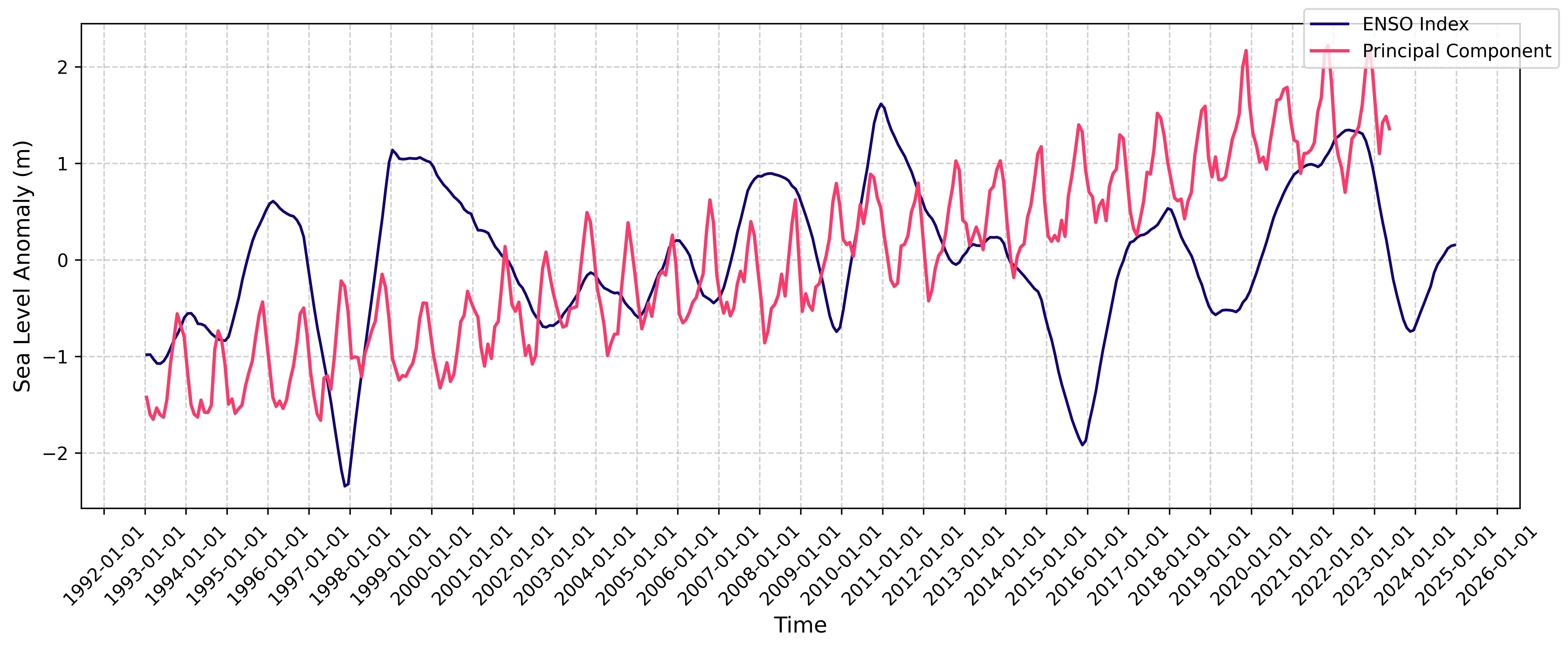}
    \caption{Signal from cluster 8 (red) together with the ENSO index (blue).}
    \label{fig:cluster8}
  \end{subfigure}
  \begin{subfigure}{0.4\textwidth}
    \includegraphics[width=\textwidth]{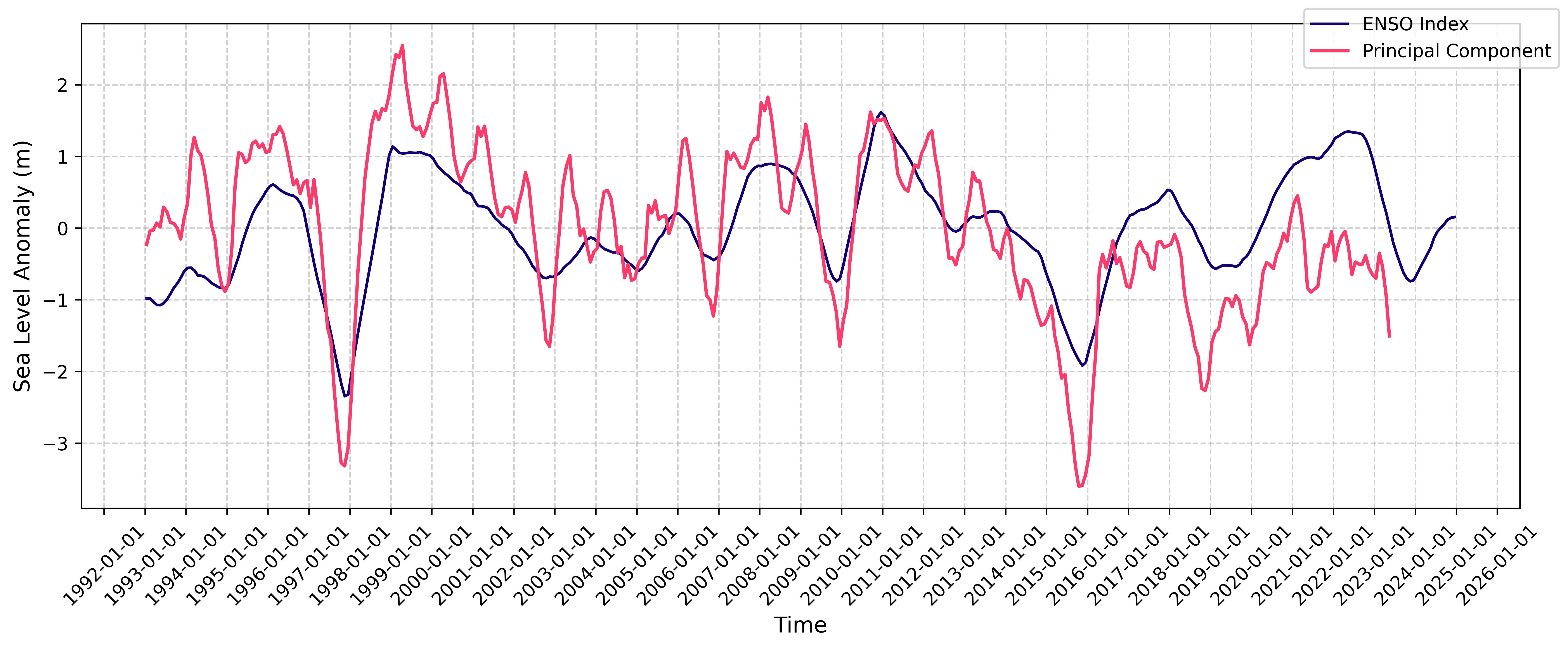}
    \caption{Signal from cluster 9 (red) together with the ENSO index (blue).}
    \label{fig:cluster9}
  \end{subfigure}
  \caption{ENSO index comparison} 
  \label{fig:cluster_compared_with_enso}
  \vspace{-6pt}
\end{figure}
Cluster 8 spans much of the Atlantic basin, whereas cluster 9 covers the eastern tropical Pacific. The leading principal component of cluster 9 shows clear visual correspondence with the El Niño--Southern
Oscillation (ENSO) index, wheras that of cluster 8 shows no comparable correspondence (\Cref{fig:cluster_compared_with_enso}). This illustrates how the clustering distinguishes regions according to their dominant temporal variability. In cluster 9, the annual cycle shifts to the second EOF. The leading component of cluster 13 similarly shows correspondence with the Indian Ocean Dipole (IOD) index, although the cluster extends beyond the IOD region and contains additional variability. 

This regional separation supports targeted analysis of sea level drivers and geodetic applications in which regional altimetry EOFs can improve tide-gauge reconstruction \cite{klos2019introducing,church2004}. Potential inconsistencies at cluster boundaries warrant further study.

%% file: conclusion.tex
\section{Conclusion}
We introduced \emph{Connected Subspace Clustering}, established strong inapproximability even for zero-dimensional subspaces and grid graphs with holes, and developed a scalable Lloyd-style heuristic with graph-based connectivity guarantees. Experiments on global sea level anomaly data show that the bottom-up merging (IterMerge) routine is strongest of the four connectivity strategies and that Conn-Subspace consistently improves its initialization. Unlike the evaluated baselines, it returns exactly $k$ connected regions while attaining substantially lower reconstruction cost. The resulting regions also expose interpretable regional modes, including ENSO-related variability.

Limitations of the proposed approach include sensitivity to the neighborhood graph and to the initialization; promising directions for the future include further study of the simultaneous optimization of the objective function and connectivity (bi-criteria optimization) and the extension to other datasets or settings.
Overall, combining the subspace clustering objective with topological information offers a practical route to achieve structure-aware clusterings.

%% file: appendix/appendix-related-work.tex
\subsection{Further related work}
\paragraph{Partitional $k$-clustering and $k$-means.}\label{related-work-k-means-and-connected-clustering}
 Our research belongs to the field of partitional metric clustering where we are given a point set $P$ from a metric space, and the goal is to partition these points into clusters while optimizing a cost function and possibly obeying side constraints. Famous partitional clustering problems are $k$-median, $k$-means or $k$-center. For us, $k$-means is specifically relevant: Given $P\subset \mathbb{R}^m$, find centers $c_1,\ldots,c_k \in \mathbb{R}^m$ to minimize the squared error when the points are replaced by their closest center, i.\,e., minimize $\sum_{x \in P} \min_{i=1,\ldots,k} ||x-c_i||^2$.
 This problem is NP-hard, both for constant $k$~\cite{ADHP09} and for constant $m$~\cite{MahajanNV12} (if both are simultaneously constant then there is a poly-time algorithm for the problem~\cite{IKI94}). The best known approximation algorithm for $k$-means achieves a factor of 5.912~\cite{CAEMN22}. 

\emph{Practical algorithms for $k$-means.} The widely known heuristic for $k$-means by Lloyd was developed in 1957~\cite{Llo57, Llo82}. A milestone result is due to \cite{AV07} who presented the $k$-means++ algorithm that outperforms Lloyd's algorithm empirically and achieves an $O(\log k)$-approximation guarantee. State-of-the-art algorithms for practical approximation improve this to an $O(1)$-worst-case guarantee~\cite{LS19}.

\paragraph{Exact and approximate algorithms for connected $k$-clustering}
\cite{GEGHBB08} and \cite{DELRRSW24} consider connected $k$-center clustering, where the objective is to minimize the maximum radius, rather than a sum of squared errors. 
\cite{GEGHBB08} show that connected $k$-center can be solved optimally when the neighborhood graph is a tree. \cite{DELRRSW24} provide an $O(\log^2 k)$ approximation for connected $k$-center with general connectivity graphs. They also consider the special case of Euclidean metrics in $\mathbb{R}^m$ where their result for connected $k$-center improves to an $O(m^{2.5})$ approximation guarantee. 

The connected $k$-means problem is a special case of the connected subspace clustering problem where $m'=0$. It is sum-based rather than the above radius-based $k$-center problem, and it uses squared Euclidean distances. 

\cite{gupta_clustering_2011} study the connected $k$-means problem under the restriction that the connectivity graph is a star-like graph. They show that the problem is NP-hard to approximate better than $\Omega(\log n)$ even for this restricted graph class and provide an $O(\log n)$-approximation for this special graph class. 

\cite{ELRRS25} show that the related metric connected $k$-median problem can be solved optimally on trees. For general connectivity graphs, they show that the problem is NP-hard to approximate  better than $\Omega(n^{1-\epsilon})$ and they give an $(n \cdot \log^2 k)$-approximation algorithm for this case.

There are no results about the special case of connectivity graphs that are grids.

%% file: appendix/appendix-remaining-proof.tex
\subsection{Proofs for hardness of connected $(k,p)$-clustering problem}\label{appendix-remaining-proof}

In this section we present our theoretical result on the computational hardness of the following problem.

\begin{restatable}{problem}{prbmain}{(connected $(k,p)$-clustering).}
    Given a metric space $(\mathbb{R}^m,d)$, $V\subset \mathbb{R}^m$,
a connected graph $G=(V,E)$ and parameters $k,p \in \mathbb{N}$.
Find a partition of $V$ into $k$ clusters $C_1,\dots,C_k$ with corresponding centers $c_1,\dots,c_k \in \mathbb{R}^m$ s.t.\ for $i \in [k]$, $C_i$ induces a connected subgraph in $G$ and 
$\sum_{i \in [k]} \sum_{x \in C_i} d(x,c_i)^p$ 
is minimal.
\end{restatable}

Note that this is a $m'$-subspace $k$-clustering problem with $m'=0$.

Concretely we will prove the following theorem.

\begin{restatable}{theorem}{thmmain}\label{thm-hardness}
    Given $k \in \mathbb{N}$ and a constant $p \in \mathbb{N}$.
    The connected $(k,p)$-clustering problem is NP-hard to approximate within factor $\Omega (|V'|^{1/2- \varepsilon})$, for any $\varepsilon > 0$ on grid graphs $G'=(V',E')$ with holes and  Euclidean $t$-dimensional space, where $t \geq k$.
\end{restatable}

To do so we first consider a discrete variant of our clustering problem that restricts the centers to originate from the set of input vertices.

\begin{problem}{(discrete connected $(k,p)$-clustering).}
Given 
a connected graph $G=(V,E)$, a distance metric $d : V \times V \rightarrow V$ and parameters $k,p \in \mathbb{N}$.
Find a partition of $V$ into $k$ clusters $C_1,\dots,C_k$ with corresponding centers $c_1,\dots,c_k \in V$ s.t.\ for $i \in [k]$, $C_i$ induces a connected subgraph in $G$ that minimizes $\sum_{i \in [k]} \sum_{x \in C_i} d(x,c_i)^p.$ 
\end{problem}
The cases $p = 1$ and $p=2$ are commonly known as $k$-median and $k$-means clustering, respectively.

We show the hardness of discrete connected $(k,p)$-clustering by reducing from the vertex disjoint paths problem in grid graphs with holes, which is known to be NP-complete~\cite{kramer1984complexity,formann1990vlsi,lynch1975equivalence}.

\begin{problem}{(vertex-disjoint paths problem).}
Given an undirected and unweighted graph $G=(V,E)$ and $k$ pairs of vertices $(s_1,t_1),(s_2,t_2),\dots,(s_k,t_k)$. Decide if there exist vertex-disjoint paths $P_1,\dots,P_k$ s.t.\ $P_i$ is an $(s_i,t_i)$-path.
\end{problem}

 Let $G = (V, E)$ with $(s_1, t_1), \ldots, (s_k, t_k)$ be the input instance for the \emph{vertex-disjoint paths} problem, where $G$ is a grid and $s_i, t_i \in V$ for all $1 \le i \le k$. 
 Since $G$ is a grid, each vertex in $V$ can be represented by a pair of integer coordinates. For $v \in V$, we write $v = (x(v), y(v))$, where $x : V \rightarrow \mathbb{N}$ and $y : V \rightarrow \mathbb{N}$. We also assume that $G$ might contain holes, i.e., $G$ is not necessarily a complete grid.

Based on $G = (V, E)$ and the $k$ node pairs we construct an instance for \emph{discrete connected $(k,p)$-clustering}. Let $m$ denote some even integer parameter to be chosen later. Then the connectivity graph $G_m=(V',E')$ is defined as follows:

We start with $(V',E')=(V,E)$ and then we apply the following modifications. For each horizontal edge $\{u,v\}\in E$ (i.e., $y(u)=y(v)$ and $|x(u)-x(v)|=1$) in the grid, we check whether there exists a pair of vertices $(s_i,t_i)$ with $x(s_i)\in\{x(u),x(v)\}$ or $x(t_i)\in\{x(u),x(v)\}$, respectively. If this is the case then we replace in $E'$ the edge $(u,v)$ by a path (in the same row $y(u)=y(v)$) with $m$ intermediate vertices. Let $V'_{(u,v)}$ be the set of vertices on this path, including $u$ and $v$ itself. If there does not exist a node $s_i$ or $t_i$ in the same column as $u$ or $v$ then we define $V'_{(u,v)}=\{u,v\}$.
Analogously, we replace each vertical edge $\{u,v\}\in E'$ by a path if there exists a node $s_i$ or $t_i$ in one of the corresponding rows $y(u)$ and $y(v)$.

For $i \in [k]$, we add to $V'$ an  $(2m+1) \times (2m+1)$ grid centered at $s_i$.  Let $B_{s_i}$ be the vertex set consisting of $s_i$ and the vertices of this grid. The edges on the grid $B_{s_i}$ are chosen such that the grid consists of
    $4$ disjoint $((m+1)\times m)$ complete grids $X^1_{s_i},X^2_{s_i},X^3_{s_i},X^4_{s_i}$ and there is no edge between two grids $X^a_{s_i},X^b_{s_i}$ with $a \neq b$ and each such grid is connected to $s_i$. Additionally if there exists an edge $(u,s_i)$ in $E$ then one subgrid is connected to $u$. We do the same construction for $t_i$ and let $B_i = B_{s_i}\cup B_{t_i}$.    

Then $V' = (\bigcup_{i \in [k]}B_i) \cup  (\bigcup_{(u,v)\in E}V'_{(u,v)})\cup V$.   
See Figure \ref{fig: construction} for an example of the construction. Note that by construction, if $G$ is a grid then $G_m$ is a grid with holes.  

We now define the distance function $d: V'\times V' \rightarrow \mathbb{R}_{\ge 0}$. Let $A=V'\setminus \bigcup_{i\in [k]} B_i $. We set $d(x,y) = 0$ with $x,y \in A$ and $d(x,y) = 0 $ with $x,y \in B_i$ for $i \in [k]$. Additionally $d(x,y) = 1$ with $x \in  A, y \in B_i$ and finally 
$d(x,y) = \lambda $ with $\lambda \in (1,2]$ and $x\in B_i,y \in B_j$ for $i\neq j$. The constructed metric $d$ can be viewed as a star metric, where all vertices in $ A$ are located in the center and all vertices in $B_i$ are located at a leaf with distance $1$. 

\begin{figure}[t]
    \centering
    \begin{minipage}[b]{0.47\linewidth}
        \centering
        \includegraphics[page=1,width=\linewidth]{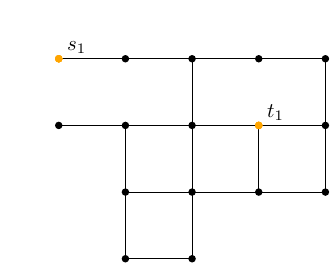}
        \par\smallskip (a)
    \end{minipage}\hfill
    \begin{minipage}[b]{0.47\linewidth}
        \centering
        \includegraphics[page=3,width=\linewidth]{fig/connected-k-median-construction.pdf}
        \par\smallskip (b)
    \end{minipage}
    \caption{Figure (a) shows a grid graph $G = (V,E)$ and vertex pair $(s_1,t_1)$.
    Figure (b) shows the graph $G_m$ with $m = 3$. The subgraphs added at $s_1$ and $t_1$  are indicated in color blue.
    Each edge $(u,v) \in E$ that crosses the row or column of $s_1$, $t_1$ is replaced by a path as indicated by the red dots.
    \label{fig: construction}}
\end{figure}

We first argue that if there exists a solution for the vertex-disjoint paths problem in $G$, then there exists a solution for the discrete connected $(k,p)$-clustering in $G_m$ with cost at most $n+n^2m$.

 \begin{lemma}\label{lem: ExistsVDPsol}
    If there exist vertex-disjoint $(s_i,t_i)$-paths $P_i$ for all $i \in [k]$ in $G$, then for all $p \in \mathbb{N}$ there exists a solution for  discrete connected $(k,p)$-clustering in $G_m$ with cost at most $|V|+|V|^2m$. 
 \end{lemma}

\begin{proof}
To show the lemma, we construct a feasible solution for discrete connected $(k,p)$-clustering on $G_m$ that has the claimed cost.
We choose the vertices $s_1,\dots,s_k$ as the $k$ centers and assign all vertices in $B_i$ to $s_i$.

Additionally, for all edges $(u,v)$ of path $P_i$ we assign the vertices $V'_{(u,v)}$ to $s_i$ to ensure that all points assigned to $s_i$ are connected.  
Next we deal with vertices in $V'$ that are not yet assigned. While there are unassigned vertices, we pick an arbitrary  vertex $v \in V'$ that is assigned to some cluster but has unassigned neighbors. We then assign each neighbor of $v$ to the same cluster as $v$. By assigning this way, we make sure that all clusters in $G_m$ are connected.

Now we bound the cost of this solution. Every vertex in $A$ has distance $1$ to its assigned center. For all $i\in[k]$ the vertices in $B_i$ have distance $0$ to their assigned center since we have a center located in each $B_i$. Therefore we have costs 
$|A| \leq |V| + \left |\bigcup_{(u,v)\in E} V_{(u,v)}'\right | \leq |V| + |V|^2m,$
where we used that $|E| \leq \binom{|V|}{2} \leq |V|^2/2$ and for all $(u,v) \in E$, $|V'_{(u,v)}| \leq m+2\le 2m$.
\end{proof}

Next we argue that if there does not exist a solution for the vertex-disjoint paths problem in $G$, then there does not exist a  solution for discrete connected $(k,p)$-clustering in $G_m$ with cost smaller than $m^2$.

\begin{lemma}\label{lem: ExistsNoVDPsol}
    If there does not exist a solution for the vertex-disjoint paths problem in $G$, then for all $p \in \mathbb{N}$ there does not exist a solution for discrete connected $(k,p)$-clustering in $G_m$ with cost smaller than $m^2$. 
\end{lemma}

\begin{proof}
    Assume there is a solution for  discrete connected $(k,p)$-clustering in $G_m$ with cost smaller than $m^2$. 
    Then there needs to be a center located in each set $B_i$, since otherwise each vertex in $B_i$ incurs distance at least $1$ to its assigned center and therefore the total cost would be at least $|B_i| = 2\cdot (2m+1)^2  > m^2$. Therefore, each set $B_i$ contains one center $c_i$ and the set $A$ contains none. 
    
    Additionally, at most $m^2$ vertices of the set $B_{s_i}$ can be assigned to a center not equal to $c_i$, otherwise they would incur cost at least $m^2$. This implies that at least $|B_{s_i}| - m^2 =3m^2+4m+1$   vertices in $B_{s_i}$ are assigned to $c_i$.  
   Note that $B_{s_i}$ consists of $4$ disjoint parts with size $m^2 + m$ each. Since at least $3m^2+4m +1$ vertices in $B_{s_i}$ are assigned to $c_i$ it holds by the pigeonhole principle that at least one vertex from each part $X^1_{s_i},\dots X^4_{s_i}$  is assigned to $c_i$. That implies that the cluster corresponding to $c_i$ contains connected subsets of all sets $X^1_{s_i},\dots, X^4_{s_i}$. Since
   $s_i$ is enclosed by $X^1_{s_i},\dots X^4_{s_i}$  and each $X^1_{s_i},\dots, X^4_{s_i}$ has at most one edge that leaves $B_{s_i}$,  it must be that $s_i$ is  assigned to $c_i$ as otherwise the cluster containing $s_i$ would not be connected.
   Analogously we argue that $t_i$ is assigned to $c_i$. The cluster that contains $s_i$ and $t_i$ also needs to contain an $(s_i,t_i)$-path in $G_m$ to ensure that the cluster is connected. Since the clusters are disjoint, this would imply that there are vertex-disjoint $(s_i,t_i)$-paths in $G$, which is a contradiction to the assumption.
\end{proof} 

We are now ready to prove our hardness result for  discrete $(k,p)$-connected clustering.

\begin{theorem}\label{thm: approx hardness discrete}
Given $k,p \in \mathbb{N}$.
Discrete connected $(k,p)$-clustering is NP-hard to approximate within factor  $\Omega(|V'|^{1/2 -\varepsilon})$, for any $\varepsilon >0$, on grid graphs $G'=(V',E')$ with holes.
\end{theorem}

\begin{proof}
Given a graph $G = (V,E)$ with $|V| = n$ and $k$ pairs $(s_i,t_i)$ as input instance for the vertex-disjoint paths problem. Let $m=n^a$ for some $a\ge 2$ to be chosen later and consider the graph $G_m$. 

By Lemma \ref{lem: ExistsVDPsol} and Lemma $\ref{lem: ExistsNoVDPsol}$ there cannot be an approximation algorithm for discrete connected $(k,p)$-clustering on grids with holes that has an approximation factor smaller than
$ \frac{m^2}{n+n^2m}  \ge \frac{m^2}{2n^2m}=
 \frac{m}{2n^2} = \frac{n^{a-2}}{2}.$
for all $k,p \in \mathbb{N}$.

Notice that the number $n'=|V'|$ of nodes in the instances of connected clustering satisfies $n' \in O(k \cdot m^2 + m \cdot n^2)$. Since $k\le n$ and $m\ge n$, this implies
$  n'\le cnm^2 = cn^{2a+1}$
for a sufficiently large constant $c$. Therefore, $n \ge (n'/c)^{1/(2a+1)}$ and the bound on the approximation factor becomes
  $\frac{n^{a-2}}{2} \ge \frac{(n'/c)^{(a-2)/(2a+1)}}{2} = \Omega((n')^{\frac{1}{2}-\frac{5}{4a+2}}).$
Now it suffices to choose $a$ large enough so that $5/(4a+2) \le \varepsilon$.
\end{proof}

Next we show an auxiliary lemma to relate the discrete connected $(k,p)$-clustering problem and the continuous variant. Effectively we get that an $\alpha$-approximation for the continuous variant implies a $(2^p\alpha)$-approximation for the discrete variant. 
\begin{lemma} \label{lem: discreteApproxCont}
    Given  metric space $(\mathbb{R}^m,d)$ and a connected graph $G=(V,E)$ with $V\subset \mathbb{R}^m$, $n = |V|$ and $k,p \in \mathbb{N}$. Let $C^*_1,\dots,C^*_k$ be an $\alpha$-approximation for the connected $(k,p)$-clustering problem on $G$. Then one can compute in $O(kn)$ time a solution $C_1,\dots,C_k$ with centers $c_1,\dots,c_k \in V$ that is an $(2^{p}\alpha)$-approximation for the discrete connected $(k,p)$-clustering problem on $G$.
    
\end{lemma}
\begin{proof}
Let $\Delta$ be the optimal cost for the  connected $(k,p)$-clustering problem on $G$ and $\Delta'$ be the optimal cost for the discrete connected $(k,p)$-clustering problem on $G$. Let $c^*_1,\dots,c^*_k$ be the corresponding centers of $C^*_1,\dots,C^*_k$. For $i \in [k]$, let $c'_i = \argmin_{x \in C^*_i} d(x,c^*_i)^p$.
By the relaxed triangle inequality we get
\begin{align*}
    \sum_{i = 1}^k \sum_{x \in C^*_i} d(x,c'_i)^p
    &\leq \sum_{i = 1}^k \sum_{x \in C^*_i} 2^{p-1}(d(x,c_i^*)^p + d(c_i^*,c_i')^p)\\
    &\leq \sum_{i = 1}^k \sum_{x \in C^*_i} 2^{p-1}(d(x,c_i^*)^p + d(c_i^*,x)^p)\\
    &\leq 2^p \alpha \Delta \\
    &\leq 2^p \alpha \Delta'.
\end{align*}
Note that computing $c'_1,\dots,c_k'$ can be done in $O(kn)$ time.
Since $C^*_1,\dots,C^*_k$ with centers $c'_1,\dots,c_k'$ is a feasible solution for discrete connected $(k,p)$-clustering, the proof is complete.
\end{proof}

We will again show the hardness of the connected $(k,p)$-clustering by reducing from the vertex-disjoint paths problem. 
Note that we cannot directly apply the reduction used for the discrete setting, since the continuous variant gets the metric space as input. To circumvent this issue, we focus on the Euclidean metric space and tailor an instance in  $k$-dimensions that incurs the same distances as utilized in the construction for the discrete variant, which then allows us to use the discrete variant to solve the vertex-disjoint paths problem.

\thmmain*

\begin{proof}
 Let $G = (V,E)$ with $(s_1,t_1),\dots,(s_k,t_k)$ be an instance for the vertex-disjoint paths problem. Let $n = |V|$ and consider the graph $G_m =(V',E')$ with $m \in \mathbb{N}$ . Let $d$ denote the Euclidean distance. Next consider a $k$-dimensional simplex with edge length $1$ and the origin as center. We place all the points in $A$ at the origin and all the points in $B_i$ at its own vertex of the simplex. Then  $d(x,y) = 0$ with $x,y \in A$ and $d(x,y) = 0 $ with $x,y \in B_i$ for $i \in [k]$. Therefore it holds $d(x,y) = 1$ with $x\in A, y \in B_i$ and $d(x,y) = \sqrt{2}$  for $x \in B_i,y \in B_j$ for $i \neq j.$

 Assume an $\alpha$-approximation for the connected $(k,p)$-clustering on $G_m$ and $d$ exists. By Lemma \ref{lem: discreteApproxCont}, we can construct in polynomial time a $2^p\alpha$-approximation for discrete connected $(k,p)$-clustering on $G_m$ and $d'$, where $d'$ is $d$ restricted to the vertices of $G_m$. Since $2^p$ is constant and
 through Theorem \ref{thm: approx hardness discrete} it must be that $\alpha \in \Omega(|V'|^{1/2 - \varepsilon})$.
 
\end{proof}

 \input{appendix/appendix_image_subspace_clustering}

%% file: appendix/appendix_image_subspace_clustering.tex
% #1 = x0  (Aufpunkt x)
% #2 = y0  (Aufpunkt y)
% #3 = dx  (Richtungsvektor x)
% #4 = dy  (Richtungsvektor y)
% #5 = t   (wie weit wir gehen: Skalar)
% #6 = Farbe
\newcommand{\pointOnLine}[6]{%
  \pgfmathsetmacro{\px}{#1 + #5*#3}%
  \pgfmathsetmacro{\py}{#2 + #5*#4}%
  \fill [#6] (\px,\py) circle[radius=2pt];
}

% #1 = x0  (Aufpunkt x)
% #2 = y0  (Aufpunkt y)
% #3 = dx  (Richtungsvektor x)
% #4 = dy  (Richtungsvektor y)
% #5 = t   (Parameter entlang der Geraden)
% #6 = s   (Abstand senkrecht zur Geraden)
% #7 = Farbe
\newcommand{\pointOrthOnLine}[7]{%
  % Punkt P auf der Geraden
  \pgfmathsetmacro{\Px}{#1 + #5*#3}%
  \pgfmathsetmacro{\Py}{#2 + #5*#4}%
  % Länge des Richtungsvektors
  \pgfmathsetmacro{\len}{sqrt(#3*#3 + #4*#4)}%
  % normierter Normalvektor (-dy, dx)/len
  \pgfmathsetmacro{\nx}{-#4/\len}%
  \pgfmathsetmacro{\ny}{ #3/\len}%
  % verschobener Punkt Q
  \pgfmathsetmacro{\Qx}{\Px + #6*\nx}%
  \pgfmathsetmacro{\Qy}{\Py + #6*\ny}%
  % Zeichnen (nur der Punkt):Flet
  \fill [#7] (\Qx,\Qy) circle[radius=2pt];%
  \draw[gray, thin] (\Px,\Py) -- (\Qx,\Qy);
}

% #1 = x0  (Aufpunkt x)
% #2 = y0  (Aufpunkt y)
% #3 = dx  (Richtungsvektor x)
% #4 = dy  (Richtungsvektor y)
% #5 = t_min (Startparameter)
% #6 = t_max (Endparameter)
% #7 = Farbe

\newcommand{\drawLine}[7]{%
  \pgfmathsetmacro{\Ax}{#1 + #5*#3}%
  \pgfmathsetmacro{\Ay}{#2 + #5*#4}%
  \pgfmathsetmacro{\Bx}{#1 + #6*#3}%
  \pgfmathsetmacro{\By}{#2 + #6*#4}%
  \draw [thick,#7] (\Ax,\Ay) -- (\Bx,\By);
}

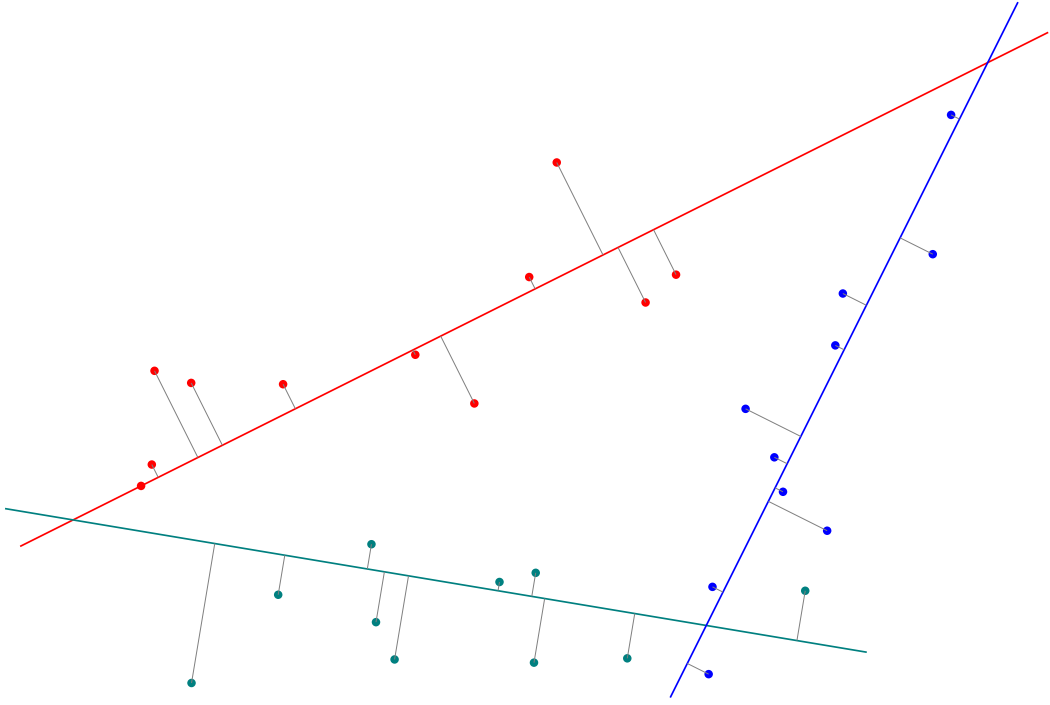
\begin{figure*}[t]

\scalebox{0.8}{
\begin{tikzpicture}
  % Gerade durch (0,0) in Richtung (1,0.5), gezeichnet für t in [-2, 12]
  \drawLine{0}{0}{1}{0.5}{-2}{15}{red}
  \drawLine{0}{-0.75}{1.5}{-0.25}{-1.5}{8}{teal}
  \drawLine{10}{-1}{0.5}{1}{-2.5}{9}{blue}

  % Aufpunkt (0,0), Richtungsvektor (1,0.5), t = 3
  \pointOnLine{0}{0}{1}{0.5}{0}{red}
  \foreach \t/\s in {
      1.344/1.153,
      8.474/-0.832,
      7.638/1.710,
      2.551/0.454,
      4.954/-1.245,
      4.495/-0.088,
      6.516/0.219,
      7.887/-1.018,
      0.939/1.603,
      0.283/0.237
  } {
    \pointOrthOnLine{0}{0}{1}{0.5}{\t}{\s}{red}
  }

  \foreach \t/\s in {
  0.812/-2.338,
  1.584/-0.663,
  4.307/0.395,
  3.935/0.147,
  7.230/0.835,
  2.948/-1.402,
  2.494/0.415,
  5.442/-0.751,
  4.449/-1.075,
  2.682/-0.844
} {
  \pointOrthOnLine{0}{-0.75}{1.5}{-0.25}{\t}{\s}{teal}
}

\foreach \t/\s in {
  3.251/0.163,
  1.367/0.237,
  5.102/-0.604,
  3.987/0.435,
  0.745/-1.087,
  1.818/1.021,
  -1.936/-0.395,
  0.968/-0.146,
  -0.756/0.194,
  7.066/0.155
} {
  \pointOrthOnLine{10}{-1}{0.5}{1}{\t}{\s}{blue}
}

\end{tikzpicture}
}
\caption{A point set and three one-dimensional subspaces in $\mathbb{R}^2$. For each point, the distance to its closest subspace (symbolized by gray lines) is computed and squared. The sum of these squared distances is the objective that shall be minimized by the choice of the subspaces.}\label{fig:subspace_distances}
\end{figure*}

%% file: appendix/appendix-data-set-choice.tex
\subsection{Dataset}\label{sec:data-set-choice}

There are two principal sources of satellite-derived SLA products: The  \cite{copernicus_climate_change_data}, which provides climate-oriented products emphasizing long-term homogeneity, and the Copernicus Marine Service, which prioritizes accuracy regarding the sea level estimation at each time step and enhanced spatial sampling. The CMEMS dataset \cite{altimetry_data} builds upon the steady two-satellite merged constellation formed by the reference missions (TOPEX-Poseidon, Jason-1/2/3, Sentinel-6A) (also used in C3S products) by homogenizing five additional missions with respect to this reference set. Although the constellation evolves over time, the integration of more satellites improves spatial coverage and temporal accuracy, potentially at the expense of long-term stability (see Quality Information Document \cite{altimetry_data}). Given our focus on the identification and analysis of regional spatiotemporal patterns, the higher-resolution, time-accurate CMEMS product is preferred. 
Spatial interpolation of the altimetry data onto the uniform grid is performed by the CMEMS using the Optimal Interpolation method described by  \cite{le1998improved}.

\input{appendix/appendix-tm}

%% file: appendix/appendix-tm.tex
\subsection{Reimplementation of Thompson and Merrifield}\label{appendix-tm}

\cite{TM14} develop a specific distance function for clusters that they use with an average linkage algorithm without explicit enforcement of connectivity. We implemented their method following \cite{EMastersthesis}, resulting in Algorithm~\ref{alg-al-tm}. 

\begin{algorithm}[H]
\caption{Hierarchical Agglomerative Clustering (Average Linkage), no explicit connectivity constraint\label{alg-al-tm}}
\begin{algorithmic}[1]
\STATE \textbf{Require:}\enspace $P = \{x_1, \dots, x_n\}$, number of clusters $k$
\STATE \textbf{Ensure:}\enspace A clustering of $P$ into $k$ clusters

\STATE Create $H = \{ C_i=\{x_i\} \mid x_i \in P\}$
\STATE Compute $D(C_i,C_j)=D(x_i, x_j)$ for all $x_i, x_j \in P$
\WHILE{$|H|> k$}
    \STATE Add $C^*=C_i \cup C_j$ to $H$, remove $C_i$ and $C_j$ 
    \FORALL{remaining clusters $C_l\in H$}
        \STATE Compute $D(C^*, C_l)$
    \ENDFOR
\ENDWHILE

\end{algorithmic}
\end{algorithm}

This algorithm does not enforce connectivity directly. Thompson and Merrifield use average linkage as dissimilarity measure and carefully craft a distance function that takes into account both the location of a point and the time series associated with it. For  $X_i=(x_i,t_i)$ and $X_j=(x_j,t_j)$ which contain both the spatial and temporal information, they use 
\[
    D(X_i, X_j) = 1 - \exp\left(-\frac{d(x_i, x_j)}{c_0}\right) r(t_i,t_k)
    \label{eq:spatiotemporal_distance_function}
\]
where $d(x_i, x_j)$ describes the spatial distance between two points $x_i$ and $x_j$,  $r(t_i, t_j)$ is the correlation coefficient between the time series  $t_i$ and $t_j$, and $c_0 \approx 4328$ is a normalization constant. We use the Haversine distance as the spatial distance which computes the distance between two points $x=(\lambda_1, \phi_1)$ and $y=(\lambda_2, \phi_2) $ in the simplified model that the earth is a perfect sphere with radius $R$ of roughly $6.370~\text{km}$. The Haversine distance is defined as
\[
R(x, y) =
R \cdot 2 \cdot \arctan2\left(
    \sqrt{a}, \sqrt{1 - a}
\right),
\]
where
\[
a =
\sin^2\left( \frac{\phi_2 - \phi_1}{2} \right)
+ \cos(\phi_1) \cos(\phi_2)
\sin^2\left( \frac{\lambda_2 - \lambda_1}{2} \right).
\]

For the correlation between two time series $t_1$ and $t_2$, we employ the Pearson correlation coefficient. This measure yields a value in the range $[-1,1]$ where 1 denotes a perfect positive linear correlation, 0 indicates no linear correlation, and –1 signifies a perfect negative linear correlation. The exponential function produces values in the interval $[0,1]$, assigning higher values to closer grid points and lower values to those that are more distant. 

Consequently, two time series that exhibit a strong negative correlation but are geographically close are considered more dissimilar than two equally negatively correlated time series that are spatially distant. In cases where two time series are uncorrelated, their resulting distance is 1, regardless of their spatial separation. Overall, the greater the similarity between two time series and the closer their associated grid points are, the smaller the resulting distance, approaching zero in the case of high correlation and close proximity.

Algorithm~\ref{alg-al-tm} does not enforce connectivity.  Connected clusters only emerge if the underlying data matches certain criteria, such as smoothness in time and space. 
If the input data is not appropriately filtered, the resulting segmentation lacks spatial coherence and fails to capture meaningful structures. In such cases, local extrema are typically isolated into their own clusters.  \Cref{fig:thompson_filtering} shows an example, and Table~\ref{tab:8_unfiltered} shows the decline in performance of the algorithm for unfiltered data.

%% file: appendix/appendix_eofs_andPCs.tex
\subsection{EOFs and PCs for global sea level data}
As can be seen in Figure~\ref{fig:eofs}, the first three EOFs and their corresponding PCs capture the long-term trend, the seasonal cycle, and the El Niño--Southern Oscillation (ENSO) signal, respectively. The ENSO signal is partially mixed with other modes of variability, which is a known issue in EOF analysis.
\begin{figure*}[h]
  \centering
  \begin{tikzpicture}
     \node [label=below:{(a) EOF 0}] at (0,0) {\includegraphics[width=0.4\textwidth]{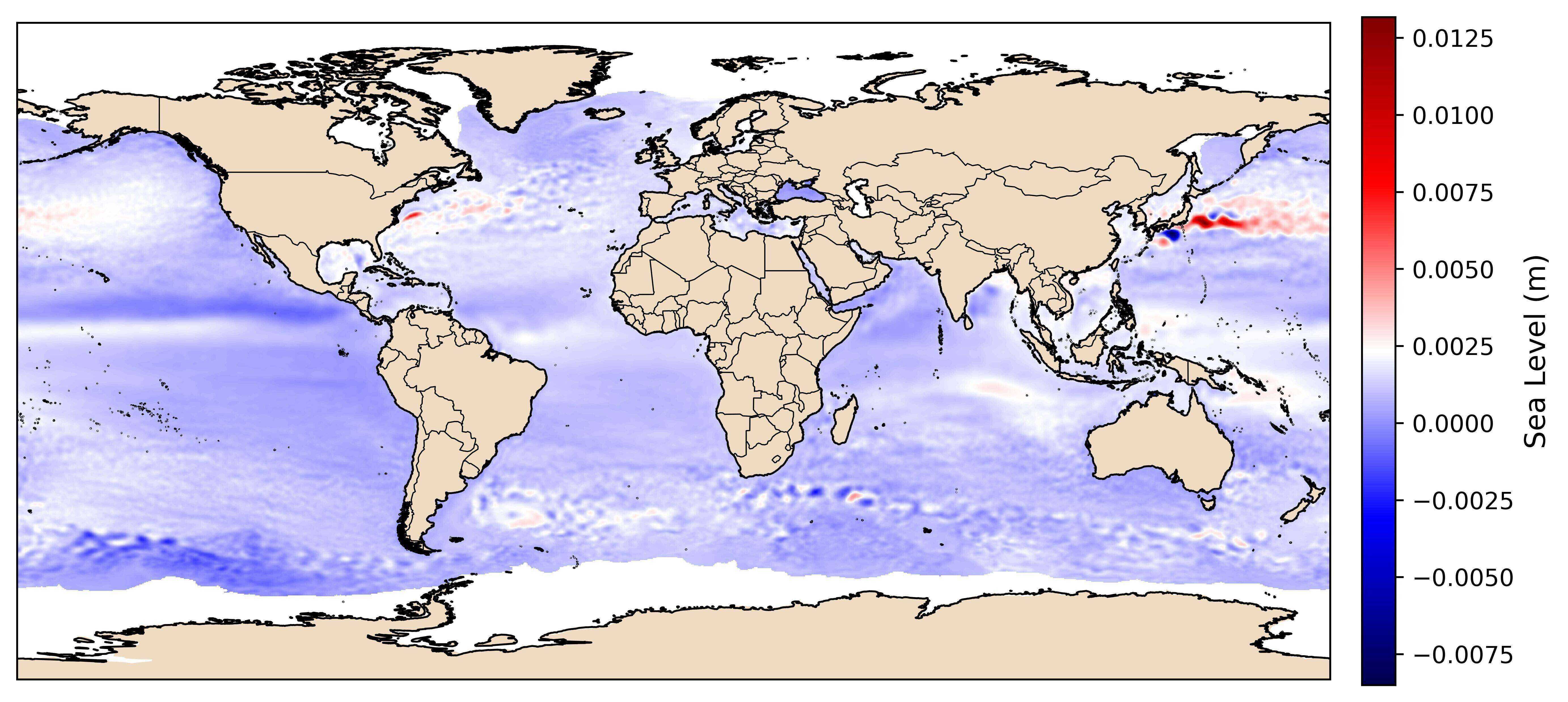}};
     \node [label=below:{(b) PC 0}]at (7.1,0) {\includegraphics[width=0.4\textwidth]{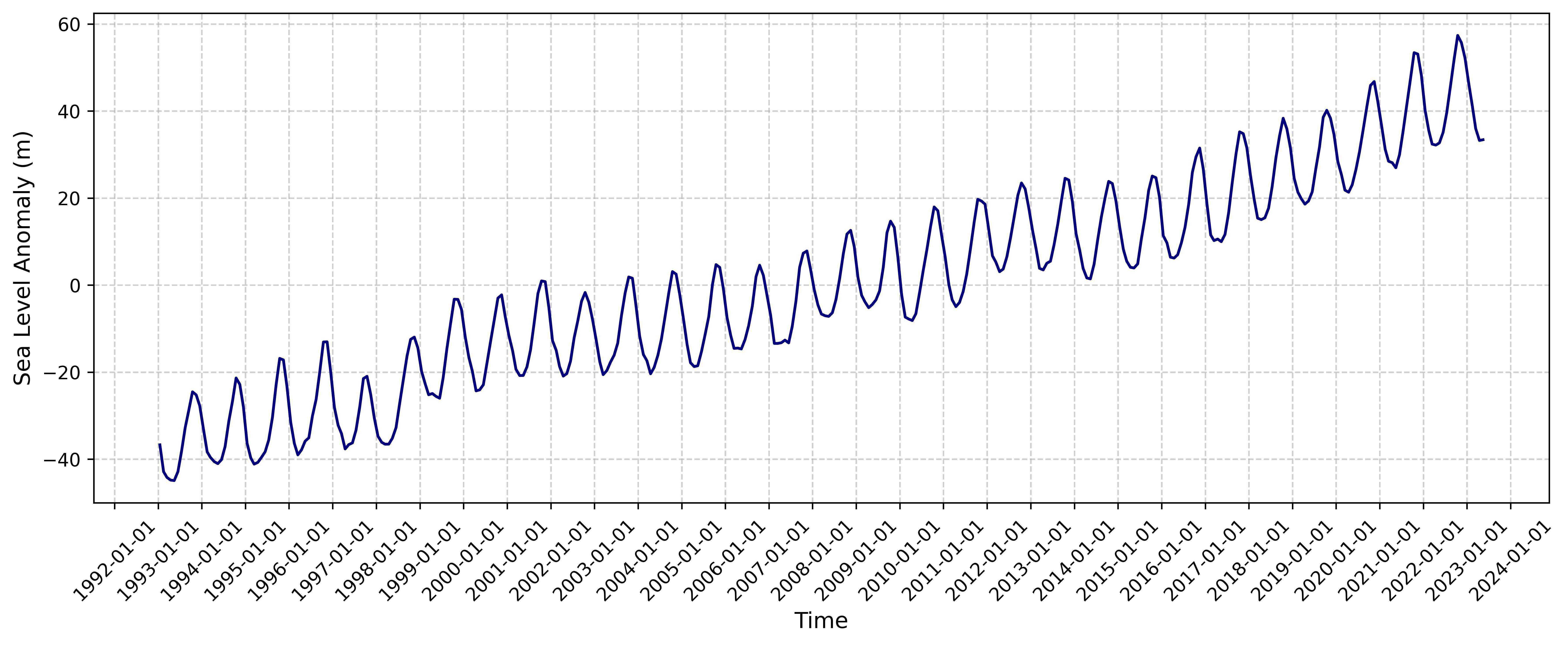}};
     \node [label=below:{(c) EOF 1}]at (0,-4) {\includegraphics[width=0.4\textwidth]{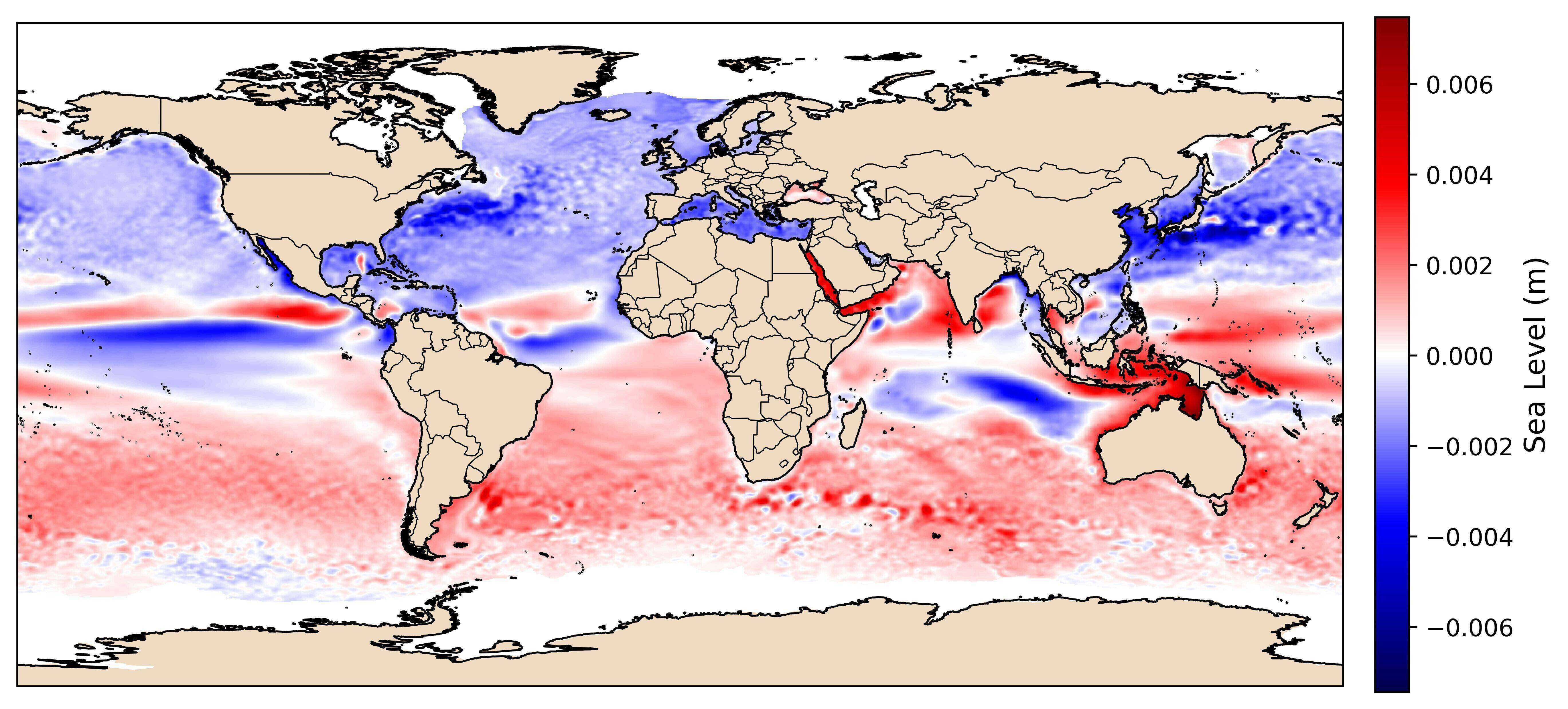}};     
     \node [label=below:{(d) PC 1}]at (7.1,-4) {\includegraphics[width=0.4\textwidth]{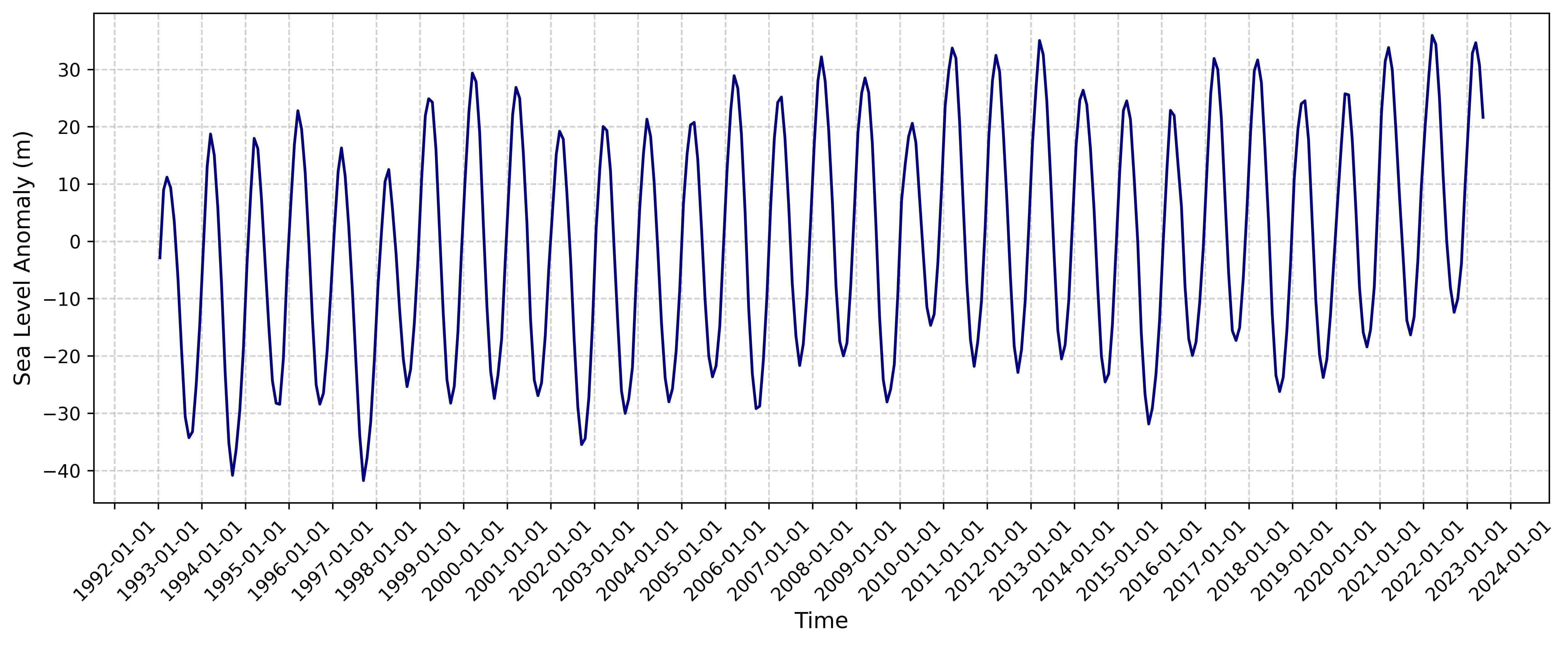}} ;
     \node [label=below:{(e) EOF 2}]at (0,-8) {\includegraphics[width=0.4\textwidth]{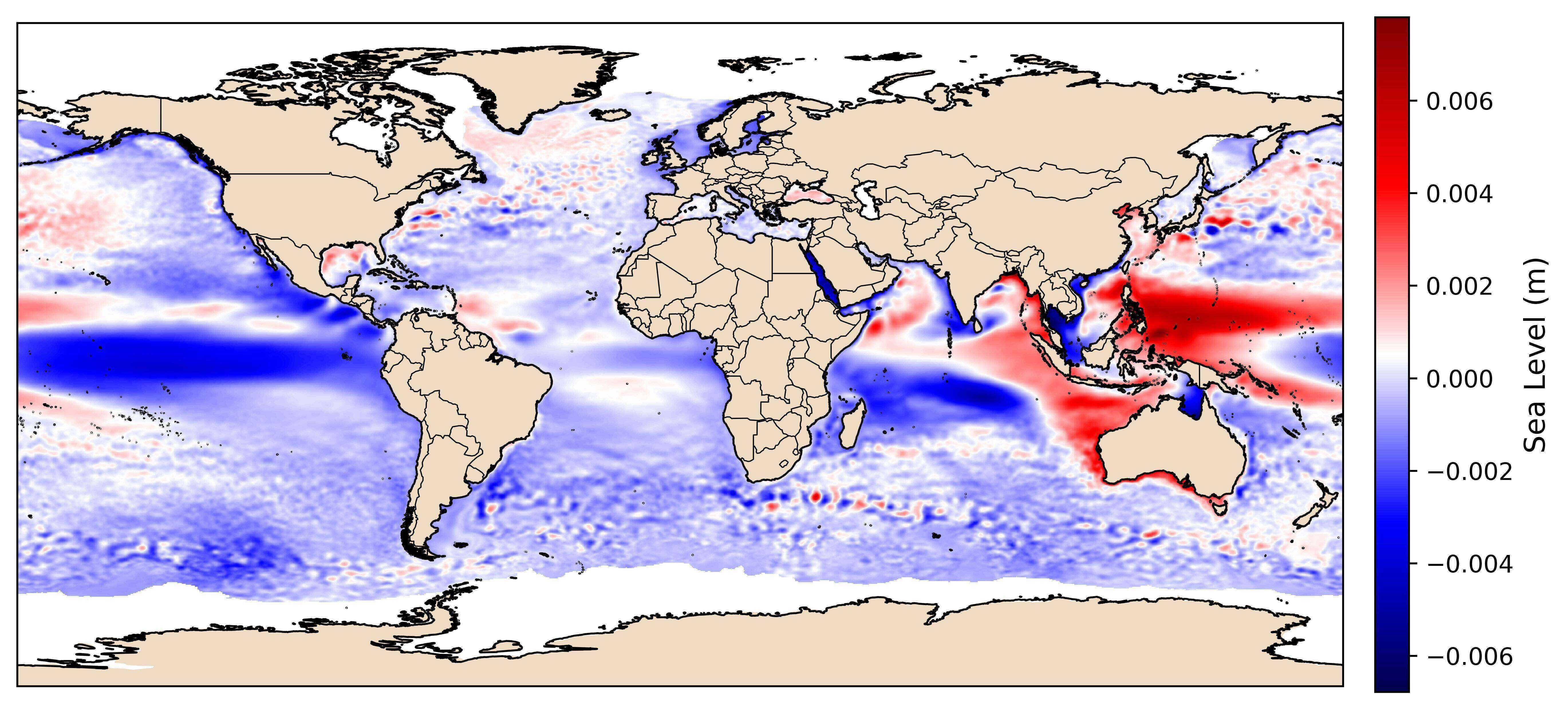}} ;
     \node [label=below:{(f) PC 2}]at (7.1,-8) {\includegraphics[width=0.4\textwidth]{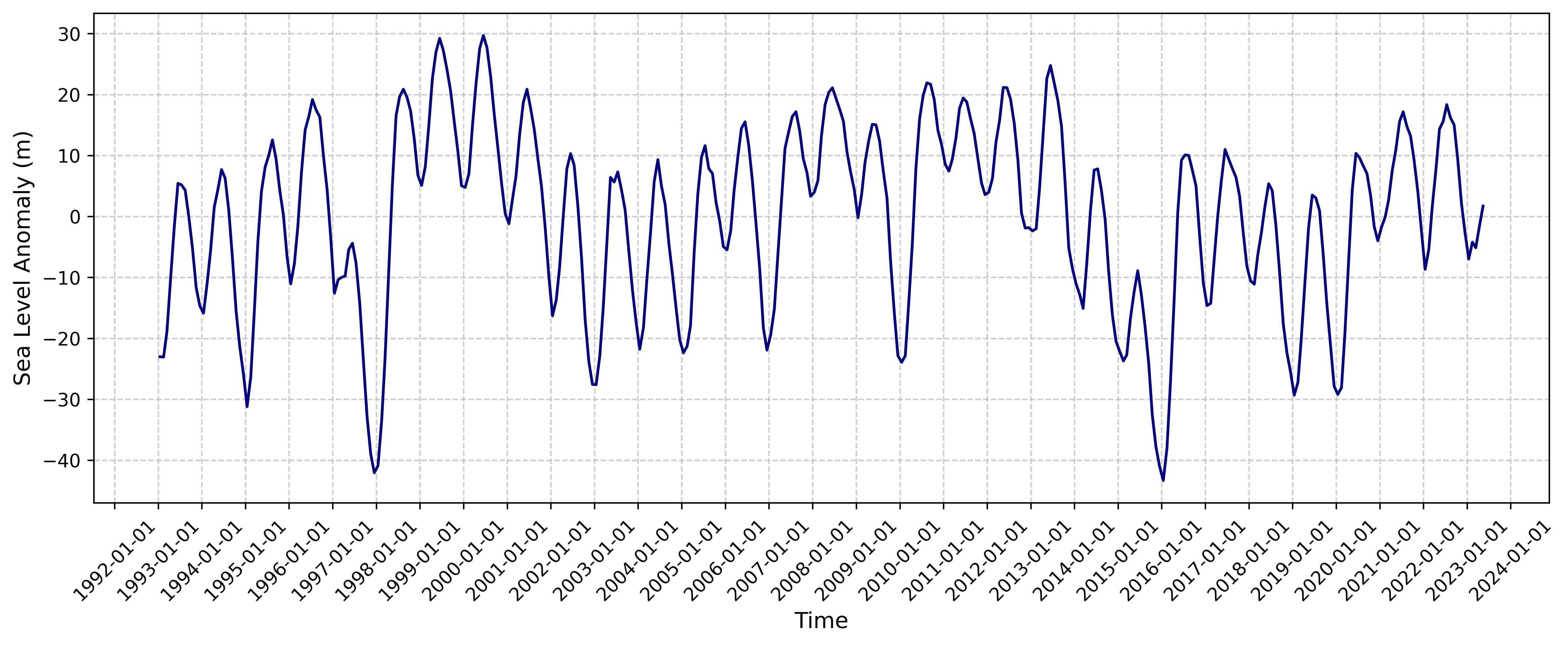}}; \end{tikzpicture}
  \caption{Empirical Orthogonal Functions (EOFs) and their corresponding Principal Components (PCs) derived via Singular Value Decomposition (SVD) of global sea level data. Panels (a), (c), and (e) show the spatial patterns (EOFs 0–2), while panels (b), (d), and (f) present their respective temporal variations (PCs 0–2). EOFs 0 and 1, along with their PCs, primarily capture the long-term trend and the seasonal cycle, respectively. The El Niño–Southern Oscillation (ENSO) signal is evident in EOF 2, although it is partially mixed with other modes of variability. } 
  \label{fig:eofs}
\end{figure*}

%% file: appendix/appendix_geodetic-analysis.tex
\subsection{Implications for sea level research}\label{appendix:sea_level_research}
\begin{figure}[tb]
    \centering
  \begin{subfigure}{0.4\textwidth}
    \includegraphics[width=\textwidth]{fig/cluster_8.0_first_component.jpg}
    \caption{Signal from cluster 8 (red) together with the ENSO index (blue).}
    \label{fig:cluster8}
  \end{subfigure}
  \begin{subfigure}{0.45\textwidth}
    \includegraphics[width=\textwidth]{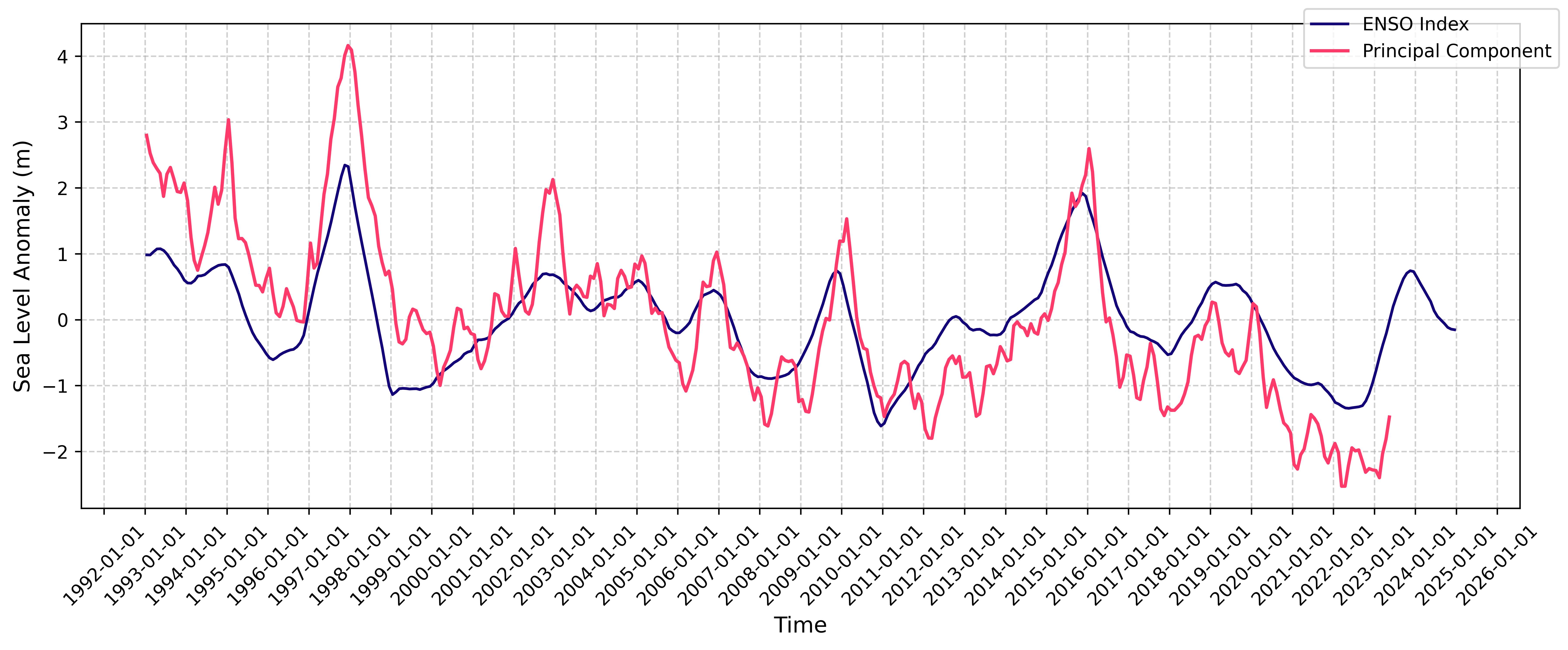}
    \caption{Signal from cluster 6 (red line) together with the ENSO index (blue line).}
    \label{fig:cluster6}
  \end{subfigure}
  \begin{subfigure}{0.4\textwidth}
    \includegraphics[width=\textwidth]{fig/cluster_9.0_first_component.jpg}
    \caption{Signal from cluster 9 (red) together with the ENSO index (blue).}
    \label{fig:cluster9}
  \end{subfigure}
  \caption{First EOFs for clusters 6, 8 and 9 together with the ENSO index. The first EOFs for clusters 6 and 9 show a strong correlation with the ENSO index, while the first EOF for cluster 8 shows a weak correlation.} 
  \label{fig:cluster_compared_with_enso_appendix}
\end{figure}

\begin{figure}[tb]
    \centering
  \begin{subfigure}{0.4\textwidth}
    \includegraphics[width=\textwidth]{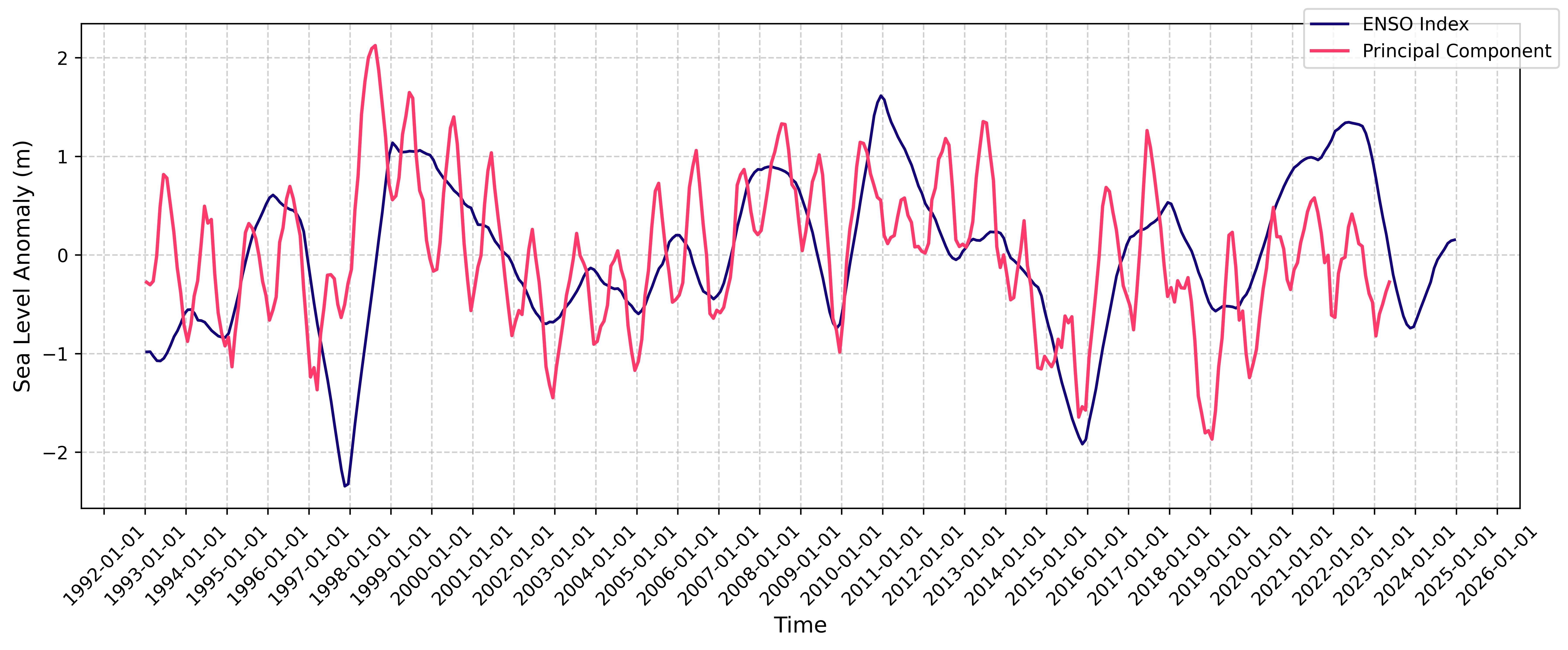}
    \caption{Second EOF for cluster 6.}
    \label{fig:cluster6_second_eof}
  \end{subfigure}
  \begin{subfigure}{0.45\textwidth}
    \includegraphics[width=\textwidth]{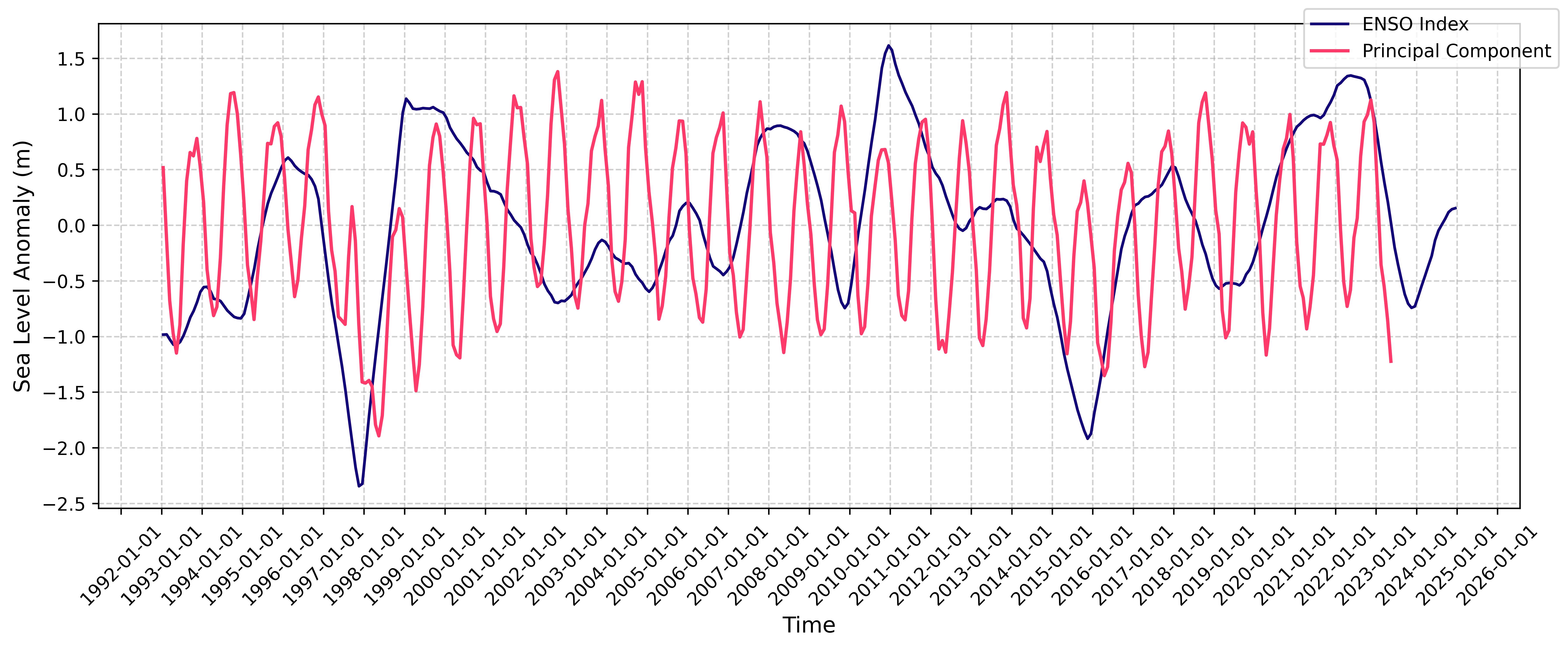}
    \caption{Second EOF for cluster 9.}
    \label{fig:cluster9_second_eof}
  \end{subfigure}
  \caption{Second EOFs for clusters 6 and 9. The second EOFs for both clusters show a strong annual cycle, which is the second most dominant signal in these regions.} 
  \label{fig:second_principal_components}
\end{figure}

Combining subspace clustering with \emph{Empirical Orthogonal Function} (EOF) decomposition offers a practical way to extract regionally coherent patterns in sea level variability that may not be clearly represented in a standard global EOF analysis. 

Conn-Subspace can identify interesting regions with a specific sea level rise behavior. As an example for this, we discuss the dominant EOFs for the global and for two specific interesting regions, to showcase how two clusters reflect a known pattern of sea level behavior. 

As expected, globally, the most dominant sea level signals are the trend (see Figure~\ref{fig:eofs}(a),(b)) and the seasonal cycle which is driven by solar forcing and ocean heat uptake, and land-ocean fluxes such as seasonally varying freshwater flow from ice sheets and glaciers, the world's rivers and ocean precipitation and evaporation (see Figure ~\ref{fig:eofs}(c),(d) in the appendix). The associated EOFs 0 and 1 shows a trend pattern and clear positive/negative separation between the northern and southern hemispheres, with the corresponding PCs showing a clear trend and annual cycle.
By applying subspace clustering, the analysis can operate on more localized feature spaces, potentially revealing regional modes that are less prominent in a global context.

For the subspace clustering approach, in most subregions, such as the Atlantic Ocean \Cref{fig:cluster_compared_with_enso_appendix}(a), the trend and annual cycle are also the dominating signals. However, in the tropical Pacific Ocean the dominant signal is driven by the so called El Niño–Southern Oscillation effect (ENSO), a climate phenomenon driven by varying winds and sea surface temperatures in the tropical Pacific that impacts climate and weather globally, e.g. leading to extensive rain in South America while causing droughts in Australia. In the global EOF decomposition, this effect is still one of the most dominant signals (see \Cref{fig:eofs}(e),(f)) indicating its impacts all over the globe.

The subspace clustering approach is able to derive clusters, which fit well to the eastern and western Pacific regions responsible for driving the ENSO phenomenon (see \Cref{fig:cluster_compared_with_enso_appendix}(b), (c), see Cluster 6 and 9 in \Cref{fig:clustering_geodesy}). In these clusters, the dominant signal aligns closely with the ENSO signal, while the annual signal is found in the second EOF associated with these regions (see \Cref{fig:second_principal_components}). Finding clear regional separations based on signal content using the subspace clustering approach enables a much more in-depth data analysis on regional scales, focusing on the dominant signals and drivers of sea level change relevant for these regions. 

In geodesy, EOFs are also used to reconstruct incomplete tide gauge data. To achieve this, satellite altimetry data are decomposed into EOFs and corresponding temporal components. The longer signals of tide gauge records are then fitted to the data using a least squares approach, see for example \cite{church2004}. This is usually done for the entire globe. A clustering of the sea enables employing this technique regionally, e.g. for regions of good tide gauge record density. To the best of our knowledge this has seldom been applied (e.g. by \cite{klos2019introducing}) for the eastern Pacific cluster and we plan to explore this in future work. However, the regional separation may lead to inconsistencies regarding signal content at the borders to neighboring regions, which warrants further investigations.

The separation of the ocean into contiguous regions additionally enables a regional reconstruction of historic sea levels prior to the altimetry era.

%% file: appendix/appendix_accompanying_pseudo_code.tex
\subsection{Accompanying pseudocode}\label{appendix:pseudocode}

The following code was skipped in the main part of the paper.

\begin{algorithm}
\caption{Agglomerative Connected Clustering\label{alg-ac-connected}}
\begin{algorithmic}[1]
\REQUIRE $P = \{x_1, \dots, x_n\}$, graph $N$, $k \in \mathbb{N}$
\ENSURE A clustering of $P$ into $k$ connected clusters

\STATE Create $H = \{ C_i=\{x_i\} \mid x_i \in P\}$
\STATE Compute $D(C_i,C_j)=D(x_i, x_j)$ for all $x_i, x_j \in P$

\FOR{each cluster $C_i \in H$}
    \STATE Set $\mathcal{N}_i \gets \{C_j \mid x_i \text{ and } x_j \text{ are adjacent in } N\}$
\ENDFOR

\WHILE{$|H| > k$}
    \STATE Identify pair $(C_i, C_j)$ with minimal distance
    \STATE Add $C^*=C_i \cup C_j$ to $H$, remove $C_i$ and $C_j$ 
    \STATE $\mathcal{N}^* \gets \mathcal{N}_i \cup \mathcal{N}_j$
    \FOR{each $C_l \in \mathcal{N}^*$}
        \STATE Remove $C_i$, $C_j$ from $\mathcal{N}_l$ and add $C^*$ to $\mathcal{N}_l$
    \ENDFOR

    \FOR{each $C_j \in \mathcal{N}^*$}
        \STATE  Compute distance $D(C_i, C_j)$ \label{distanzneuberechnung}
    \ENDFOR
\ENDWHILE
\end{algorithmic}
\end{algorithm}
Notice that in Step~\ref{distanzneuberechnung} in the average linkage version of Algorithm~\ref{alg-ac-connected}, a new computation of distances is only necessary for the clusters in the new neighborhood. For those, it has to be recomputed (only) if a cluster was neighbor to only one of $C_i, C_j$ since in this case the distances to the other cluster have not been computed yet. If a cluster $C'$ was neighbor to both, then we  use a weighted sum of $d(C',C_i)$ and $d(C',C_j)$ instead.

\begin{algorithm}[t]
\caption{Greedy $k$-means++ Clustering\label{alg:greedykmeans}}
\begin{algorithmic}[1]
\REQUIRE Point set $P$, $k\in \mathbb{N}$,candidate parameter$\ell\in \mathbb{N}$
\ENSURE centers $c_1, \ldots, c_k$ and clusters $C_1, \ldots, C_k$

\STATE Pick initial center $c_1$ uniformly at random from $P$

\FOR{$i = 2$ to $k$}
    \STATE Sample $\ell$ candidate points $x_1, \ldots, x_\ell \in P$, where each $x \in P$ is chosen with probability
    \[
    p(x) = \frac{\min_{1 \leq j < i} \|x - c_j\|^2}{\sum_{y \in P} \min_{1 \leq j < i} \|y - c_j\|^2}
    \]
    \STATE For each candidate $x_m$, compute
    \[
    \phi(x_m) = \sum_{y \in P} \min\left(\min_{1 \leq j < i} \|y - c_j\|^2, \|y - x_m\|^2\right)
    \]
    \STATE Set $c_i$ to be  $x_m$ that minimizes $\phi(x_m)$
\ENDFOR

\REPEAT
    \STATE Assign each $x \in P$ to $\arg\min_{1 \leq i \leq k} \|x - c_i\|^2$ and let $C_1, \ldots, C_k$ be the resulting clusters
    \FOR{$i = 1$ to $k$}
        \STATE Compute $\mu(C_i) = \frac{1}{|C_i|} \sum_{x \in C_i} x$
        \STATE Update center: $c_i \gets \mu(C_i)$
    \ENDFOR
\UNTIL{convergence criterion is met}

\end{algorithmic}
\end{algorithm}

%% file: appendix/appendix-merging.tex
\subsection{Details on connectivity methods}
In this section, we expand on the algorithms discussed in \Cref{sec:connectivity}.
\subsubsection{Establishing connectivity via merging}\label{appendix-merging}
The first approach converts a clustering with possibly disconnected clusters into a connected clustering. The idea of this algorithm is that we typically expect that, while the clusters may be disconnected, we still expect medium size regions that belong to the same cluster. For an example of such a situation, see Figure~\ref{fig:merging}(a). Our aim is to iteratively remove small regions and merge them into neighboring regions of other clusters. The procedure is given in Algorithm~\ref{algo:merging_for_connectivity}. Notice that this method inherently uses that the time series and the locations \emph{are} somewhat correlated because otherwise we would not start with small contiguous regions.

Step~\ref{step-merging-graph} constructs a graph where two points are connected if they are neighbors in $N$ and are assigned to the same cluster in the input clustering. 
Visually, when looking at a clustering like in Figure~\ref{fig:merging}(a), points are connected in $G_C$ if they have the same color and are adjacent. (Notice that the set of edges $E'$ is a subset of the edges $E$ in $N$). The connected components of $G_C$ are the contiguous single-color regions visible in Figure~\ref{fig:merging}(a).
The connected components are computed in Step~\ref{step-merging-cc} and stored in $Z$. Notice that every $A \in Z$ satisfies $A \subset C_i$ for some $i\in[k]$, and that $A$ would be a feasible connected cluster in $N$. However, there are more than $k$ components in $Z$ because the input clusters were not connected.

Step~\ref{step-merging-gcc} computes $G_{CC}$ which uses the connected components as nodes and connects two regions $A$ and $B$ if they contain points $a\in A$ and $b\in B$ that are adjacent in $N$, i.\,e., $\{a,b\}\in E$. Note that regions that belong to the same cluster are never connected in $G_{CC}$ since they would be one component then. 
We use $G_{CC}$ to reduce the number of disconnected regions by merging regions into larger ones, beginning with the one of smallest cardinality. For each iteration of the loop in Step~\ref{step-merge-loop-start} to Step~\ref{step-merge-loop-end} we compute the region with the least number of points and then merge it into a neighboring cluster. In order to choose the cluster to merge into, we compute the preference of each point $x \in A_{\min}$ in Step~\ref{step-merge-winner}: We compute the cluster $C_i$ where the associated subspace $S_i$ is closest to $x$, but only consider those $C_i$ which have a subregion $B$ neighboring $A_{\min}$. So from the clusters surrounding $A_{\min}$, we compute which cluster is preferred by $x$. Then we do a majority vote in Step~\ref{step-merge-clusterwinner} and choose the neighboring cluster, which is preferred by the largest number of points in $A_{\min}$. Then $A_{\min}$ is merged into that cluster. This decreases the number of regions by one.

The process stops when the number of regions reaches $k$. Each region now becomes a cluster. During Step~\ref{step-merge-output} we make sure that if an input cluster is split by this process, one of the clusters retains the original cluster number, and if two connected components belonging to different clusters are merged, the new cluster assumes the place of one of those it was created from (mostly for picture consistency and evaluation).

\subsubsection{Proof for \Cref{thm:merging_runtime}}\label{appendix:runtime-proof}
\Cref{thm:merging_runtime} states that the iterative merging algorithm presented in \Cref{algo:merging_for_connectivity} can be implemented in $O(|E|+nk+n\log n+n\alpha(n))$ time. Here we provide the accompanying proof.

\thmruntime*
\begin{proof}\label{proof-runtime}
Given the neighborhood graph $N = (V,E)$ with $|V|=n$ and $|E|=m$, a clustering $C_1,\dots, C_k$ and the subspaces associated with each cluster $S_1,\dots,S_k$. Let $z_0$ be the initial number of connected components and $z_i= |Z|$ in iteration $i$ of the while-loop. For a component $A\in Z$, let $n(A)$ be the clusters associated with the connected components adjacent to $A$ in the component graph $G_{CC}=(Z, E'')$. So, this describes clusters a connected component can possibly be merged into. For easy lookup these neighboring clusters are stored in a bitset. We assume that for each $A \in Z$, $n(A) > 0$, so each component has at least one neighbor it can be merged into. 

\paragraph{Preprocessing}
First, we need to build the graph $G_C$ by scanning $E$ and only keeping edges $u,v$, if $u\in C_i$ \emph{and} $v \in C_i$, for any $C_i$. Next, computing the connected components of $G_C$ is $O(n+m)$ by BFS/DFS. Constructing $G_{CC}$ by scanning $E$ and adding an edge between components that have at least one adjacent pair of nodes in $N$ takes $O(m)$. 

\paragraph{Winner computation and merging}
In order to increase efficiency, we can precompute a numerical scoring vector for each connected component $A \in Z$ that contains the aggregated preferences of all $x \in A$ over the subspaces $S_i$ for $i\in [k]$.
This is done as follows; we precompute for each point $x\in V$ a ranking of the clusters $\pi_x(1),\dots, \pi_x(k)$ by increasing distance to the subspaces $S_i$. Let $r_x[i]$ denote the position of cluster $i$ for point $x$. Fix any non-decreasing positional scoring function $s$ on ranks. The contribution of $x$ to its components valuation of cluster $[i]$ is $s(r_x[i])$, which is fixed throughout the algorithm. 

For each connected component $A \in Z$ aggregate these contributions in a score vector: $\textit{score}_A[i]= \sum_{x \in A} s(r_x[i])$ for all $i \in [k]$. 
Given the adjacent-cluster set $n(A) \subseteq [k]$, the chosen cluster for merging is then $i^w(A) = \arg\max_{i \in n(A)} \textit{score}_A[i]$. 
This scoring rule maps the majority decision well enough for this algorithm, as updating ordinal preferences would require more time and space with limited benefits. 
When two components $A$ and $B$ merge, the new component's score vector is updated by component-wise addition $\textit{score}_{A \cup B} = \textit{score}_A + \textit{score}_B$. This leads to winner lookup being in $O(|n(A)|) \subseteq O(k)$ and score updates are $O(k)$.

Computing all distances takes time and space $O(n\cdot k)$.

Additional to these metadata updates, in each merging step we need to join all connected components $B^{i^w}$ that belong to the winning cluster $i^w$ and are adjacent to the cluster that is being merged $A$.
This information is stored in each connected component. All unions of components over the complete execution cost $O((z_0-k)\alpha(z_0)) \subseteq O(n\alpha(n))$ using union-by-rank and path compression \cite{tarjan1984worst}.
Furthermore, we need to merge the set of neighbors $n(A\cup B) = n(A) \cup n(B) \setminus \{c_A, c_B\}$, where $c_A$ is the cluster that component $A$ belongs to and $c_B$ is the one component $B$ belongs to and adjacent components belonging to each cluster. 

\paragraph{Maintaining \boldmath$A_{min}$} 
In step~\ref{step-merge-min} $A_{min}$, the component of smallest cardinality, needs to be identified. This can either be done via naive sorting in $O(z_i \log(z_i))$ or improved to $O(\log(z_i))$ by using a min-heap. Note that $z_i$ decreases by at least one in each iteration, yielding $O((z_0 - k)\log(z_0))$ across the $O(z_0 - k)$ iterations, which is in $O(n \log(n))$, since $z_0 \leq n$. 
\paragraph{Iterations}
The number of times this loop is executed depends on $z_0$, the initial size of $Z$, since each iteration decreases the number of connected components by at least one.
Thus, the loop runs at most $z_0 - k$ times.

\paragraph{Conclusion}
The total time is $O(nk)$ for distance and score precomputation, $O(n\log n)$ for maintaining the heap, $O(nk)$ for winner scans and score additions, $O(n\alpha(n))$ for union--find, and $O(m)$ for scanning the edges of the neighborhood graph.

Hence the claimed bound $O(m+nk+n\log n+n\alpha(n))$ holds. For bounded-degree neighborhood graphs, including the grid graphs used in our application, $m\in O(n)$ and this becomes $O(n(k+\log n))$; whenever $k=\Omega(\log n)$, it simplifies further to $O(nk)$.
\end{proof}

\paragraph{Remark}
We implemented a simpler version of this algorithm without the use of the sophisticated data structures described above.
Our experiments show that this variant is empirically fast enough, however, if one needs a very fast version of the merging routine, the modifications described in \Cref{proof-runtime} can be used.

%% belongs to next section, just here for the looks

\begin{figure*}[t]
  \centering
  \begin{tikzpicture}
     \node [label=below:{(a) Clustering using EnSC.}] at (0,0) {\includegraphics[width=0.42\textwidth]{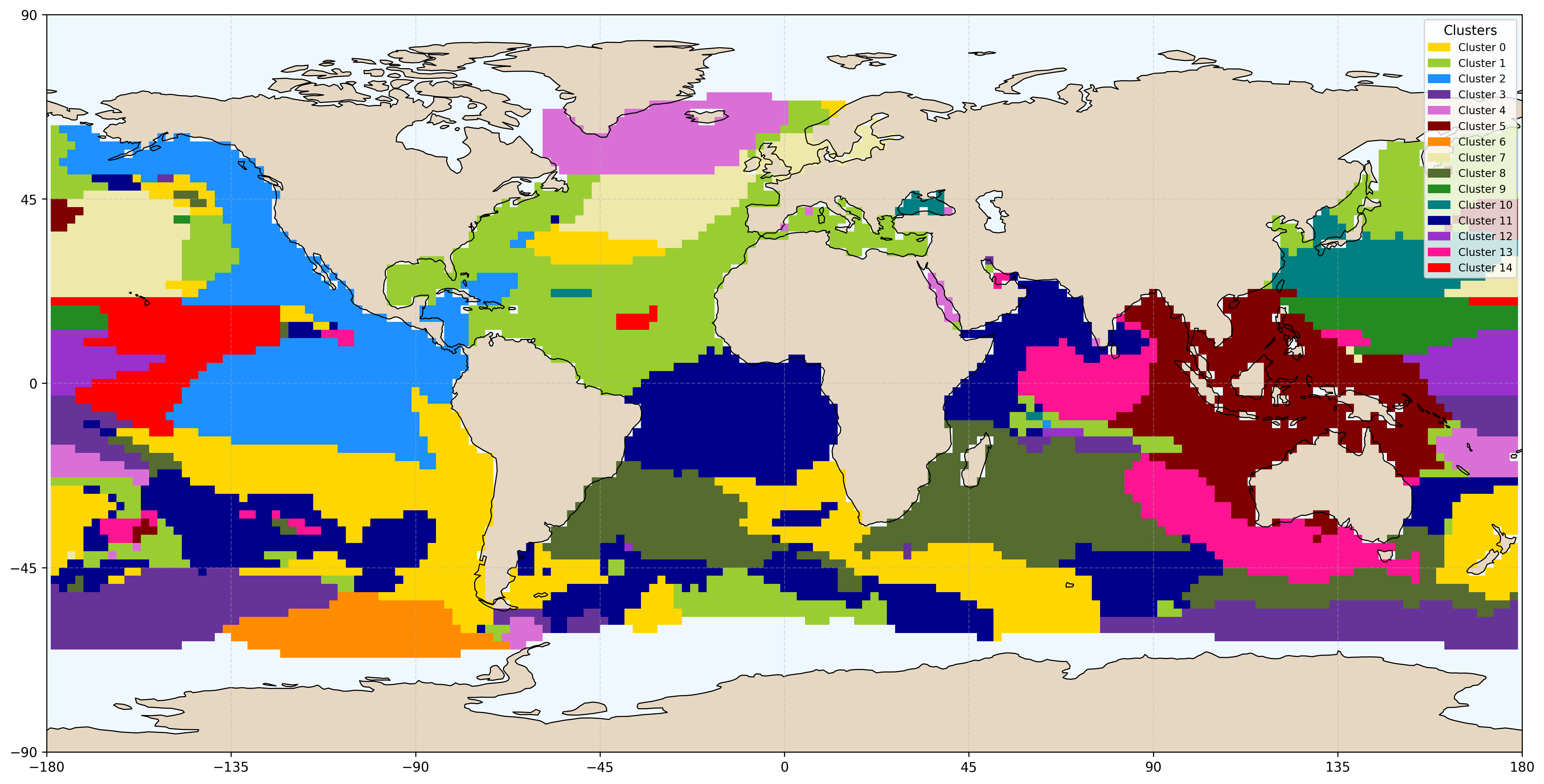}};
     \node [label=below:{(b) Clustering using SSC-OMP.}]at (7.8,0) {\includegraphics[width=0.42\textwidth]{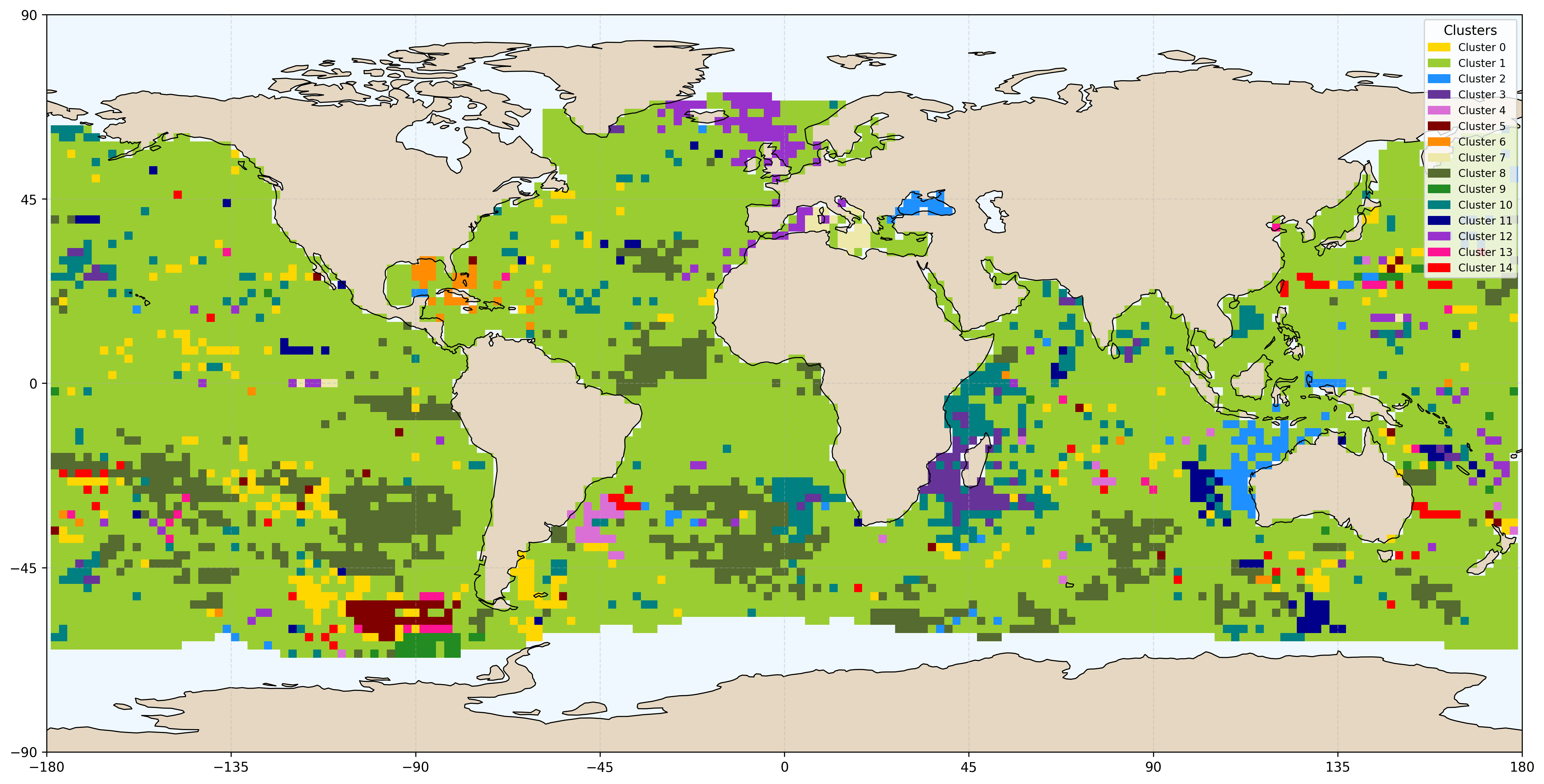}};\end{tikzpicture}
  \caption{Examples for the clusterings produced by SSC-OMP \cite{you2016oracle} and EnSC \cite{YouRV16}.}
  \label{fig:sc-comp}
\end{figure*}

\begin{table*}[b]
\centering
\caption{Number of connected components for SSC-OMP and EnSC on filtered and unfiltered input data, grouped by number of clusters. The second column indicates how many clusters should be found by the algorithm and columns 3 and 4 show how many connected components the resulting clustering has.}\label{tab:number_of_ssc-omp-ensc}
\begin{tabular}{|c|c|c|c|}
\hline
\makecell{method} & $\#$ clusters wanted & filtered & unfiltered \\ \hline
SSC-OMP & 8  & 225 & 52  \\ \hline
EnSC  & 8  & 128 & 31 \\ \hline
SSC-OMP & 15 & 796 & 90 \\ \hline
EnSC  & 15 & 169 & 54  \\ \hline
SSC-OMP & 20 & 1374 & 114 \\ \hline
EnSC  & 20 & 222 & 102  \\ \hline
SSC-OMP & 25 & 1966 & 137  \\ \hline
EnSC  & 25 & 219 & 132  \\ \hline
\end{tabular}
\end{table*}

\subsubsection{Filtering}\label{appendix-filtering-connectivity}

For this approach, we implement Step~\ref{step:establish_conn} of Algorithm~\ref{alg:sc} as follows: 
We treat the grid-like graph $N$ as an image, where each vertex corresponds to a pixel.
Pixels that represent land are masked using a binary mask and ignored during the filtering process. 

More specifically, set \( \sigma = \frac{h}{1.178} \) and for any point $x$, let $\mathcal{B}(x)$ denote the ball around $x$ containing all grid points that satisfy $d_H(x,y) \le 3 \sigma$ where $d_H(x,y)$ is the Haversine distance between $x$ and $y$. Then assign a weight to every point $y \in \mathcal{B}(x)$  using a Gaussian kernel:
\[
w(x, y) = \exp\left(-\frac{d_H(x, y)^2}{2\sigma^2}\right),
\]
Notice that the definition of $\sigma$ ensures that  the weight drops to half its maximum at the given half-width~$h$. These weights are then normalized: \mbox{$w_{x,y}^{norm} = \frac{w_{x,y}}{\sum_{y' \in \mathcal{B}_x \cup \{x\}} w_{x,y'}}.$}
The cluster that $x$ shall belong to is determined by weighting the clusters that the points in $\mathcal{B}(x)$ belong to. This is done by defining the following weighted support for a cluster \( C_i \in \{C_1, \dots, C_k\} \) at point \( x \):
\[
S(x,C_i) = \sum_{\substack{y \in \mathcal{B}(x) \\ y \in C_i}} w^{norm}(x, y).
\]
The new cluster label assigned to \( x \) is then the one that has the highest weighted support:
\[
C^*(x) = \arg\max_{C_i \in \mathcal{C}} S(x,C_i).
\]
As is typical for Gaussian filtering this reduces noise and creates more cohesive regions. However, this does not necessary result in exactly $k$ connected clusters, and thus Step~\ref{step:final_conn} is also executed.

%% file: appendix/appendix_experiments.tex
\begin{figure*}[t]
  \centering
  \begin{tikzpicture}
     \node [label=below:{(a) Clustering using EKGCSC.}] at (0,0) {\includegraphics[width=0.42\textwidth]{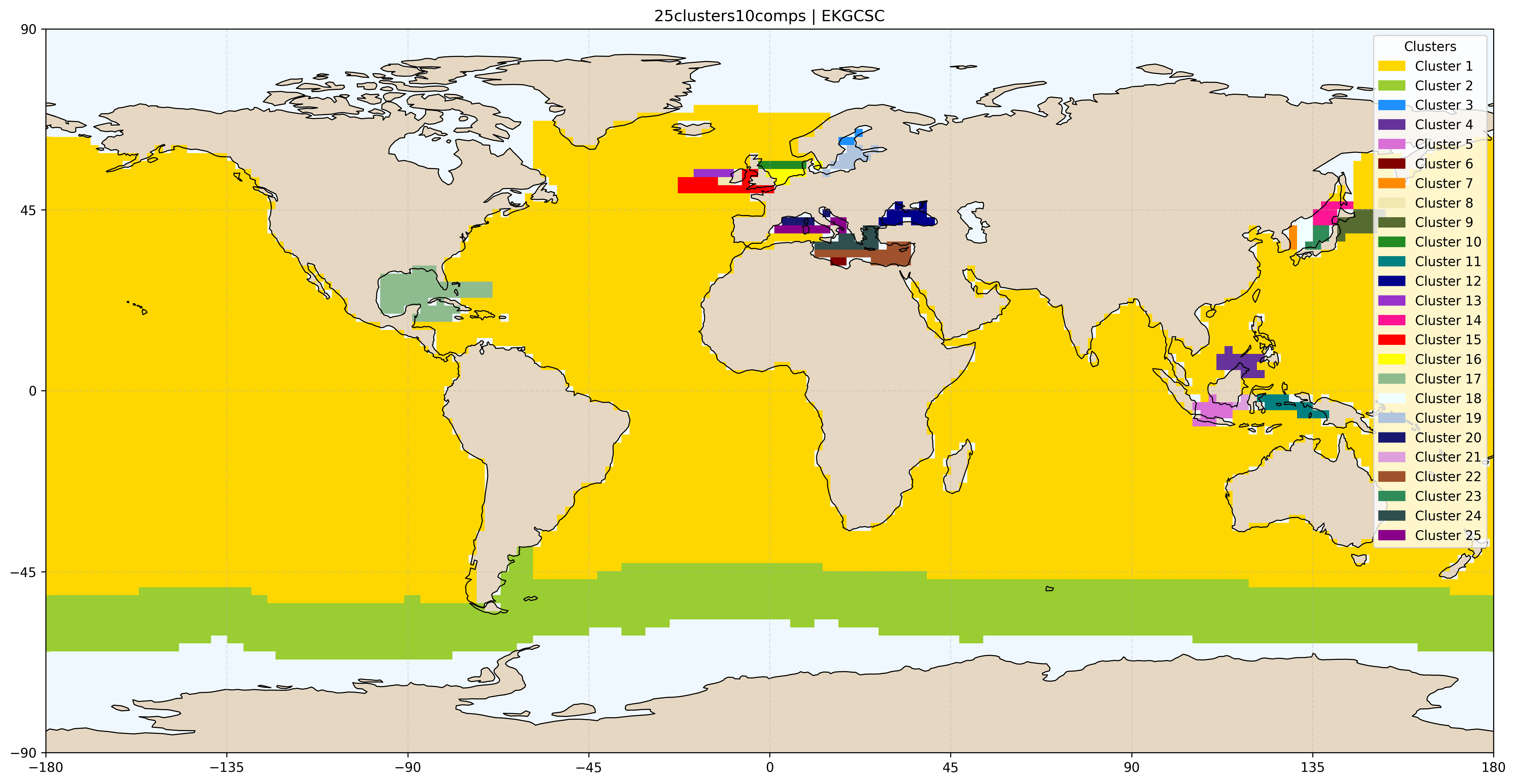}};
     \node [label=below:{(b) Clustering using EGCSC.}]at (7.8,0) {\includegraphics[width=0.42\textwidth]{fig/sla-EGCSC-clusters-globe-25.png}};\end{tikzpicture}
  \caption{Example of a clustering resulting from two HSI methods.}
  \label{fig:hsi-comp}
\end{figure*}

\begin{table*}[b]
\centering
\caption{Parameters configured for GraphConvSC}
\label{tab:graphcon-params}
\small
\begin{tabular}{|l|l|l|}
\hline
\textbf{Category}&\textbf{Parameter} & \textbf{Value} \\
Patch & \texttt{PATCH\_SIZE} & 15 \\
\hline
Clustering parameters & \texttt{N\_CLASSES\_LIST} & $\{8,15,20,25\}$ \\
\hline
Clustering parameters & \texttt{NB\_COMPS\_LIST} & $\{5,10,15,30\}$ \\
\hline
PCA & \texttt{n\_components} & \texttt{nb\_comps} from \texttt{NB\_COMPS\_LIST} \\
\hline
EGCSC & \texttt{regu\_coef} & $1\times10^{2}$ \\
\hline
EGCSC & \texttt{n\_neighbors} & 15 \\
\hline
EGCSC & \texttt{ro} & 0.8 \\
\hline
EGCSC & \texttt{n\_clusters} & \texttt{n\_classes} from \texttt{N\_CLASSES\_LIST} \\
\hline
EGCSC & \texttt{save\_affinity} & \texttt{False} \\
\hline
EKGCSC & \texttt{regu\_coef} & $1\times10^{3}$ \\
\hline
EKGCSC & \texttt{n\_neighbors} & 15 \\
\hline
EKGCSC & \texttt{ro} & 0.8 \\
\hline
EKGCSC & \texttt{gamma} & 0.5 \\
\hline
EKGCSC & \texttt{n\_clusters} & \texttt{n\_classes} from \texttt{N\_CLASSES\_LIST} \\
\hline
EKGCSC & \texttt{save\_affinity} & \texttt{False} \\
\hline
\end{tabular}
\end{table*}

\begin{table*}[t]
\centering
\caption{Number of connected components for GraphConvSC on filtered input data, grouped by number of clusters and subspace dimension $m'$. The second column indicates how many clusters should be found by the algorithm and columns 3-6 show how many connected components the resulting clustering has.}\label{tab:number_of_conn_comps_hsi}
\begin{subtable}{\textwidth}
\centering
\subcaption[]{Filtered input data}
\begin{tabular}{|c|c|c|c|c|c|}
\hline
\makecell{method} & $\#$ clusters wanted & $m'=5$ & $m'=10$ & $m'=15$ & $m'=30$ \\ \hline
EKGCSC & 8  & 28 & 24 & 24 & 25 \\ \hline
EGCSC  & 8  & 29 & 27 & 25 & 34 \\ \hline
EKGCSC & 15 & 34 & 37 & 39 & 38 \\ \hline
EGCSC  & 15 & 37 & 36 & 37 & 37 \\ \hline
EKGCSC & 20 & 48 & 45 & 45 & 44 \\ \hline
EGCSC  & 20 & 39 & 43 & 42 & 43 \\ \hline
EKGCSC & 25 & 59 & 53 & 49 & 50 \\ \hline
EGCSC  & 25 & 52 & 52 & 48 & 51 \\ \hline
\end{tabular}
\end{subtable}
\hfill
\begin{subtable}{\textwidth}
\centering
\subcaption[]{Unfiltered input data}
\begin{tabular}{|c|c|c|c|c|c|}
\hline
\makecell{method} & $\#$ clusters wanted & $m'=5$ & $m'=10$ & $m'=15$ & $m'=30$ \\ \hline
EKGCSC & 8  & 27 & 27 & 21 & 22 \\ \hline
EGCSC  & 8  & 24 & 21 & 24 & 18 \\ \hline
EKGCSC & 15 & 33 & 33 & 28 & 29 \\ \hline
EGCSC  & 15 & 30 & 36 & 34 & 31 \\ \hline
EKGCSC & 20 & 31 & 33 & 35 & 34 \\ \hline
EGCSC  & 20 & 43 & 48 & 33 & 33 \\ \hline
EKGCSC & 25 & 47 & 38 & 40 & 41 \\ \hline
EGCSC  & 25 & 61 & 54 & 53 & 39 \\ \hline
\end{tabular}
\end{subtable}
\end{table*}

\subsection{Data Preprocessing}\label{appendix-preprocessing}
Since the Earth is spherical and the grid is geodetic (latitude/longitude), one can not calculate the Gaussian kernel once and apply it to each data point. To our knowledge, existing tools do not provide an efficient, curvature-aware spatial smoothing pipeline; we therefore implement a parallelized spherical Gaussian filter that respects the Earth's curvature. 
Given a structured spatial grid $N$, as defined in \Cref{sec:grid_graph}, over the Earth's surface, where each $x \in N$ corresponds to a geodetic coordinate pair $(\phi_x, \lambda_x)$ in latitude and longitude, we define a smoothing operator based on a spherical Gaussian kernel. Let $S(t,x)$ represent the sea level anomaly data at time $t$ and spatial location $x$.
We construct a binary validity mask $V$ over $N$ by identifying all $x \in N$ for which $S(t,x)$ contains no invalid (missing) observations across all $t$. For each valid $x \in N, x \notin V$, we define its spherical neighborhood 
\[
\mathcal{B}_x = \{y \in N \setminus \{x\} \mid d_H(x,y)< r_c\}
\]
where $d_H$ is the great-circle distance computed using the Haversine formula and $r_c = 3\sigma$ is the cutoff radius. The Gaussian kernel is parameterized by a user-specified half-width~$h \in \mathbb{R}$, where $\sigma = \frac{h}{1.178}$. Each neighbor $y \in \mathcal{B}_x$ is assigned a weight $w_{x,y} = \exp (-\frac{d_H(x,y)}{2\sigma^2})$ and the central point x has a weight $w_{x,x} = 1 $. All weights are normalized $w_{x,y}^{norm} = \frac{w_{x,y}}{\sum_{y' \in \mathcal{B}_x \cup \{x\}} w_{x,y'}}$. The filtered time series at location $x$ is computed as 
\[
\widetilde{S}(t,x) = \sum_{y \in \mathcal{B} \cup \{x\}} w_{x,y}^{norm} \cdot S(t,y)
\]
where NaN-values are excluded to avoid an influence of the missing values from the weighted average. This operation can be performed independently for each time step $t$ and each point $x$. For scalability reasons the filter is executed in parallel using multiprocessing. The result is a spatially smoothed dataset, which has been calculated with respecting the Earth's curvature and with robustness to missing values.

\subsection{Subspace Clustering}\label{sec:additional-experiments}

In addition to the experiments presented in the main text, we conducted further analyses to compare our method to other subspace clustering methods and hyperspectral image segmentation algorithms. Please note that the methods discussed here were selected to be representatives of their respective fields, we performed additional experiments using other algorithms and chose these for their stability, feasibility and solution quality.

To compare our connected subspace clustering approach to state-of-the-art subspace clustering (without explicit connectivity constraints) we ran experiments using the code accompanying \cite{YouRV16}, who propose a method called \emph{sparse subspace clustering with orthogonal matching pursuit} (SSC-OMP), and \cite{you2016oracle}, who devised a method called \emph{elastic net subspace clustering} (EnSC).

These experiments were run for $k$ = 8, 15, 20 and 25 on unfiltered and filtered data. However, it is not possible to directly specify the subspace dimensions $m'$, so we are not able to include those in our results. We tried different configurations of the parameters and ended up with $\gamma = 50$ and otherwise used the default parameters as this produced the best results.
As both of these methods do not enforce connectivity, the resulting clustering is fragmented, especially when using SSC-OMP (see \Cref{fig:sc-comp}). On filtered data, the clusters formed by EnSC tend to be formed by multiple coherent connected components, while SSC-OMP consistently produces less coherent clusters. This is not surprising as SSC-OMP does not consider connectivity, while it is specifically integrated into EnSC. However, even EnSC does not enforce strict connectivity, but encourages the formation of more coherent clusters by combining sparse subspace clustering with least squares regression. The number of connected components of the subgraph induced by the clusters are reported in \Cref{tab:number_of_ssc-omp-ensc}. 
The subspace clustering cost of both methods are reported in \Cref{sec:results_appendix}.

\subsection{Hyperspectral Imaging}

In addition to the more traditional subspace clustering methods, we also investigated the field of hyperspectral image segmentation, as one explicit goal here is to produce spatially coherent clusters. 
We performed experiments using the code provided by \cite{cai2020graph} in their paper on \emph{Graph Convolutional Subspace Clustering} (GraphConvSC).

To ensure a fair comparison we modified the algorithm such that it handles spherical and incomplete input data. We ran these experiments for both \emph{efficient graph convolutional subspace clustering} (EGCSC) and \emph{efficient kernel graph convolutional subspace clustering} (EKGCSC), which are the two methods proposed by \cite{cai2020graph}.
The parameters were set as can be seen in \Cref{tab:graphcon-params}. We chose these  parameters as they produced the best results.
It is important to note that these methods are not designed to explicitly enforce connectivity. This can be seen in \Cref{tab:number_of_conn_comps_hsi}, where the number of connected components that are actually built by the algorithm proposed by \cite{cai2020graph} are listed. We additionally report the summed distances of the points in a cluster to the the subspace spanned by that cluster, i.e. the subspace clustering cost in the tables shown in \Cref{sec:results_appendix}. Figures showing examples of the resulting clusterings are given in \Cref{fig:hsi-comp}.
Interestingly, EKGCSC does produce less coherent clusters that EGCSC.

%% file: appendix/results_table.tex
% Every result table in this file is wider than one ALENEX column.  Keep the
% table bodies unchanged and route the existing table environments through
% full-width table* floats while this file is being read.
\let\alenexOriginalTable\table
\let\endalenexOriginalTable\endtable
\renewenvironment{table}[1][]{\begin{table*}[t]}{\end{table*}}

\subsection{Results}\label{sec:results_appendix}
In this section we present the results from our set of experiments. 
Each table contains the subspace clustering cost of the initial clustering methods described in \Cref{sec:initialclustering} in the corresponding column. The remaining columns report how the subspace clustering cost changes when using \Cref{alg:sc} with the corresponding connectivity methods as described in \Cref{sec:connectivity}.
Additionally, there are comparisons to the subspace clustering and HSI methods discussed in \Cref{sec:additional-experiments}. Those are the last four rows in each table. As they are meant as standalone comparison methods, we did not apply the merging procedures to them.
Please note that the methods SSC-OMP and EnSC do not allow for a specific $m'$ to be provided, which is why there is only one value in those rows.
\begin{table}[H]
\centering
\caption{Distances from points to subspaces using filtered input data, with 8 clusters. Blue indicates the best value in each row.}\label{tab:8_filtered}
\setlength{\tabcolsep}{2pt}  % default is 6pt
\resizebox{\textwidth}{!}{%
\begin{tabular}{|c|c|c|c|c|c|c|}
\hline
& $m'$ & \makecell{initial \\ clustering} & IterMerge & PostMerge & \makecell{Smooth\\Merge} & \makecell{Integrated\\Conn}\\ \hline
\makecell{Agglo-ST} & 
\begin{minipage}{0.05\textwidth}\centering\begin{tabular}{c}
5 \\
10 \\
15 \\
30
\end{tabular} \end{minipage}  
& 
\begin{minipage}{0.16\textwidth}\centering\begin{tabular}{c}
7692.49489 \\
6711.79823 \\
6058.78586 \\
4739.49195 \\\end{tabular} \end{minipage} 
& \begin{minipage}{0.16\textwidth}\centering\begin{tabular}{c} 
\textcolor{blue}{7449.85439 }   \\
\textcolor{blue}{6416.84714}    \\
5681.50449 \\
\textcolor{blue}{4275.62553 }   \\\end{tabular} \end{minipage} 
& \begin{minipage}{0.16\textwidth}\centering\begin{tabular}{c}
7465.55920   \\
6577.24879    \\
5885.93997    \\
4597.55834  \\\end{tabular} \end{minipage} 
& \begin{minipage}{0.16\textwidth}\centering\begin{tabular}{c}
7478.41837   \\
6567.92509    \\
5916.59557   \\
4602.77377    \\\end{tabular} \end{minipage} 
& \begin{minipage}{0.16\textwidth}\centering\begin{tabular}{c}
7606.71637   \\
6496.16303   \\
\textcolor{blue}{5678.59392}    \\
4300.00575    \\\end{tabular} \end{minipage}\\
\hline
%%%%%%%%%%%%%%%%%%%%%%%%%%%%%%%%%%%%%%% agl conn cl euclidean
\makecell{Conn-Agglo-Euc} &
\begin{minipage}{0.05\textwidth}\centering\begin{tabular}{c}
5 \\
10 \\
15 \\
30
\end{tabular} \end{minipage}  
& \begin{minipage}{0.16\textwidth}\centering\begin{tabular}{c} 
7577.79424    \\
6552.16401    \\
5892.12265    \\
4529.30297    \\
\end{tabular} \end{minipage} & \begin{minipage}{0.16\textwidth}\centering\begin{tabular}{c}  
\textcolor{blue}{7407.11306 }  \\
\textcolor{blue}{6434.20938}   \\
\textcolor{blue}{5656.52604}  \\
\textcolor{blue}{4310.62839}   \\ \end{tabular} \end{minipage} 
& \begin{minipage}{0.16\textwidth}\centering\begin{tabular}{c} 
7516.26784   \\
6533.40635   \\
5880.22969   \\
4524.68553   \\ \end{tabular} \end{minipage} & \begin{minipage}{0.16\textwidth}\centering\begin{tabular}{c} 
7521.30797   \\
6457.07786   \\
5867.63781   \\
4506.63351   \\ \end{tabular} \end{minipage} & \begin{minipage}{0.16\textwidth}\centering\begin{tabular}{c}
7472.77518  \\
6471.56399  \\
5793.80565  \\
4337.03429  \\ \end{tabular} \end{minipage}\\
\hline
%%%%%%%%%%%%%%%%%%%% agl conn cl thompson function
\makecell{Conn-Agglo-ST} &
\begin{minipage}{0.05\textwidth}\centering\begin{tabular}{c}
5 \\
10 \\
15 \\
30
\end{tabular} \end{minipage}   
& \begin{minipage}{0.16\textwidth}\centering\begin{tabular}{c} 
7665.13937   \\
6626.54510   \\
5925.35443   \\
4543.34794    \\
\end{tabular} \end{minipage} & \begin{minipage}{0.16\textwidth}\centering\begin{tabular}{c} 
\textcolor{blue}{7471.02903}   \\
\textcolor{blue}{6401.78462}    \\
5716.76295    \\
\textcolor{blue}{4328.07089}    \\ \end{tabular} \end{minipage} & \begin{minipage}{0.16\textwidth}\centering\begin{tabular}{c}
7480.37794   \\
6614.75710    \\
5891.04896    \\
4502.91717    \\ \end{tabular} \end{minipage} & \begin{minipage}{0.16\textwidth}\centering\begin{tabular}{c} 
7563.32227   \\
6517.96774  \\
\textcolor{blue}{5712.20077}    \\
4589.32804    \\ \end{tabular} \end{minipage} & \begin{minipage}{0.16\textwidth}\centering\begin{tabular}{c} 
7553.95301  \\
6522.68229   \\
5841.31301   \\
4335.89561   \\ \end{tabular} \end{minipage}\\
\hline
%%%%%%%%%%%%%%%%%%%%%%%%%%%%%%%%%%%% k-means
\makecell{Conn-KMeans++ }&
\begin{minipage}{0.05\textwidth}\centering\begin{tabular}{c}
5 \\
10 \\
15 \\
30
\end{tabular} \end{minipage}   
& \begin{minipage}{0.16\textwidth}\centering\begin{tabular}{c}  
7672.81813    \\
6684.77629   \\
6019.41239    \\
4657.46825    \\
\end{tabular} \end{minipage} & \begin{minipage}{0.16\textwidth}\centering\begin{tabular}{c} 
7456.54788   \\
\textcolor{blue}{6404.43239}    \\
5741.98456    \\
4367.77138    \\ \end{tabular} \end{minipage} & \begin{minipage}{0.16\textwidth}\centering\begin{tabular}{c}
7618.64891   \\
6586.35707    \\
5930.38415    \\
4572.23265    \\ \end{tabular} \end{minipage} & \begin{minipage}{0.16\textwidth}\centering\begin{tabular}{c} 
7536.54157  \\
6517.60083   \\
5903.71485    \\
4574.05150   \\ \end{tabular} \end{minipage} & \begin{minipage}{0.16\textwidth}\centering\begin{tabular}{c}
\textcolor{blue}{7412.93632}   \\
6530.95691    \\
\textcolor{blue}{5725.26091}    \\
\textcolor{blue}{4360.19843}   \\ \end{tabular} \end{minipage}\\
\hline
%%%%%%%%%%%%%%%%%%%%%%%%%%%%%%%%%%%%% wards method
\makecell{Conn-Ward} &
\begin{minipage}{0.05\textwidth}\centering\begin{tabular}{c}
5 \\
10 \\
15 \\
30
\end{tabular} \end{minipage}  
& \begin{minipage}{0.16\textwidth}\centering\begin{tabular}{c} 
7818.73543  \\
6884.08856   \\
6250.65279    \\
4952.76399    \\
\end{tabular} \end{minipage} & 
\begin{minipage}{0.16\textwidth}\centering\begin{tabular}{c} 
7559.41281   \\
\textcolor{blue}{6390.55542}   \\
\textcolor{blue}{5728.35510}    \\
\textcolor{blue}{4561.76944}    \\ \end{tabular} \end{minipage} & 
\begin{minipage}{0.16\textwidth}\centering\begin{tabular}{c}
7547.78283   \\
6625.74611    \\
5948.23946    \\
4773.37472    \\ \end{tabular} \end{minipage} & 
\begin{minipage}{0.16\textwidth}\centering\begin{tabular}{c}
\textcolor{blue}{7519.00091}   \\
6605.33257    \\
5941.08003    \\
4717.35659    \\ \end{tabular} \end{minipage} & 
\begin{minipage}{0.16\textwidth}\centering\begin{tabular}{c}
7662.99054   \\
6721.36554    \\
6059.41774    \\
4770.03843    \\ \end{tabular} \end{minipage}\\
\hline
%%%%%%%%%%%%%%%%%%%%%%%%%%%%%%%% graphconvsc ekcsc %%%%%%%%%
\makecell{GraphConvSC \\ EKGCSC} &
\begin{minipage}{0.05\textwidth}\centering\begin{tabular}{c}
5 \\
10 \\
15 \\
30
\end{tabular} \end{minipage}  
& \begin{minipage}{0.16\textwidth}\centering\begin{tabular}{c}
13124.82028 \\
13088.99551 \\
13084.56918 \\
13115.11044 \\
\end{tabular} \end{minipage} & - & - & - & - \\
\hline
%%%%%%%%%%%%%%%%%%%%%%%% graphconvsc egcsc %%%%%%%%%%%%%%
\makecell{GraphConvSC \\ EGCSC} &
\begin{minipage}{0.05\textwidth}\centering\begin{tabular}{c}
5 \\
10 \\
15 \\
30
\end{tabular} \end{minipage}  
& \begin{minipage}{0.16\textwidth}\centering\begin{tabular}{c}
13128.60392 \\
13091.91857 \\
13250.86692 \\
13136.15622 \\
\end{tabular} \end{minipage} & - & - & - & - \\
\hline
%%%%%%%%%%%%%%%%%%%%%%%% SSC OMP %%%%%%%%%%%%%%
\makecell{SSC-OMP} &
& 15276.42554 & - & - & - & - \\
\hline
%%%%%%%%%%%%%%%%%%%%%%%% EnSC %%%%%%%%%%%%%%
\makecell{EnSC} &
& 14088.99291 & - & - & - & - \\
\hline
\end{tabular}%
}
\end{table}

%%%%%%%%%%%%%%%%%%%%%%%%%%%%%%%%%%%%%%%%%%%%%%%%%%%% unfiltered, 8 cluster

\begin{table}[H]
\centering
\caption{Distances from points to subspaces using unfiltered input data, with 8 clusters. Blue indicates the best value in each row. Note that the algorithm Agglo-ST performs poorly on unfiltered data (see \Cref{fig:thompson_filtering}), resulting in clusters unsuitable for Conn-Subspace. This is because one large cluster dominates while the others are too small to compute subspaces for $m'$ dimensions. This is why the corresponding values are missing in the table.}\label{tab:8_unfiltered}
\setlength{\tabcolsep}{2pt}  % default is 6pt
\resizebox{\textwidth}{!}{%
\begin{tabular}{|c|c|c|c|c|c|c|}
\hline
%%%%%%%%%%%%%%%%%%%%%%%%%%%%%%%%%%%% Thompson clustering
& $m'$ & \makecell{initial \\ clustering} & IterMerge & PostMerge & \makecell{Smooth\\Merge} & \makecell{Integrated\\Conn}\\ \hline
\makecell{Agglo-ST}& 
\begin{minipage}{0.05\textwidth}\centering\begin{tabular}{c}
5 \\
10 \\
15 \\
30
\end{tabular} \end{minipage}  
&\begin{minipage}{0.16\textwidth}\centering\begin{tabular}{c}  
8357.61138  \\
7465.22683  \\
\textcolor{blue}{6934.41398}  \\
\textcolor{blue}{5827.25554}  \\
\end{tabular} \end{minipage} & 
\begin{minipage}{0.16\textwidth}\centering\begin{tabular}{c}
-                 \\
-                 \\
-                 \\
-                 \\ \end{tabular} \end{minipage} & 
\begin{minipage}{0.16\textwidth}\centering\begin{tabular}{c} 
\textcolor{blue}{8269.62100}  \\
\textcolor{blue}{7268.01158}  \\
-             \\
-              \\ \end{tabular} \end{minipage} & 
\begin{minipage}{0.16\textwidth}\centering\begin{tabular}{c}
-                 \\
-                 \\
-                 \\
-                 \\ \end{tabular} \end{minipage} & 
\begin{minipage}{0.16\textwidth}\centering\begin{tabular}{c} 
8363.89983   \\
7461.79679   \\
-   \\
-   \\ \end{tabular} \end{minipage}\\
\hline
%%%%%%%%%%%%%%%%%%%%%%%%%%%%%%%%%%%%%%%%% connected euclidean
\makecell{Conn-Agglo-Euc} &
\begin{minipage}{0.05\textwidth}\centering\begin{tabular}{c}
5 \\
10 \\
15 \\
30
\end{tabular} \end{minipage}  &
\begin{minipage}{0.16\textwidth}\centering\begin{tabular}{c}  
7731.16245  \\
6813.48844    \\
6189.81632    \\
4902.32309    \\
\end{tabular} \end{minipage} & \begin{minipage}{0.16\textwidth}\centering\begin{tabular}{c}  
\textcolor{blue}{7544.00497}   \\
6589.54375    \\
\textcolor{blue}{5877.08022}   \\
\textcolor{blue}{4494.57215}   \\ \end{tabular} \end{minipage} & \begin{minipage}{0.16\textwidth}\centering\begin{tabular}{c}
7613.92384       \\
6732.86394     \\
5974.47518      \\
4755.52725     \\\end{tabular} \end{minipage} & \begin{minipage}{0.16\textwidth}\centering\begin{tabular}{c}  
7581.09045     \\
\textcolor{blue}{6526.79913}    \\
5977.69428    \\
4637.71137    \\ \end{tabular} \end{minipage} & 
\begin{minipage}{0.16\textwidth}\centering\begin{tabular}{c}  
7607.04228     \\
6703.52087    \\
6001.12751     \\
4719.07835     \\ \end{tabular} \end{minipage}\\
\hline
%%%%%%%%%%%%%%%%%%%%%%%%%%%%%%%%%%% connected thompson function
\makecell{Conn-Agglo-ST}&
\begin{minipage}{0.05\textwidth}\centering\begin{tabular}{c}
5 \\
10 \\
15 \\
30
\end{tabular} \end{minipage}   
& \begin{minipage}{0.16\textwidth}\centering\begin{tabular}{c} 
7587.75423 \\ 
6517.61474 \\ 
5821.57802 \\ 
4433.68417 \\ 
\end{tabular} \end{minipage} & 
\begin{minipage}{0.16\textwidth}\centering\begin{tabular}{c}   
\textcolor{blue}{7459.38257} \\ 
\textcolor{blue}{6413.93545} \\ 
\textcolor{blue}{5680.3945} \\ 
4389.54962 \\  \end{tabular} \end{minipage} & 
\begin{minipage}{0.16\textwidth}\centering\begin{tabular}{c} 
7616.73677 \\ 
6562.13264 \\ 
5802.8574 \\ 
4449.47911 \\  \end{tabular} \end{minipage} & 
\begin{minipage}{0.16\textwidth}\centering\begin{tabular}{c} 
7554.43019 \\ 
6421.3932 \\ 
5873.59309 \\ 
4454.67732 \\  \end{tabular} \end{minipage} & 
\begin{minipage}{0.16\textwidth}\centering\begin{tabular}{c}
7477.76603 \\ 
6438.73443 \\ 
5739.71282 \\ 
\textcolor{blue}{4361.23423} \\  \end{tabular} \end{minipage}\\
\hline
%%%%%%%%%%%%%%%%%%%%%%%%%%%%%%%%%%%% k-means
\makecell{Conn-KMeans++} &
\begin{minipage}{0.05\textwidth}\centering\begin{tabular}{c}
5 \\
10 \\
15 \\
30
\end{tabular} \end{minipage}  
& \begin{minipage}{0.16\textwidth}\centering\begin{tabular}{c} 
7787.54983    \\
6853.38225    \\
6208.53329    \\
4908.12907    \\
\end{tabular} \end{minipage} & 
\begin{minipage}{0.16\textwidth}\centering\begin{tabular}{c}  
\textcolor{blue}{7402.50782 }  \\
6443.82804   \\
5756.01577   \\
\textcolor{blue}{4314.02620}   \\ \end{tabular} \end{minipage} & 
\begin{minipage}{0.16\textwidth}\centering\begin{tabular}{c} 
7576.29373   \\
6519.65702   \\
5861.48456   \\
4664.09250   \\ \end{tabular} \end{minipage} & 
\begin{minipage}{0.16\textwidth}\centering\begin{tabular}{c} 
7580.04322   \\
6710.61413   \\
5941.10418    \\
4646.05051   \\ \end{tabular} \end{minipage} & 
\begin{minipage}{0.16\textwidth}\centering\begin{tabular}{c}  
7511.55950    \\
\textcolor{blue}{6434.85064 }   \\
\textcolor{blue}{5744.50432 }    \\
4669.39230    \\ \end{tabular} \end{minipage}\\
\hline
%%%%%%%%%%%%%%%%%%%%%%%%%%%%%%%%%%%%%%%%% wards method
\makecell{Conn-Ward} &
\begin{minipage}{0.05\textwidth}\centering\begin{tabular}{c}
5 \\
10 \\
15 \\
30
\end{tabular} \end{minipage}  
& \begin{minipage}{0.16\textwidth}\centering\begin{tabular}{c}  
7731.16195      \\
6813.49009      \\
6189.85300      \\
4902.18010      \\
\end{tabular} \end{minipage} & \begin{minipage}{0.16\textwidth}\centering\begin{tabular}{c}  
7544.00447      \\
6609.47078      \\
\textcolor{blue}{5866.06992 }     \\
\textcolor{blue}{4502.06469}      \\ \end{tabular} \end{minipage} & 
\begin{minipage}{0.16\textwidth}\centering\begin{tabular}{c} 
\textcolor{blue}{7534.85592}        \\
6577.85847        \\
5997.38423         \\
4722.56464        \\ \end{tabular} \end{minipage} & \begin{minipage}{0.16\textwidth}\centering\begin{tabular}{c} 
7569.99355     \\
\textcolor{blue}{6532.33852 }    \\
5991.65580     \\
4641.04607      \\ \end{tabular} \end{minipage} & \begin{minipage}{0.16\textwidth}\centering\begin{tabular}{c} 
7607.04291    \\
6703.51083    \\
6000.58657    \\
4719.93802    \\ \end{tabular} \end{minipage}\\
\hline
%%%%%%%%%%%%%%%%%%%%%%%%%%%%%%%% graphconvsc ekcsc %%%%%%%%%
\makecell{GraphConvSC \\ EKGCSC} &
\begin{minipage}{0.05\textwidth}\centering\begin{tabular}{c}
5 \\
10 \\
15 \\
30
\end{tabular} \end{minipage}  
& \begin{minipage}{0.16\textwidth}\centering\begin{tabular}{c}
13529.17649 \\
13531.45765 \\
13720.32976 \\
13721.31299 \\
\end{tabular} \end{minipage} & - & - & - & - \\
\hline
%%%%%%%%%%%%%%%%%%%%%%%% graphconvsc egcsc %%%%%%%%%%%%%%
\makecell{GraphConvSC \\ EGCSC} &
\begin{minipage}{0.05\textwidth}\centering\begin{tabular}{c}
5 \\
10 \\
15 \\
30
\end{tabular} \end{minipage}  
& \begin{minipage}{0.16\textwidth}\centering\begin{tabular}{c}
13534.20219 \\
13730.57769 \\
13717.10833 \\
13890.14598 \\
\end{tabular} \end{minipage} & - & - & - & - \\
\hline
%%%%%%%%%%%%%%%%%%%%%%%% SSC OMP %%%%%%%%%%%%%%
\makecell{SSC-OMP} &
& 15122.25802 & - & - & - & - \\
\hline
%%%%%%%%%%%%%%%%%%%%%%%% EnSC %%%%%%%%%%%%%%
\makecell{EnSC} &
& 14900.37739 & - & - & - & - \\
\hline
\end{tabular}%
}
\end{table}

%%%%%%%%%%%%%%%%%%%%%%%%%%%%%%%%%%%%%% filtered 15 cluster 

\begin{table}[H]
\centering
\caption{Distances from points to subspaces using filtered input data, with 15 clusters. Blue indicates the best value in each row.}\label{tab:15_filtered}
\setlength{\tabcolsep}{2pt}  % default is 6pt
\resizebox{\textwidth}{!}{%
\begin{tabular}{|c|c|c|c|c|c|c|}
\hline
%%%%%%%%%%%%%%%%%%%%%%%%%%%%%%%%%%%% Thompson clustering
& $m'$ & \makecell{initial \\ clustering} & IterMerge & PostMerge & \makecell{Smooth\\Merge} & \makecell{Integrated\\Conn}\\ \hline
\makecell{Agglo-ST} &
\begin{minipage}{0.05\textwidth}\centering\begin{tabular}{c}
5 \\
10 \\
15 \\
30
\end{tabular} \end{minipage} & \begin{minipage}{0.16\textwidth}\centering\begin{tabular}{c}
7401.75317 \\
6373.87940 \\
5692.49488 \\
4324.59394 \\
\end{tabular} \end{minipage} & \begin{minipage}{0.16\textwidth}\centering\begin{tabular}{c}
\textcolor{blue}{6928.48755 } \\
\textcolor{blue}{5866.01001} \\
\textcolor{blue}{5187.14312 } \\
\textcolor{blue}{3737.01157} \\
\end{tabular} \end{minipage} & \begin{minipage}{0.16\textwidth}\centering\begin{tabular}{c}
7017.48275 \\
5936.91435 \\
5206.89874 \\
3940.42473 \\
\end{tabular} \end{minipage} & \begin{minipage}{0.16\textwidth}\centering\begin{tabular}{c}
7000.30368 \\
6085.79398 \\
5318.16172 \\
3949.27144 \\
\end{tabular} \end{minipage} & \begin{minipage}{0.16\textwidth}\centering\begin{tabular}{c}
7153.61319 \\
6261.35012 \\
5538.33864 \\
4027.25431 \\
\end{tabular} \end{minipage}\\
\hline
%%%%%%%%%%%%%%%%%%%%%%%%%%%%%%%%%%%%%%%%% aggl connected cl euclidean dist
\makecell{Conn-Agglo-Euc} &
\begin{minipage}{0.05\textwidth}\centering\begin{tabular}{c}
5 \\
10 \\
15 \\
30
\end{tabular} \end{minipage} & \begin{minipage}{0.16\textwidth}\centering\begin{tabular}{c}
7114.02254 \\
6072.58344 \\
5370.74695 \\
3956.67356 \\
\end{tabular} \end{minipage} & \begin{minipage}{0.16\textwidth}\centering\begin{tabular}{c}
\textcolor{blue}{6850.53328} \\
{5847.53931} \\
\textcolor{blue}{5052.96718} \\
\textcolor{blue}{3655.53732 } \\
\end{tabular} \end{minipage} & \begin{minipage}{0.16\textwidth}\centering\begin{tabular}{c}
7123.69161 \\
6099.05409 \\
5249.62310 \\
3831.89306 \\
\end{tabular} \end{minipage} & \begin{minipage}{0.16\textwidth}\centering\begin{tabular}{c}
7015.46846 \\
6014.41582 \\
5168.60667 \\
3816.64460 \\
\end{tabular} \end{minipage} & \begin{minipage}{0.16\textwidth}\centering\begin{tabular}{c}
6928.5687 \\
\textcolor{blue}{5828.97291} \\
5141.67013 \\
3724.09671 \\
\end{tabular} \end{minipage}\\
\hline
%%%%%%%%%%%%%%%%%%%%%%%%%%%%%%%%%%% aggl connected cl with thompson distance
\makecell{Conn-Agglo-ST} &
\begin{minipage}{0.05\textwidth}\centering\begin{tabular}{c}
5 \\
10 \\
15 \\
30
\end{tabular} \end{minipage} & \begin{minipage}{0.16\textwidth}\centering\begin{tabular}{c}
7165.77020 \\
6043.90989 \\
5286.28692 \\
3830.37398 \\
\end{tabular} \end{minipage} & \begin{minipage}{0.16\textwidth}\centering\begin{tabular}{c}
\textcolor{blue}{6900.80387 } \\
\textcolor{blue}{5738.33263 } \\
\textcolor{blue}{4995.24785} \\
\textcolor{blue}{3633.49710} \\
\end{tabular} \end{minipage} & \begin{minipage}{0.16\textwidth}\centering\begin{tabular}{c}
7168.83511 \\
5971.14997 \\
5007.97436 \\
3829.96746 \\
\end{tabular} \end{minipage} & \begin{minipage}{0.16\textwidth}\centering\begin{tabular}{c}
7124.80622 \\
5908.16445 \\
5061.30361 \\
3678.17079 \\
\end{tabular} \end{minipage} & \begin{minipage}{0.16\textwidth}\centering\begin{tabular}{c}
7004.66985 \\
5825.62817 \\
5053.81309 \\
3708.89951 \\
\end{tabular} \end{minipage}\\
\hline
%%%%%%%%%%%%%%%%%%%%%%%%%%%%%%%%%%%% k-means
\makecell{Conn-KMeans++} &
\begin{minipage}{0.05\textwidth}\centering\begin{tabular}{c}
5 \\
10 \\
15 \\
30
\end{tabular} \end{minipage} & \begin{minipage}{0.16\textwidth}\centering\begin{tabular}{c}
7252.59653 \\
6185.78104 \\
5460.38901 \\
4035.94281 \\
\end{tabular} \end{minipage} & \begin{minipage}{0.16\textwidth}\centering\begin{tabular}{c}
\textcolor{blue}{6935.95820} \\
\textcolor{blue}{5794.88755} \\
\textcolor{blue}{4994.50226} \\
\textcolor{blue}{3659.64386} \\
\end{tabular} \end{minipage} & \begin{minipage}{0.16\textwidth}\centering\begin{tabular}{c}
7102.93392 \\
5955.52421 \\
5210.23140 \\
3765.80204 \\
\end{tabular} \end{minipage} & \begin{minipage}{0.16\textwidth}\centering\begin{tabular}{c}
7120.01640 \\
5835.30837 \\
4974.27926 \\
3617.59620 \\
\end{tabular} \end{minipage} & \begin{minipage}{0.16\textwidth}\centering\begin{tabular}{c}
6987.27644 \\
5815.82064 \\
5128.15819 \\
3668.58070 \\
\end{tabular} \end{minipage}\\
\hline
%%%%%%%%%%%%%%%%%%%%%%%%%%%%%%%%%%%%%%%%% wards method
\makecell{Conn-Ward}  &
\begin{minipage}{0.05\textwidth}\centering\begin{tabular}{c}
5 \\
10 \\
15 \\
30
\end{tabular} \end{minipage} & \begin{minipage}{0.16\textwidth}\centering\begin{tabular}{c}
7220.75107 \\
6267.09090 \\
5616.07627 \\
4308.16897 \\
\end{tabular} \end{minipage} & \begin{minipage}{0.16\textwidth}\centering\begin{tabular}{c}
\textcolor{blue}{6989.61962} \\
\textcolor{blue}{5934.10865} \\
\textcolor{blue}{5080.41592} \\
\textcolor{blue}{3905.82830} \\
\end{tabular} \end{minipage} & \begin{minipage}{0.16\textwidth}\centering\begin{tabular}{c}
7094.40494 \\
6066.03266 \\
5458.88736 \\
4210.32269 \\
\end{tabular} \end{minipage} & \begin{minipage}{0.16\textwidth}\centering\begin{tabular}{c}
7017.41398 \\
6039.21859 \\
5142.28195 \\
3879.76737 \\
\end{tabular} \end{minipage} & \begin{minipage}{0.16\textwidth}\centering\begin{tabular}{c}
7018.86197 \\
6036.39050 \\
5372.64187 \\
4000.63856 \\
\end{tabular} \end{minipage}\\
\hline
%%%%%%%%%%%%%%%%%%%%%%%%%%%%%%%% graphconvsc ekcsc %%%%%%%%%
\makecell{GraphConvSC \\ EKGCSC} &
\begin{minipage}{0.05\textwidth}\centering\begin{tabular}{c}
5 \\
10 \\
15 \\
30
\end{tabular} \end{minipage}  
& \begin{minipage}{0.16\textwidth}\centering\begin{tabular}{c}
12816.98138 \\
12696.11403 \\
12629.34658 \\
12607.53272 \\
\end{tabular} \end{minipage} & - & - & - & - \\
\hline
%%%%%%%%%%%%%%%%%%%%%%%% graphconvsc egcsc %%%%%%%%%%%%%%
\makecell{GraphConvSC \\ EGCSC} &
\begin{minipage}{0.05\textwidth}\centering\begin{tabular}{c}
5 \\
10 \\
15 \\
30
\end{tabular} \end{minipage}  
& \begin{minipage}{0.16\textwidth}\centering\begin{tabular}{c}
12595.44702 \\
12671.02373 \\
12715.13632 \\
12833.03298 \\
\end{tabular} \end{minipage} & - & - & - & - \\
\hline
%%%%%%%%%%%%%%%%%%%%%%%% SSC OMP %%%%%%%%%%%%%%
\makecell{SSC-OMP} &
& 15056.41022 & - & - & - & - \\
\hline
%%%%%%%%%%%%%%%%%%%%%%%% EnSC %%%%%%%%%%%%%%
\makecell{EnSC} &
& 13462.20711 & - & - & - & - \\
\hline
\end{tabular}%
}
\end{table}

%%%%%%%%%%%%%%%%%%%%%%%%%%%%%%%%%%%%%% unfiltered 15 cluster 

\begin{table}[H]
\centering
\caption{Distances from points to subspaces using unfiltered input data, with 15 clusters. Blue indicates the best value in each row.}\label{tab:15_unfiltered}
\setlength{\tabcolsep}{2pt}  % default is 6pt
\resizebox{\textwidth}{!}{%
\begin{tabular}{|c|c|c|c|c|c|c|}
\hline
%%%%%%%%%%%%%%%%%%%%%%%%%%%%%%%%%%%% Thompson clustering
& $m'$ & \makecell{initial \\ clustering} & IterMerge & PostMerge & \makecell{Smooth\\Merge} & \makecell{Integrated\\Conn}\\ \hline
\makecell{Agglo-ST} &
\begin{minipage}{0.05\textwidth}\centering\begin{tabular}{c}
5 \\
10 \\
15 \\
30
\end{tabular} \end{minipage} & \begin{minipage}{0.16\textwidth}\centering\begin{tabular}{c}
8138.18456 \\
7204.38252 \\
6658.85384 \\
5568.49755 \\
\end{tabular} \end{minipage} & \begin{minipage}{0.16\textwidth}\centering\begin{tabular}{c}
\textcolor{blue}{7053.37876 } \\
\textcolor{blue}{6334.25977} \\
\textcolor{blue}{5778.67113} \\
\textcolor{blue}{4440.49430 } \\
\end{tabular} \end{minipage} & \begin{minipage}{0.16\textwidth}\centering\begin{tabular}{c}
7496.84981 \\
6598.76775 \\
6296.71185 \\
4973.61062 \\
\end{tabular} \end{minipage} & \begin{minipage}{0.16\textwidth}\centering\begin{tabular}{c}
7818.48633 \\
6779.11463 \\
6081.50442 \\
5276.09128 \\
\end{tabular} \end{minipage} & \begin{minipage}{0.16\textwidth}\centering\begin{tabular}{c}
7964.90935 \\
6955.51336 \\
6509.29316 \\
5469.03702 \\
\end{tabular} \end{minipage}\\
\hline
%%%%%%%%%%%%%%%%%%%%%%%%%%%%%%%%%%%%%%%%% aggl connected cl euclidean dist
\makecell{Conn-Agglo-Euc}  &
\begin{minipage}{0.05\textwidth}\centering\begin{tabular}{c}
5 \\
10 \\
15 \\
30
\end{tabular} \end{minipage} & \begin{minipage}{0.16\textwidth}\centering\begin{tabular}{c}
7300.91007 \\
6357.32784 \\
5737.88249 \\
4460.20328 \\
\end{tabular} \end{minipage} & \begin{minipage}{0.16\textwidth}\centering\begin{tabular}{c}
7116.94334 \\
\textcolor{blue}{5987.91657} \\
\textcolor{blue}{5274.70075} \\
\textcolor{blue}{3862.28418} \\
\end{tabular} \end{minipage} & \begin{minipage}{0.16\textwidth}\centering\begin{tabular}{c}
7180.58798 \\
6187.01996 \\
5301.06031 \\
4005.64651 \\
\end{tabular} \end{minipage} & \begin{minipage}{0.16\textwidth}\centering\begin{tabular}{c}
\textcolor{blue}{7094.45986 } \\
6074.20510 \\
5448.38671 \\
4225.21069 \\
\end{tabular} \end{minipage} & \begin{minipage}{0.16\textwidth}\centering\begin{tabular}{c}
7152.81688 \\
6191.78486 \\
5560.49203 \\
4274.54517 \\
\end{tabular} \end{minipage}\\
\hline
%%%%%%%%%%%%%%%%%%%%%%%%%%%%%%%%%%% aggl connected cl with thompson distance
\makecell{Conn-Agglo-ST} &
\begin{minipage}{0.05\textwidth}\centering\begin{tabular}{c}
5 \\
10 \\
15 \\
30
\end{tabular} \end{minipage} & \begin{minipage}{0.16\textwidth}\centering\begin{tabular}{c}
7119.17181 \\
6005.0503 \\
5268.05087 \\
3823.02996 \\
\end{tabular} \end{minipage} & \begin{minipage}{0.16\textwidth}\centering\begin{tabular}{c}
\textcolor{blue}{6918.34345} \\
\textcolor{blue}{5778.90979 } \\
5104.49715 \\
3643.09308 \\
\end{tabular} \end{minipage} & \begin{minipage}{0.16\textwidth}\centering\begin{tabular}{c}
7029.95831 \\
5883.24081 \\
\textcolor{blue}{5099.97031} \\
3792.69081 \\
\end{tabular} \end{minipage} & \begin{minipage}{0.16\textwidth}\centering\begin{tabular}{c}
7124.18824 \\
5977.09426 \\
5085.76466 \\
\textcolor{blue}{3616.81066} \\
\end{tabular} \end{minipage} & \begin{minipage}{0.16\textwidth}\centering\begin{tabular}{c}
6939.76197 \\
5838.78528 \\
5116.52196 \\
3736.21723 \\
\end{tabular} \end{minipage}\\
\hline
%%%%%%%%%%%%%%%%%%%%%%%%%%%%%%%%%%%% k-means
\makecell{Conn-KMeans++} &
\begin{minipage}{0.05\textwidth}\centering\begin{tabular}{c}
5 \\
10 \\
15 \\
30
\end{tabular} \end{minipage} & \begin{minipage}{0.16\textwidth}\centering\begin{tabular}{c}
7237.21547 \\
6140.85496 \\
5410.36757 \\
3960.26859 \\
\end{tabular} \end{minipage} & \begin{minipage}{0.16\textwidth}\centering\begin{tabular}{c}
\textcolor{blue}{6961.75097} \\
\textcolor{blue}{5821.98319} \\
\textcolor{blue}{5061.52510} \\
\textcolor{blue}{3739.47498} \\
\end{tabular} \end{minipage} & \begin{minipage}{0.16\textwidth}\centering\begin{tabular}{c}
7091.07433 \\
5993.05102 \\
5298.03961 \\
3944.39825 \\
\end{tabular} \end{minipage} & \begin{minipage}{0.16\textwidth}\centering\begin{tabular}{c}
6974.68506 \\
6030.58822 \\
5288.48289 \\
3956.10770 \\
\end{tabular} \end{minipage} & \begin{minipage}{0.16\textwidth}\centering\begin{tabular}{c}
7037.43763 \\
6007.85102 \\
5291.30412 \\
3817.96416 \\
\end{tabular} \end{minipage}\\
\hline
%%%%%%%%%%%%%%%%%%%%%%%%%%%%%%%%%%%%%%%%% wards method
\makecell{Conn-Ward} &
\begin{minipage}{0.05\textwidth}\centering\begin{tabular}{c}
5 \\
10 \\
15 \\
30
\end{tabular} \end{minipage} & \begin{minipage}{0.16\textwidth}\centering\begin{tabular}{c}
7300.91029 \\
6357.33007 \\
5737.91865 \\
4460.12404 \\
\end{tabular} \end{minipage} & \begin{minipage}{0.16\textwidth}\centering\begin{tabular}{c}
7113.16163 \\
6045.03966 \\
\textcolor{blue}{5321.37167} \\
\textcolor{blue}{3793.78203} \\
\end{tabular} \end{minipage} & \begin{minipage}{0.16\textwidth}\centering\begin{tabular}{c}
7194.19093 \\
6177.00189 \\
5353.26835 \\
4001.61043 \\
\end{tabular} \end{minipage} & \begin{minipage}{0.16\textwidth}\centering\begin{tabular}{c}
\textcolor{blue}{7074.05179} \\
\textcolor{blue}{5944.17437} \\
5542.68503 \\
4223.34628 \\
\end{tabular} \end{minipage} & \begin{minipage}{0.16\textwidth}\centering\begin{tabular}{c}
7152.80277 \\
6188.96634 \\
5567.75397 \\
4059.74273 \\
\end{tabular} \end{minipage}\\
\hline
%%%%%%%%%%%%%%%%%%%%%%%%%%%%%%%% graphconvsc ekcsc %%%%%%%%%
\makecell{GraphConvSC \\ EKGCSC} &
\begin{minipage}{0.05\textwidth}\centering\begin{tabular}{c}
5 \\
10 \\
15 \\
30
\end{tabular} \end{minipage}  
& \begin{minipage}{0.16\textwidth}\centering\begin{tabular}{c}
13357.03370 \\
13437.32476 \\
13679.37267 \\
13684.34093 \\
\end{tabular} \end{minipage} & - & - & - & - \\
\hline
%%%%%%%%%%%%%%%%%%%%%%%% graphconvsc egcsc %%%%%%%%%%%%%%
\makecell{GraphConvSC \\ EGCSC} &
\begin{minipage}{0.05\textwidth}\centering\begin{tabular}{c}
5 \\
10 \\
15 \\
30
\end{tabular} \end{minipage}  
& \begin{minipage}{0.16\textwidth}\centering\begin{tabular}{c}
13480.88323 \\
13266.20388 \\
13455.86709 \\
13666.36885 \\
\end{tabular} \end{minipage} & - & - & - & - \\
\hline
%%%%%%%%%%%%%%%%%%%%%%%% SSC OMP %%%%%%%%%%%%%%
\makecell{SSC-OMP} &
& 14875.43995 & - & - & - & - \\
\hline
%%%%%%%%%%%%%%%%%%%%%%%% EnSC %%%%%%%%%%%%%%
\makecell{EnSC} &
& 14671.32921 & - & - & - & - \\
\hline
\end{tabular}%
}
\end{table}

%%%%%%%%%%%%%%%%%%%%%%%%%%%%%%%%%%%%%% filtered 20 cluster 

\begin{table}[H]
\centering
\caption{Distances from points to subspaces using filtered input data, with 20 clusters. Blue indicates the best value in each row.}\label{tab:20_filtered}
\setlength{\tabcolsep}{2pt}  % default is 6pt
\resizebox{\textwidth}{!}{%
\begin{tabular}{|c|c|c|c|c|c|c|}
\hline
%%%%%%%%%%%%%%%%%%%%%%%%%%%%%%%%%%%% Thompson clustering
& $m'$ & \makecell{initial \\ clustering} & IterMerge & PostMerge & \makecell{Smooth\\Merge} & \makecell{Integrated\\Conn}\\ \hline
\makecell{Agglo-ST} &
\begin{minipage}{0.05\textwidth}\centering\begin{tabular}{c}
5 \\
10 \\
15 \\
30
\end{tabular} \end{minipage} & \begin{minipage}{0.16\textwidth}\centering\begin{tabular}{c}
7127.56897 \\
6101.62757 \\
5406.82642 \\
4012.58077 \\
\end{tabular} \end{minipage} & \begin{minipage}{0.16\textwidth}\centering\begin{tabular}{c}
\textcolor{blue}{6718.72885} \\
5713.96147 \\
4911.83706 \\
3565.78342 \\
\end{tabular} \end{minipage} & \begin{minipage}{0.16\textwidth}\centering\begin{tabular}{c}
6810.31878 \\
5769.48728 \\
4912.01064 \\
\textcolor{blue}{3370.35309} \\
\end{tabular} \end{minipage} & \begin{minipage}{0.16\textwidth}\centering\begin{tabular}{c}
6819.93585 \\
\textcolor{blue}{5670.97155} \\
\textcolor{blue}{4889.30909} \\
3591.39863 \\
\end{tabular} \end{minipage} & \begin{minipage}{0.16\textwidth}\centering\begin{tabular}{c}
6971.35316 \\
6001.03185 \\
5076.7984 \\
3693.48877 \\
\end{tabular} \end{minipage}\\
\hline
%%%%%%%%%%%%%%%%%%%%%%%%%%%%%%%%%%%%%%%%% aggl connected cl euclidean dist
\makecell{Conn-Agglo-Euc} &
\begin{minipage}{0.05\textwidth}\centering\begin{tabular}{c}
5 \\
10 \\
15 \\
30
\end{tabular} \end{minipage} & \begin{minipage}{0.16\textwidth}\centering\begin{tabular}{c}
6865.53475 \\
5832.31292 \\
5126.73700 \\
3713.09606 \\
\end{tabular} \end{minipage} & \begin{minipage}{0.16\textwidth}\centering\begin{tabular}{c}
\textcolor{blue}{6584.89575} \\
\textcolor{blue}{5551.75766} \\
\textcolor{blue}{4801.07043} \\
\textcolor{blue}{3446.30952} \\
\end{tabular} \end{minipage} & \begin{minipage}{0.16\textwidth}\centering\begin{tabular}{c}
6750.88909 \\
5865.71361 \\
4919.41237 \\
3579.53496 \\
\end{tabular} \end{minipage} & \begin{minipage}{0.16\textwidth}\centering\begin{tabular}{c}
6831.68193 \\
5729.15024 \\
4946.12142 \\
3490.69984 \\
\end{tabular} \end{minipage} & \begin{minipage}{0.16\textwidth}\centering\begin{tabular}{c}
6651.79576 \\
5552.82496 \\
4862.35879 \\
3556.92483 \\
\end{tabular} \end{minipage}\\
\hline
%%%%%%%%%%%%%%%%%%%%%%%%%%%%%%%%%%% aggl connected cl with thompson distance
\makecell{Conn-Agglo-ST} &
\begin{minipage}{0.05\textwidth}\centering\begin{tabular}{c}
5 \\
10 \\
15 \\
30
\end{tabular} \end{minipage} & 
\begin{minipage}{0.16\textwidth}\centering\begin{tabular}{c}
6900.30819 \\
5794.35416 \\
5049.92481 \\
3595.09951 \\
\end{tabular} \end{minipage} & \begin{minipage}{0.16\textwidth}\centering\begin{tabular}{c}
\textcolor{blue}{6654.56156} \\
\textcolor{blue}{5520.3482} \\
4802.62151 \\
\textcolor{blue}{3373.72568} \\
\end{tabular} \end{minipage} & \begin{minipage}{0.16\textwidth}\centering\begin{tabular}{c}
6927.05043 \\
5687.2591 \\
4768.72953 \\
3613.7709 \\
\end{tabular} \end{minipage} & \begin{minipage}{0.16\textwidth}\centering\begin{tabular}{c}
6786.10306 \\
5691.85431 \\
\textcolor{blue}{4764.92476} \\
3537.25328 \\
\end{tabular} \end{minipage} & \begin{minipage}{0.16\textwidth}\centering\begin{tabular}{c}
6714.21266 \\
5600.75871 \\
4840.57379 \\
3490.14 \\
\end{tabular} \end{minipage}\\
\hline
%%%%%%%%%%%%%%%%%%%%%%%%%%%%%%%%%%%% k-means
\makecell{Conn-KMeans++} &
\begin{minipage}{0.05\textwidth}\centering\begin{tabular}{c}
5 \\
10 \\
15 \\
30
\end{tabular} \end{minipage} & 
\begin{minipage}{0.16\textwidth}\centering\begin{tabular}{c}
7064.48992 \\
6027.69854 \\
5329.54776 \\
3932.44561 \\
\end{tabular} \end{minipage} & \begin{minipage}{0.16\textwidth}\centering\begin{tabular}{c}
\textcolor{blue}{6633.74908 } \\
\textcolor{blue}{5529.78331} \\
\textcolor{blue}{4789.39302} \\
\textcolor{blue}{3542.17183} \\
\end{tabular} \end{minipage} & \begin{minipage}{0.16\textwidth}\centering\begin{tabular}{c}
6935.84391 \\
5908.07834 \\
5119.8224 \\
3857.21766 \\
\end{tabular} \end{minipage} & \begin{minipage}{0.16\textwidth}\centering\begin{tabular}{c}
6826.91025 \\
5676.37115 \\
5046.53568 \\
3579.53919 \\
\end{tabular} \end{minipage} & \begin{minipage}{0.16\textwidth}\centering\begin{tabular}{c}
6816.79815 \\
5618.88795 \\
4945.9793 \\
3595.48081 \\
\end{tabular} \end{minipage}\\
\hline
%%%%%%%%%%%%%%%%%%%%%%%%%%%%%%%%%%%%%%%%% wards method
\makecell{Conn-Ward} &
\begin{minipage}{0.05\textwidth}\centering\begin{tabular}{c}
5 \\
10 \\
15 \\
30
\end{tabular} \end{minipage} & \begin{minipage}{0.16\textwidth}\centering\begin{tabular}{c}
6911.97254 \\
5870.80942 \\
5154.80347 \\
3753.07503 \\
\end{tabular} \end{minipage} & \begin{minipage}{0.16\textwidth}\centering\begin{tabular}{c}
\textcolor{blue}{6685.90989} \\
\textcolor{blue}{5607.57718} \\
\textcolor{blue}{4880.22641} \\
\textcolor{blue}{3507.30482} \\
\end{tabular} \end{minipage} & \begin{minipage}{0.16\textwidth}\centering\begin{tabular}{c}
6851.87697 \\
5705.9256 \\
4893.75679 \\
3624.77683 \\
\end{tabular} \end{minipage} & \begin{minipage}{0.16\textwidth}\centering\begin{tabular}{c}
6816.23469 \\
5634.00779 \\
5002.42116 \\
3590.02407 \\
\end{tabular} \end{minipage} & \begin{minipage}{0.16\textwidth}\centering\begin{tabular}{c}
6705.06569 \\
5685.12404 \\
4958.20896 \\
3633.92843 \\
\end{tabular} \end{minipage}\\
\hline
%%%%%%%%%%%%%%%%%%%%%%%%%%%%%%%% graphconvsc ekcsc %%%%%%%%%
\makecell{GraphConvSC \\ EKGCSC} &
\begin{minipage}{0.05\textwidth}\centering\begin{tabular}{c}
5 \\
10 \\
15 \\
30
\end{tabular} \end{minipage}  
& \begin{minipage}{0.16\textwidth}\centering\begin{tabular}{c}
12585.39110 \\
12506.02452 \\
12733.70084 \\
12267.75023 \\
\end{tabular} \end{minipage} & - & - & - & - \\
\hline
%%%%%%%%%%%%%%%%%%%%%%%% graphconvsc egcsc %%%%%%%%%%%%%%
\makecell{GraphConvSC \\ EGCSC} &
\begin{minipage}{0.05\textwidth}\centering\begin{tabular}{c}
5 \\
10 \\
15 \\
30
\end{tabular} \end{minipage}  
& \begin{minipage}{0.16\textwidth}\centering\begin{tabular}{c}
12215.19311 \\
12168.04789 \\
12060.37141 \\
12216.14648 \\
\end{tabular} \end{minipage} & - & - & - & - \\
\hline
%%%%%%%%%%%%%%%%%%%%%%%% SSC OMP %%%%%%%%%%%%%%
\makecell{SSC-OMP} &
& 14910.37771 & - & - & - & - \\
\hline
%%%%%%%%%%%%%%%%%%%%%%%% EnSC %%%%%%%%%%%%%%
\makecell{EnSC} &
& 13226.72225 & - & - & - & - \\
\hline
\end{tabular}%
}
\end{table}

%%%%%%%%%%%%%%%%%%%%%%%%%%%%%%%%%%%%%% unfiltered 20 cluster 

\begin{table}[H]
\centering
\caption{Distances from points to subspaces using unfiltered input data, with 20 clusters. Blue indicates the best value in each row. Note that the algorithm Agglo-ST performs poorly on unfiltered data (see \Cref{fig:thompson_filtering}), resulting in clusters unsuitable for Conn-Subspace. This is because one large cluster dominates while the others are too small to compute subspaces for $m'$ dimensions.}\label{tab:20_unfiltered}
\setlength{\tabcolsep}{2pt}  % default is 6pt

\resizebox{\textwidth}{!}{%
\begin{tabular}{|c|c|c|c|c|c|c|}
\hline
%%%%%%%%%%%%%%%%%%%%%%%%%%%%%%%%%%%% Thompson clustering
& $m'$ & \makecell{initial \\ clustering} & IterMerge & PostMerge & \makecell{Smooth\\Merge} & \makecell{Integrated\\Conn}\\ \hline
\makecell{Agglo-ST} &
\begin{minipage}{0.05\textwidth}\centering\begin{tabular}{c}
5 \\
10 \\
15 \\
30
\end{tabular} \end{minipage} & \begin{minipage}{0.16\textwidth}\centering\begin{tabular}{c}
7783.69078 \\
6825.71836 \\
6245.5336 \\
5098.51104 \\
\end{tabular} \end{minipage} & \begin{minipage}{0.16\textwidth}\centering\begin{tabular}{c}
\textcolor{blue}{6650.41124} \\
\textcolor{blue}{5924.41911} \\
\textcolor{blue}{4936.45871} \\
4195.20874 \\
\end{tabular} \end{minipage} & \begin{minipage}{0.16\textwidth}\centering\begin{tabular}{c}
7242.28414 \\
6224.32097 \\
5301.69381 \\
4093.73849 \\
\end{tabular} \end{minipage} & \begin{minipage}{0.16\textwidth}\centering\begin{tabular}{c}
7179.5592 \\
6005.97634 \\
- \\
\textcolor{blue}{3934.63734} \\
\end{tabular} \end{minipage} & \begin{minipage}{0.16\textwidth}\centering\begin{tabular}{c}
7528.35562 \\
6318.75352 \\
5771.76803 \\
4865.15531 \\
\end{tabular} \end{minipage}\\
\hline
%%%%%%%%%%%%%%%%%%%%%%%%%%%%%%%%%%%%%%%%% aggl connected cl euclidean dist
\makecell{Conn-Agglo-Euc} &
\begin{minipage}{0.05\textwidth}\centering\begin{tabular}{c}
5 \\
10 \\
15 \\
30
\end{tabular} \end{minipage} & \begin{minipage}{0.16\textwidth}\centering\begin{tabular}{c}
6802.40018 \\
5661.22009 \\
4892.76619 \\
3413.25815 \\
\end{tabular} \end{minipage} & \begin{minipage}{0.16\textwidth}\centering\begin{tabular}{c}
\textcolor{blue}{6560.96462} \\
\textcolor{blue}{5430.55569} \\
\textcolor{blue}{4647.6343} \\
\textcolor{blue}{3237.95838} \\
\end{tabular} \end{minipage} & \begin{minipage}{0.16\textwidth}\centering\begin{tabular}{c}
6711.08028 \\
5681.3181 \\
4901.48325 \\
3405.64826 \\
\end{tabular} \end{minipage} & \begin{minipage}{0.16\textwidth}\centering\begin{tabular}{c}
6917.901 \\
5495.87935 \\
4963.01987 \\
3589.80958 \\
\end{tabular} \end{minipage} & \begin{minipage}{0.16\textwidth}\centering\begin{tabular}{c}
6703.29191 \\
5482.47043 \\
4821.95063 \\
3386.6587 \\
\end{tabular} \end{minipage}\\
\hline
%%%%%%%%%%%%%%%%%%%%%%%%%%%%%%%%%%% aggl connected cl with thompson distance
\makecell{Conn-Agglo-ST} &
\begin{minipage}{0.05\textwidth}\centering\begin{tabular}{c}
5 \\
10 \\
15 \\
30
\end{tabular} \end{minipage} & 
\begin{minipage}{0.16\textwidth}\centering\begin{tabular}{c}
6798.59783 \\
5620.91742 \\
4848.81668 \\
3369.86502 \\
\end{tabular} \end{minipage} & \begin{minipage}{0.16\textwidth}\centering\begin{tabular}{c}
\textcolor{blue}{6605.36159} \\
\textcolor{blue}{5448.75877} \\
\textcolor{blue}{4648.76742} \\
\textcolor{blue}{3255.90307} \\
\end{tabular} \end{minipage} & \begin{minipage}{0.16\textwidth}\centering\begin{tabular}{c}
6784.49403 \\
5920.02513 \\
4691.61664 \\
3458.22741 \\
\end{tabular} \end{minipage} & \begin{minipage}{0.16\textwidth}\centering\begin{tabular}{c}
6790.3771 \\
5797.54924 \\
4916.4418 \\
3332.83429 \\
\end{tabular} \end{minipage} & \begin{minipage}{0.16\textwidth}\centering\begin{tabular}{c}
6689.41092 \\
5544.80481 \\
4768.92015 \\
3285.58812 \\
\end{tabular} \end{minipage}\\
\hline
%%%%%%%%%%%%%%%%%%%%%%%%%%%%%%%%%%%% k-means
\makecell{Conn-KMeans++} &
\begin{minipage}{0.05\textwidth}\centering\begin{tabular}{c}
5 \\
10 \\
15 \\
30
\end{tabular} \end{minipage} & 
\begin{minipage}{0.16\textwidth}\centering\begin{tabular}{c}
6891.16075 \\
5762.35931 \\
5007.12359 \\
3553.72154 \\
\end{tabular} \end{minipage} & \begin{minipage}{0.16\textwidth}\centering\begin{tabular}{c}
\textcolor{blue}{6596.3755} \\
5508.66299 \\
\textcolor{blue}{4702.8914} \\
\textcolor{blue}{3346.93706} \\
\end{tabular} \end{minipage} & \begin{minipage}{0.16\textwidth}\centering\begin{tabular}{c}
6977.26662 \\
5638.61196 \\
4901.24059 \\
3460.08787 \\
\end{tabular} \end{minipage} & \begin{minipage}{0.16\textwidth}\centering\begin{tabular}{c}
6800.17259 \\
5569.52283 \\
4874.53135 \\
3564.00204 \\
\end{tabular} \end{minipage} & \begin{minipage}{0.16\textwidth}\centering\begin{tabular}{c}
6713.21372 \\
\textcolor{blue}{5503.7572} \\
4761.15466 \\
3452.07467 \\
\end{tabular} \end{minipage}\\
\hline
%%%%%%%%%%%%%%%%%%%%%%%%%%%%%%%%%%%%%%%%% wards method
\makecell{Conn-Ward} &
\begin{minipage}{0.05\textwidth}\centering\begin{tabular}{c}
5 \\
10 \\
15 \\
30
\end{tabular} \end{minipage} & \begin{minipage}{0.16\textwidth}\centering\begin{tabular}{c}
7142.41807 \\
6185.16338 \\
5540.21058 \\
4224.92307 \\
\end{tabular} \end{minipage} & \begin{minipage}{0.16\textwidth}\centering\begin{tabular}{c}
\textcolor{blue}{6702.55058} \\
\textcolor{blue}{5776.47192} \\
\textcolor{blue}{4936.7126} \\
\textcolor{blue}{3621.75654} \\
\end{tabular} \end{minipage} & \begin{minipage}{0.16\textwidth}\centering\begin{tabular}{c}
6920.79741 \\
5931.42334 \\
5293.17786 \\
3988.19091 \\
\end{tabular} \end{minipage} & \begin{minipage}{0.16\textwidth}\centering\begin{tabular}{c}
6852.91056 \\
5984.3858 \\
5328.62205 \\
3755.66247 \\
\end{tabular} \end{minipage} & \begin{minipage}{0.16\textwidth}\centering\begin{tabular}{c}
6997.92878 \\
5944.67975 \\
5142.70403 \\
3655.8976 \\
\end{tabular} \end{minipage}\\
\hline
%%%%%%%%%%%%%%%%%%%%%%%%%%%%%%%% graphconvsc ekcsc %%%%%%%%%
\makecell{GraphConvSC \\ EKGCSC} &
\begin{minipage}{0.05\textwidth}\centering\begin{tabular}{c}
5 \\
10 \\
15 \\
30
\end{tabular} \end{minipage}  
& \begin{minipage}{0.16\textwidth}\centering\begin{tabular}{c}
13597.25370 \\
13509.21213 \\
13604.91292 \\
13604.26299 \\
\end{tabular} \end{minipage} & - & - & - & - \\
\hline
%%%%%%%%%%%%%%%%%%%%%%%% graphconvsc egcsc %%%%%%%%%%%%%%
\makecell{GraphConvSC \\ EGCSC} &
\begin{minipage}{0.05\textwidth}\centering\begin{tabular}{c}
5 \\
10 \\
15 \\
30
\end{tabular} \end{minipage}  
& \begin{minipage}{0.16\textwidth}\centering\begin{tabular}{c}
13165.23867 \\
12989.97725 \\
13504.44223 \\
13620.71771 \\
\end{tabular} \end{minipage} & - & - & - & - \\
\hline
%%%%%%%%%%%%%%%%%%%%%%%% SSC OMP %%%%%%%%%%%%%%
\makecell{SSC-OMP} &
& 14689.48839 & - & - & - & - \\
\hline
%%%%%%%%%%%%%%%%%%%%%%%% EnSC %%%%%%%%%%%%%%
\makecell{EnSC} &
& 14458.67851 & - & - & - & - \\
\hline
\end{tabular}%
}
\end{table}

%%%%%%%%%%%%%%%%%%%%%%%%%%%%%%%%%%%%%% filtered 25 cluster 

\begin{table}[H]
\centering
\caption{Distances from points to subspaces using filtered input data, with 25 clusters. Blue indicates the best value in each row.} \label{tab:25_filtered}
\setlength{\tabcolsep}{2pt}  % default is 6pt
\resizebox{\textwidth}{!}{%
\begin{tabular}{|c|c|c|c|c|c|c|}
\hline
%%%%%%%%%%%%%%%%%%%%%%%%%%%%%%%%%%%% Thompson clustering
& $m'$ & \makecell{initial \\ clustering} & IterMerge & PostMerge & \makecell{Smooth\\Merge} & \makecell{Integrated\\Conn}\\ \hline
\makecell{Agglo-ST} &
\begin{minipage}{0.05\textwidth}\centering\begin{tabular}{c}
5 \\
10 \\
15 \\
30
\end{tabular} \end{minipage} & \begin{minipage}{0.16\textwidth}\centering\begin{tabular}{c}
6935.05471 \\
5912.08858 \\
5211.14154 \\
3802.41450 \\
\end{tabular} \end{minipage} & \begin{minipage}{0.16\textwidth}\centering\begin{tabular}{c}
\textcolor{blue}{6477.34752} \\
5433.27462 \\
4679.71563 \\
\textcolor{blue}{3325.62566 } \\
\end{tabular} \end{minipage} & \begin{minipage}{0.16\textwidth}\centering\begin{tabular}{c}
6604.73022 \\
5439.0021 \\
4618.39123 \\
3352.27087 \\
\end{tabular} \end{minipage} & \begin{minipage}{0.16\textwidth}\centering\begin{tabular}{c}
6664.46905 \\
\textcolor{blue}{5413.75331} \\
\textcolor{blue}{4528.27139} \\
3522.85521 \\
\end{tabular} \end{minipage} & \begin{minipage}{0.16\textwidth}\centering\begin{tabular}{c}
6684.20784 \\
5599.36789 \\
4714.31267 \\
3447.12802 \\
\end{tabular} \end{minipage}\\
\hline
%%%%%%%%%%%%%%%%%%%%%%%%%%%%%%%%%%%%%%%%% aggl connected cl euclidean dist
\makecell{Conn-Agglo-Euc} &
\begin{minipage}{0.05\textwidth}\centering\begin{tabular}{c}
5 \\
10 \\
15 \\
30
\end{tabular} \end{minipage} & \begin{minipage}{0.16\textwidth}\centering\begin{tabular}{c}
6622.15663 \\
5538.29201 \\
4799.06170 \\
3340.89806 \\
\end{tabular} \end{minipage} & \begin{minipage}{0.16\textwidth}\centering\begin{tabular}{c}
\textcolor{blue}{6350.81598} \\
\textcolor{blue}{5229.74159} \\
4521.96939 \\
3126.14016 \\
\end{tabular} \end{minipage} & \begin{minipage}{0.16\textwidth}\centering\begin{tabular}{c}
6667.53307 \\
5462.91872 \\
4590.01205 \\
\textcolor{blue}{3107.19511} \\
\end{tabular} \end{minipage} & \begin{minipage}{0.16\textwidth}\centering\begin{tabular}{c}
6530.66165 \\
5352.08113 \\
\textcolor{blue}{4454.23534} \\
3230.85019 \\
\end{tabular} \end{minipage} & \begin{minipage}{0.16\textwidth}\centering\begin{tabular}{c}
6385.11851 \\
5276.7694 \\
4549.961 \\
3195.57073 \\
\end{tabular} \end{minipage}\\
\hline
%%%%%%%%%%%%%%%%%%%%%%%%%%%%%%%%%%% aggl connected cl with thompson distance
\makecell{Conn-Agglo-ST} &
\begin{minipage}{0.05\textwidth}\centering\begin{tabular}{c}
5 \\
10 \\
15 \\
30
\end{tabular} \end{minipage} & 
\begin{minipage}{0.16\textwidth}\centering\begin{tabular}{c}
6765.17821 \\
5649.72727 \\
4884.40787 \\
3395.02750 \\
\end{tabular} \end{minipage} & \begin{minipage}{0.16\textwidth}\centering\begin{tabular}{c}
\textcolor{blue}{6399.39178} \\
\textcolor{blue}{5303.26748} \\
\textcolor{blue}{4527.33796} \\
\textcolor{blue}{3200.7043} \\
\end{tabular} \end{minipage} & \begin{minipage}{0.16\textwidth}\centering\begin{tabular}{c}
6548.14257 \\
5325.49801 \\
4608.88795 \\
3285.17893 \\
\end{tabular} \end{minipage} & \begin{minipage}{0.16\textwidth}\centering\begin{tabular}{c}
6532.95274 \\
5358.70727 \\
4606.41312 \\
3224.51674 \\
\end{tabular} \end{minipage} & \begin{minipage}{0.16\textwidth}\centering\begin{tabular}{c}
6498.09127 \\
5354.7229 \\
4574.74659 \\
3256.04637 \\
\end{tabular} \end{minipage}\\
\hline
%%%%%%%%%%%%%%%%%%%%%%%%%%%%%%%%%%%% k-means
\makecell{Conn-KMeans++} &
\begin{minipage}{0.05\textwidth}\centering\begin{tabular}{c}
5 \\
10 \\
15 \\
30
\end{tabular} \end{minipage} & 
\begin{minipage}{0.16\textwidth}\centering\begin{tabular}{c}
6846.60590 \\
5777.17176 \\
5045.16003 \\
3609.66688 \\
\end{tabular} \end{minipage} & \begin{minipage}{0.16\textwidth}\centering\begin{tabular}{c}
6412.88878 \\
\textcolor{blue}{5327.48569} \\
\textcolor{blue}{4581.94485} \\
3239.08517 \\
\end{tabular} \end{minipage} & \begin{minipage}{0.16\textwidth}\centering\begin{tabular}{c}
6546.40152 \\
5859.8821 \\
4944.36625 \\
3385.55874 \\
\end{tabular} \end{minipage} & \begin{minipage}{0.16\textwidth}\centering\begin{tabular}{c}
\textcolor{blue}{6404.29306} \\
5395.78375 \\
4616.90696 \\
3247.91839 \\
\end{tabular} \end{minipage} & \begin{minipage}{0.16\textwidth}\centering\begin{tabular}{c}
6645.10238 \\
5476.8037 \\
4645.33832 \\
\textcolor{blue}{3230.89165} \\
\end{tabular} \end{minipage}\\
\hline
%%%%%%%%%%%%%%%%%%%%%%%%%%%%%%%%%%%%%%%%% wards method
\makecell{Conn-Ward} &
\begin{minipage}{0.05\textwidth}\centering\begin{tabular}{c}
5 \\
10 \\
15 \\
30
\end{tabular} \end{minipage} & \begin{minipage}{0.16\textwidth}\centering\begin{tabular}{c}
6804.97968 \\
5748.56173 \\
5018.93476 \\
3602.40923 \\
\end{tabular} \end{minipage} & \begin{minipage}{0.16\textwidth}\centering\begin{tabular}{c}
\textcolor{blue}{6544.75466} \\
\textcolor{blue}{5387.06249} \\
\textcolor{blue}{4653.33822} \\
3373.04519 \\
\end{tabular} \end{minipage} & \begin{minipage}{0.16\textwidth}\centering\begin{tabular}{c}
6714.11371 \\
5517.56823 \\
4739.56481 \\
\textcolor{blue}{3312.15312} \\
\end{tabular} \end{minipage} & \begin{minipage}{0.16\textwidth}\centering\begin{tabular}{c}
6704.62957 \\
5495.12792 \\
4805.99741 \\
3343.36821 \\
\end{tabular} \end{minipage} & \begin{minipage}{0.16\textwidth}\centering\begin{tabular}{c}
6562.28748 \\
5499.27583 \\
4809.59483 \\
3397.42704 \\
\end{tabular} \end{minipage}\\
\hline
%%%%%%%%%%%%%%%%%%%%%%%%%%%%%%%% graphconvsc ekcsc %%%%%%%%%
\makecell{GraphConvSC \\ EKGCSC} &
\begin{minipage}{0.05\textwidth}\centering\begin{tabular}{c}
5 \\
10 \\
15 \\
30
\end{tabular} \end{minipage}  
& \begin{minipage}{0.16\textwidth}\centering\begin{tabular}{c}
12231.26304 \\
12240.14617 \\
11942.24078 \\
11934.06822 \\
\end{tabular} \end{minipage} & - & - & - & - \\
\hline
%%%%%%%%%%%%%%%%%%%%%%%% graphconvsc egcsc %%%%%%%%%%%%%%
\makecell{GraphConvSC \\ EGCSC} &
\begin{minipage}{0.05\textwidth}\centering\begin{tabular}{c}
5 \\
10 \\
15 \\
30
\end{tabular} \end{minipage}  
& \begin{minipage}{0.16\textwidth}\centering\begin{tabular}{c}
11972.89186 \\
11747.07094 \\
11735.43384 \\
11825.74653 \\
\end{tabular} \end{minipage} & - & - & - & - \\
\hline
%%%%%%%%%%%%%%%%%%%%%%%% SSC OMP %%%%%%%%%%%%%%
\makecell{SSC-OMP} &
& 14677.86061 & - & - & - & - \\
\hline
%%%%%%%%%%%%%%%%%%%%%%%% EnSC %%%%%%%%%%%%%%
\makecell{EnSC} &
& 12982.57873 & - & - & - & - \\
\hline
\end{tabular}%
}
\end{table}

%%%%%%%%%%%%%%%%%%%%%%%%%%%%%%%%%%%%%% unfiltered 25 cluster 

\begin{table}[H]
\centering
\caption{Distances from points to subspaces using unfiltered input data, with 25 clusters. Blue indicates the best value in each row. Note that the algorithm Agglo-ST performs poorly on unfiltered data (see \Cref{fig:thompson_filtering}), resulting in clusters unsuitable for Conn-Subspace. This is because one large cluster dominates while the others are too small to compute subspaces for $m'$ dimensions.}\label{tab:25_unfiltered}
\setlength{\tabcolsep}{2pt}  % default is 6pt

\resizebox{\textwidth}{!}{%
\begin{tabular}{|c|c|c|c|c|c|c|}
\hline
%%%%%%%%%%%%%%%%%%%%%%%%%%%%%%%%%%%% Thompson clustering
& $m'$ & \makecell{initial \\ clustering} & IterMerge & PostMerge & \makecell{Smooth\\Merge} & \makecell{Integrated\\Conn}\\ \hline
\makecell{Agglo-ST} &
\begin{minipage}{0.05\textwidth}\centering\begin{tabular}{c}
5 \\
10 \\
15 \\
30
\end{tabular} \end{minipage} & \begin{minipage}{0.16\textwidth}\centering\begin{tabular}{c}
7611.33551 \\
6658.96184 \\
6109.23328 \\
4992.56736 \\
\end{tabular} \end{minipage} & \begin{minipage}{0.16\textwidth}\centering\begin{tabular}{c}
\textcolor{blue}{6431.27588} \\
\textcolor{blue}{5484.5064} \\
\textcolor{blue}{4963.57033} \\
\textcolor{blue}{3122.85764} \\
\end{tabular} \end{minipage} & \begin{minipage}{0.16\textwidth}\centering\begin{tabular}{c}
6985.54103 \\
5882.39635 \\
- \\
3934.40832 \\
\end{tabular} \end{minipage} & \begin{minipage}{0.16\textwidth}\centering\begin{tabular}{c}
6780.14297 \\
- \\
- \\
- \\
\end{tabular} \end{minipage} & \begin{minipage}{0.16\textwidth}\centering\begin{tabular}{c}
7378.82363 \\
6242.66174 \\
5520.47426 \\
4761.62842 \\
\end{tabular} \end{minipage}\\
\hline
%%%%%%%%%%%%%%%%%%%%%%%%%%%%%%%%%%%%%%%%% aggl connected cl euclidean dist
\makecell{Conn-Agglo-Euc} &
\begin{minipage}{0.05\textwidth}\centering\begin{tabular}{c}
5 \\
10 \\
15 \\
30
\end{tabular} \end{minipage} & \begin{minipage}{0.16\textwidth}\centering\begin{tabular}{c}
6648.23889 \\
5488.59572 \\
4708.85936 \\
3218.62384 \\
\end{tabular} \end{minipage} & \begin{minipage}{0.16\textwidth}\centering\begin{tabular}{c}
\textcolor{blue}{6401.5226} \\
\textcolor{blue}{5239.18674} \\
\textcolor{blue}{4414.99156} \\
\textcolor{blue}{2947.91908} \\
\end{tabular} \end{minipage} & \begin{minipage}{0.16\textwidth}\centering\begin{tabular}{c}
6603.82734 \\
5564.22328 \\
4760.16073 \\
3301.88318 \\
\end{tabular} \end{minipage} & \begin{minipage}{0.16\textwidth}\centering\begin{tabular}{c}
6533.52799 \\
5428.61674 \\
4606.55632 \\
3157.41492 \\
\end{tabular} \end{minipage} & \begin{minipage}{0.16\textwidth}\centering\begin{tabular}{c}
6444.52568 \\
5306.91418 \\
4680.51634 \\
3200.42369 \\
\end{tabular} \end{minipage}\\
\hline
%%%%%%%%%%%%%%%%%%%%%%%%%%%%%%%%%%% aggl connected cl with thompson distance
\makecell{Conn-Agglo-ST} &
\begin{minipage}{0.05\textwidth}\centering\begin{tabular}{c}
5 \\
10 \\
15 \\
30
\end{tabular} \end{minipage} & 
\begin{minipage}{0.16\textwidth}\centering\begin{tabular}{c}
6585.85218 \\
5375.51503 \\
4585.05939 \\
3080.66145 \\
\end{tabular} \end{minipage} & \begin{minipage}{0.16\textwidth}\centering\begin{tabular}{c}
6375.37258 \\
\textcolor{blue}{5267.75215} \\
4467.76037 \\
3016.05834 \\
\end{tabular} \end{minipage} & \begin{minipage}{0.16\textwidth}\centering\begin{tabular}{c}
6500.81257 \\
5347.90321 \\
4523.11896 \\
\textcolor{blue}{2992.33895} \\
\end{tabular} \end{minipage} & \begin{minipage}{0.16\textwidth}\centering\begin{tabular}{c}
6552.52463 \\
5268.33666 \\
4528.61575 \\
2938.29087 \\
\end{tabular} \end{minipage} & \begin{minipage}{0.16\textwidth}\centering\begin{tabular}{c}
\textcolor{blue}{6328.12887} \\
5283.01416 \\
\textcolor{blue}{4452.95515} \\
3028.98765 \\
\end{tabular} \end{minipage}\\
\hline
%%%%%%%%%%%%%%%%%%%%%%%%%%%%%%%%%%%% k-means
\makecell{Conn-KMeans++} &
\begin{minipage}{0.05\textwidth}\centering\begin{tabular}{c}
5 \\
10 \\
15 \\
30
\end{tabular} \end{minipage} & 
\begin{minipage}{0.16\textwidth}\centering\begin{tabular}{c}
6791.59278 \\
5689.86451 \\
4942.47421 \\
3501.83054 \\
\end{tabular} \end{minipage} & \begin{minipage}{0.16\textwidth}\centering\begin{tabular}{c}
6459.77151 \\
5285.76723 \\
\textcolor{blue}{4501.4064} \\
\textcolor{blue}{3122.85764} \\
\end{tabular} \end{minipage} & \begin{minipage}{0.16\textwidth}\centering\begin{tabular}{c}
6651.56208 \\
5688.4807 \\
4604.54066 \\
3264.31383 \\
\end{tabular} \end{minipage} & \begin{minipage}{0.16\textwidth}\centering\begin{tabular}{c}
6653.9628 \\
5548.82075 \\
4541.47083 \\
3192.40442 \\
\end{tabular} \end{minipage} & \begin{minipage}{0.16\textwidth}\centering\begin{tabular}{c}
\textcolor{blue}{6434.88511} \\
\textcolor{blue}{5281.69382} \\
4715.21944 \\
3302.94252 \\
\end{tabular} \end{minipage}\\
\hline
%%%%%%%%%%%%%%%%%%%%%%%%%%%%%%%%%%%%%%%%% wards method
\makecell{Conn-Ward} &
\begin{minipage}{0.05\textwidth}\centering\begin{tabular}{c}
5 \\
10 \\
15 \\
30
\end{tabular} \end{minipage} & \begin{minipage}{0.16\textwidth}\centering\begin{tabular}{c}
6964.9354 \\
6010.5925 \\
5355.69749 \\
4044.59545 \\
\end{tabular} \end{minipage} & \begin{minipage}{0.16\textwidth}\centering\begin{tabular}{c}
\textcolor{blue}{6586.38741} \\
\textcolor{blue}{5504.71253} \\
4757.14492 \\
3361.74465 \\
\end{tabular} \end{minipage} & \begin{minipage}{0.16\textwidth}\centering\begin{tabular}{c}
6725.90101 \\
5508.17897 \\
4895.42358 \\
3385.23974 \\
\end{tabular} \end{minipage} & \begin{minipage}{0.16\textwidth}\centering\begin{tabular}{c}
6791.40374 \\
5585.96961 \\
4868.43329 \\
3423.68021 \\
\end{tabular} \end{minipage} & \begin{minipage}{0.16\textwidth}\centering\begin{tabular}{c}
6687.39939 \\
5530.5792 \\
\textcolor{blue}{4645.91934} \\
\textcolor{blue}{3335.34227} \\
\end{tabular} \end{minipage}\\
\hline
%%%%%%%%%%%%%%%%%%%%%%%%%%%%%%%% graphconvsc ekcsc %%%%%%%%%
\makecell{GraphConvSC \\ EKGCSC} &
\begin{minipage}{0.05\textwidth}\centering\begin{tabular}{c}
5 \\
10 \\
15 \\
30
\end{tabular} \end{minipage}  
& \begin{minipage}{0.16\textwidth}\centering\begin{tabular}{c}
13250.46585 \\
13528.28216 \\
13596.14924 \\
13576.80709 \\
\end{tabular} \end{minipage} & - & - & - & - \\
\hline
%%%%%%%%%%%%%%%%%%%%%%%% graphconvsc egcsc %%%%%%%%%%%%%%
\makecell{GraphConvSC \\ EGCSC} &
\begin{minipage}{0.05\textwidth}\centering\begin{tabular}{c}
5 \\
10 \\
15 \\
30
\end{tabular} \end{minipage}  
& \begin{minipage}{0.16\textwidth}\centering\begin{tabular}{c}
12476.49544 \\
12998.51352 \\
13064.42761 \\
13583.19723 \\
\end{tabular} \end{minipage} & - & - & - & - \\
\hline
%%%%%%%%%%%%%%%%%%%%%%%% SSC OMP %%%%%%%%%%%%%%
\makecell{SSC-OMP} &
& 14473.09811 & - & - & - & - \\
\hline
%%%%%%%%%%%%%%%%%%%%%%%% EnSC %%%%%%%%%%%%%%
\makecell{EnSC} &
& 14351.16981 & - & - & - & - \\
\hline
\end{tabular}%
}
\end{table}

\let\table\alenexOriginalTable
\let\endtable\endalenexOriginalTable

%% file: biblio.bib
@inbook{ipcc2021,
   author = {Fox-Kemper, B. and Hewitt, H.T. and Xiao, C. and Aðalgeirsdóttir, G. and Drijfhout, S.S. and Edwards, T.L. and Golledge, N.R. and Hemer, M. and Kopp, R.E. and Krinner, G. and Mix, A. and Notz, D. and Nowicki, S. and Nurhati, I.S. and Ruiz, L. and Sallée, J.-B. and Slangen, A.B.A. and Yu, Y.},
   title = {Ocean, Cryosphere and Sea Level Change},
   booktitle = {Climate Change 2021: The Physical Science Basis. Contribution of Working Group I to the Sixth Assessment Report of the Intergovernmental Panel on Climate Change},
   publisher = {Cambridge University Press},
   pages = {1211–1362},
   year = {2021},
   type = {Book Section}
}

@article{church2004,
    author = {Church, John A. and White, Neil J. and Coleman, Richard and Lambeck, Kurt and Mitrovica, Jerry X.} ,
    title = {Est. of the Reg. Dist. of Sea Level Rise over the 1950–2000 Period} ,
    journal = {J. Clim.},
    volume = {17},
    pages = {2609-2625},
    year = {2004} 
}

@article{ValidiBL22,
  author       = {Hamidreza Validi and
                  Austin Buchanan and
                  Eugene Lykhovyd},
  title        = {Imposing Contiguity Constraints in Political Districting Models},
  journal      = {OR},
  volume       = {70},
  number       = {2},
  pages        = {867--892},
  year         = {2022},
}

@article{scikitlearn,
  author       = {Fabian Pedregosa and
                  Ga{\"{e}}l Varoquaux and
                  Alexandre Gramfort and
                  Vincent Michel and
                  Bertrand Thirion and
                  Olivier Grisel and
                  Mathieu Blondel and
                  Peter Prettenhofer and
                  Ron Weiss and
                  Vincent Dubourg and
                  Jake VanderPlas and
                  Alexandre Passos and
                  David Cournapeau and
                  Matthieu Brucher and
                  Matthieu Perrot and
                  Edouard Duchesnay},
  title        = {Scikit-learn: Machine Learning in Python},
  journal      = {JMLR},
  volume       = {12},
  pages        = {2825--2830},
  year         = {2011},
}

@article{lynch1975equivalence,
  title={The equivalence of theorem proving and the interconnection problem},
  author={Lynch, James F},
  journal={ACM Sigda Newsl.},
  volume={5},
  number={3},
  pages={31--36},
  year={1975},
  publisher={ACM New York, NY, USA}
}

@article{kramer1984complexity,
  title={The complexity of wire routing and finding minimum area layouts for arbitrary {VLSI} circuits},
  author={Kramer, Mark R. and van Leeuwen, Jan},
  journal={Adv. Comput. Res},
  volume={2},
  pages={129--146},
  year={1984}
}

@inproceedings{formann1990vlsi,
  title={The {VLSI} layout problem in various embedding models},
  author={Formann, Michael and Wagner, Frank},
  booktitle={WG},
  pages={130--139},
  year={1990},
  organization={Springer}
}

@unpublished{Llo57,
  author       = {Stuart P. Lloyd},
  title        = {Least squares quantization in PCM},
  note         = {Bell Laboratories Technical Memorandum, later published in IEEE Trans. Inf. Theory, 1982},
  year         = {1957}
}

@article{Llo82,
  author       = {Stuart P. Lloyd},
  title        = {Least squares quantization in PCM},
  journal      = {IEEE Trans. Inf. Theory},
  volume       = {28},
  number       = {2},
  pages        = {129--137},
  year         = {1982}
}

@article{ADHP09,
  author       = {Daniel Aloise and
                  Amit Deshpande and
                  Pierre Hansen and
                  Preyas Popat},
  title        = {NP-hardness of Euclidean sum-of-squares clustering},
  journal      = {Machine Learning},
  volume       = {75},
  number       = {2},
  pages        = {245--248},
  year         = {2009},
}

@article{MahajanNV12,
  author       = {Meena Mahajan and
                  Prajakta Nimbhorkar and
                  Kasturi R. Varadarajan},
  title        = {The planar k-means problem is NP-hard},
  journal      = {TCS},
  volume       = {442},
  pages        = {13--21},
  year         = {2012},
}

@inproceedings{IKI94,
  author       = {Mary Inaba and
                  Naoki Katoh and
                  Hiroshi Imai},
  title        = {Applications of Weighted Voronoi Diagrams and Randomization to Variance-Based
                  \emph{k}-Clustering},
  booktitle    = {Proc. SoCG},
  pages        = {332--339},
  publisher    = {{ACM}},
  year         = {1994},
}

@inproceedings{CAEMN22,
  author       = {Vincent Cohen{-}Addad and
                  Hossein Esfandiari and
                  Vahab S. Mirrokni and
                  Shyam Narayanan},
  title        = {Improved approximations for Euclidean \emph{k}-means and \emph{k}-median,
                  via nested quasi-independent sets},
  booktitle    = {Proc. STOC},
  pages        = {1621--1628},
  year         = {2022},
}

@inproceedings{LS19,
  author       = {Silvio Lattanzi and
                  Christian Sohler},
  title        = {A Better k-means++ Algorithm via Local Search},
  booktitle    = {Proc. ICML},
  volume       = {97},
  pages        = {3662--3671},
  year         = {2019},
}

@article{GEGHBB08,
  author       = {Rong Ge and
                  Martin Ester and
                  Byron J. Gao and
                  Zengjian Hu and
                  Binay K. Bhattacharya and
                  Boaz Ben{-}Moshe},
  title        = {Joint cluster analysis of attribute data and relationship data: The
                  connected \emph{k}-center problem, algorithms and applications},
  journal      = {TKDD},
  volume       = {2},
  number       = {2},
  pages        = {7:1--7:35},
  year         = {2008},
}

@article{DELRRSW24,
  author       = {Lukas Drexler and
                  Jan Eube and
                  Kelin Luo and
                  Dorian Reineccius and
                  Heiko R{\"{o}}glin and
                  Melanie Schmidt and
                  Julian Wargalla},
  title        = {Connected k-Center and k-Diameter Clustering},
  journal      = {Alg.},
  volume       = {86},
  number       = {11},
  pages        = {3425--3464},
  year         = {2024},
}

@article{murtagh85,
  title     = {A Survey of Algorithms for Contiguity-Constrained Clustering and Related Problems},
  author    = {Murtagh, Fionn},
  journal   = {Comput. J.},
  volume    = {28},
  number    = {1},
  pages     = {82--88},
  year      = {1985},
  publisher = {Oxford University Press},
}

@inproceedings{AV07,
  author    = {Arthur, David and Vassilvitskii, Sergei},
  title     = {k-means++: The Advantages of Careful Seeding},
  booktitle = {Proc. SODA},
  year      = {2007},
  pages     = {1027--1035},
}

@inproceedings{ELRRS25,
  author    = {Jan Eube and Kelin Luo and Dorian Reineccius and Heiko R{\"o}glin and Melanie Schmidt},
  title     = {Connected \(k\)-Median with Disjoint and Non-disjoint Clusters},
  booktitle = {Proc. ESA},
  year      = {2025},
}

@article{TM14,
  author  = {Philip R. Thompson and Mark A. Merrifield},
  title   = {A unique asymmetry in the pattern of recent sea level change},
  journal = {Geophysical Research Letters},
  volume  = {41},
  number  = {21},
  pages   = {7675--7683},
  year    = {2014},
}

@misc{altimetry_data,
    author={{E.U. Copernicus Marine Service Information (CMEMS). Marine Data Store (MDS)}},
    title = {{Global Ocean Gridded L 4 Sea Surface Heights And Derived Variables Reprocessed 1993 Ongoing}},
    year = {2024},
    note={acc. on 2024-09-11}
}

@misc{copernicus_climate_change_data,
    author={{Copernicus Climate Change Service, Climate Data Store}},
    title={{Sea level gridded data from satellite observations for the global ocean from 1993 to present.}},
    year ={2018},
    note={acc. on 2025-07-01}
}

@article{le1998improved,
  title={An improved mapping method of multisatellite altimeter data},
  author={Le Traon, PY and Nadal, F and Ducet, N},
  journal={JTECH},
  volume={15},
  number={2},
  pages={522--534},
  year={1998}
}

@mastersthesis{EMastersthesis,
    author = {Ewerdwalbesloh, Yorck},
    title = {Finding coherent regions in the ocean in CMIP6
model simulations},
    school = {Rheinische Friedrich-Wilhelms-Universit{\"a}t Bonn},
    year = {2022}
}

@article{li11,
  author       = {Jia Li},
  title        = {Agglomerative connectivity constrained clustering for image segmentation},
  journal      = {Stat. Anal. Data Min.},
  volume       = {4},
  number       = {1},
  pages        = {84--99},
  year         = {2011},
}

@article{CKV13,
  author       = {M. Emre Celebi and
                  Hassan A. Kingravi and
                  Patricio A. Vela},
  title        = {A comparative study of efficient initialization methods for the k-means
                  clustering algorithm},
  journal      = {ESWA},
  volume       = {40},
  number       = {1},
  pages        = {200--210},
  year         = {2013},
}

@article{li2024analysis,
  title={Analysis of sea level variability and its contributions in the Bohai, Yellow Sea, and East China Sea},
  author={Li, Yanxiao and Feng, Jianlong and Yang, Xinming and Zhang, Shuwei and Chao, Guofang and Zhao, Liang and Fu, Hongli},
  journal={Frontiers in Marine Science},
  volume={11},
  pages={1381187},
  year={2024},
  publisher={Frontiers Media SA},
}

@article{rietbroek2016revisiting,
  title={Revisiting the contemporary sea-level budget on global and regional scales},
  author={Rietbroek, Roelof and Brunnabend, Sandra-Esther and Kusche, J{\"u}rgen and Schr{\"o}ter, Jens and Dahle, Christoph},
  journal={Proc. PNAS},
  volume={113},
  number={6},
  pages={1504--1509},
  year={2016},
  publisher={National Academy of Sciences}
}

@article{klos2019introducing,
  title={Introducing a vertical land motion model for improving estimates of sea level rates derived from tide gauge records affected by earthquakes},
  author={Klos, Anna and Kusche, J{\"u}rgen and Fenoglio-Marc, Luciana and Bos, Machiel S and Bogusz, Janusz},
  journal={GPS Solutions},
  volume={23},
  number={4},
  pages={102},
  year={2019},
  publisher={Springer}
}

@inproceedings{gupta_clustering_2011,
    address = {Berlin, Heidelberg},
    title = {Clustering with {Internal} {Connectedness}},
    booktitle = {{WALCOM}: {Algorithms} and {Computation}},
    publisher = {Springer},
    author = {Gupta, Neelima and Pancholi, Aditya and Sabharwal, Yogish},
    editor = {Katoh, Naoki and Kumar, Amit},
    year = {2011},
    pages = {158--169},
}

@article{sim2013survey,
  title={A survey on enhanced subspace clustering},
  author={Sim, Kelvin and Gopalkrishnan, Vivekanand and Zimek, Arthur and Cong, Gao},
  journal={Data mining and knowledge discovery},
  volume={26},
  number={2},
  pages={332--397},
  year={2013},
  publisher={Springer}
}

@article{kriegel2009clustering,
  title={Clustering high-dimensional data: A survey on subspace clustering, pattern-based clustering, and correlation clustering},
  author={Kriegel, Hans-Peter and Kr{\"o}ger, Peer and Zimek, Arthur},
  journal={TKDD},
  volume={3},
  number={1},
  pages={1--58},
  year={2009},
}

@inproceedings{agrawal1998automatic,
  title={Automatic subspace clust. of high dim. data for data mining appls.},
  author={Agrawal, Rakesh and Gehrke, Johannes and Gunopulos, Dimitrios and Raghavan, Prabhakar},
  booktitle={Proc. SIGMoD },
  pages={94--105},
  year={1998}
}

@article{aggarwal1999fast,
  title={Fast algorithms for projected clustering},
  author={Aggarwal, Charu C and Wolf, Joel L and Yu, Philip S and Procopiuc, Cecilia and Park, Jong Soo},
  journal={ACM SIGMoD record},
  volume={28},
  number={2},
  pages={61--72},
  year={1999},
  publisher={ACM New York, NY, USA}
}

@inproceedings{kailing2004density,
  title={Density-connected subspace clustering for high-dimensional data},
  author={Kailing, Karin and Kriegel, Hans-Peter and Kr{\"o}ger, Peer},
  booktitle={Proc. SDM},
  pages={246--256},
  year={2004},
  organization={SIAM}
}

@inproceedings{ArthurV06,
  author       = {David Arthur and
                  Sergei Vassilvitskii},
  title        = {How slow is the \emph{k}-means method?},
  booktitle    = {Proc. SoCG},
  pages        = {144--153},
  year         = {2006},
}

@inproceedings{wang2022convergence,
  title={Convergence and recovery guarantees of the k-subspaces method for subspace clustering},
  author={Wang, Peng and Liu, Huikang and So, Anthony Man-Cho and Balzano, Laura},
  booktitle={ICML},
  pages={22884--22918},
  year={2022},
}

@article{tarjan1984worst,
  title={Worst-case analysis of set union algorithms},
  author={Tarjan, Robert E and Van Leeuwen, Jan},
  journal={JACM},
  volume={31},
  number={2},
  pages={245--281},
  year={1984},
}

@article{elhamifar2013sparse,
  title={Sparse subspace clustering: Algorithm, theory, and applications},
  author={Elhamifar, Ehsan and Vidal, Ren{\'e}},
  journal={TPAMI},
  volume={35},
  number={11},
  pages={2765--2781},
  year={2013},
  publisher={IEEE}
}

@inproceedings{you2016oracle,
  title={Oracle based active set algorithm for scalable elastic net subspace clustering},
  author={You, Chong and Li, Chun-Guang and Robinson, Daniel P and Vidal, Ren{\'e}},
  booktitle={Proc. CVPR},
  pages={3928--3937},
  year={2016}
}

@inproceedings{YouRV16,
  author       = {Chong You and
                  Daniel P. Robinson and
                  Ren{\'{e}} Vidal},
  title        = {Scalable Sparse Subspace Clustering by Orthogonal Matching Pursuit},
  booktitle    = {CVPR},
  pages        = {3918--3927},
  year         = {2016},
}

@article{cai2020graph,
  title={Graph convolutional subspace clustering: A robust subspace clustering framework for hyperspectral image},
  author={Cai, Yaoming and Zhang, Zijia and Cai, Zhihua and Liu, Xiaobo and Jiang, Xinwei and Yan, Qin},
  journal={TGRS},
  volume={59},
  number={5},
  pages={4191--4202},
  year={2020},
}

@article{wang2022graph,
  title={Graph regularized spatial--spectral subspace clustering for hyperspectral band selection},
  author={Wang, Jun and Tang, Chang and Zheng, Xiao and Liu, Xinwang and Zhang, Wei and Zhu, En},
  journal={Neural Netw.},
  volume={153},
  pages={292--302},
  year={2022},
  publisher={Elsevier}
}

@article{guan2024contrastive,
  title={Contrastive multiview subspace clustering of hyperspectral images based on graph convolutional networks},
  author={Guan, Renxiang and Li, Zihao and Tu, Wenxuan and Wang, Jun and Liu, Yue and Li, Xianju and Tang, Chang and Feng, Ruyi},
  journal={TGRS},
  volume={62},
  pages={1--14},
  year={2024},
  publisher={IEEE}
}

@article{janheiko-connected-25,
  author       = {Jan Eube and
                  Heiko R{\"{o}}glin},
  title        = {New Algorithms and Hardness Results for Connected Clustering},
  journal      = {CoRR},
  volume       = {abs/2511.19085},
  year         = {2025},
}
